\documentclass{article}

\usepackage[accepted]{icml2020}
\usepackage{subfigure}
\usepackage{hhline}
\usepackage[space]{grffile}
\usepackage{multirow}
\usepackage{lineno}

\usepackage{mathrsfs}
\usepackage{times}
\usepackage{latexsym}
\usepackage{wrapfig}
\usepackage{amsmath}

\DeclareMathOperator*{\argmin}{arg\,min}

\usepackage{amssymb}
\usepackage{amsthm}
\usepackage{amscd}
\usepackage{amsfonts}
\usepackage{graphicx}
\usepackage{natbib}
\usepackage{xcolor}
\usepackage{algorithm}
\usepackage{enumitem}
\usepackage{dsfont}

\newcommand{\PP}{\mathbb P}
\newcommand{\BR}{\mathds 1}
\newcommand{\F}{\mathcal F}
\newcommand{\E}{\mathbb E}
\newcommand{\R}{Z}

\newcommand{\BPeer}{\BR_\text{peer}}
\newcommand{\lPeer}{\ell_\text{peer}}
\newcommand{\AlphaPeer}{\BR_\text{$\alpha$-peer}}
\newcommand{\AlphaStar}{\BR_\text{$\alpha^*$-peer}}

\newcommand{\lAlphaPeer}{\ell_\text{$\alpha$-peer}}
\newcommand{\lAlphaStar}{\ell_\text{$\alpha^*$-peer}}

\def\BState{\State\hskip-\ALG@thistlm}

\makeatletter
 \renewcommand{\ALG@name}{Mechanism}
\makeatother
\newtheorem{theorem}{Theorem}

\newtheorem{lemma}{Lemma}
\newtheorem{proposition}[theorem]{Proposition}
\newtheorem{definition}{Definition}

\newtheorem{example}{Example}
\usepackage{hyperref}
\hypersetup{%
 colorlinks=true, 
 urlcolor=blue,%
 linkcolor=blue,%
 citecolor=blue,%
 linkbordercolor=blue, 
 pdfborderstyle={/S/U/W 1}
}

\newcommand{\squishlist}{
\begin{list}{{{\small{$\bullet$}}}}
{\setlength{\itemsep}{3pt}      \setlength{\parsep}{1pt}
\setlength{\topsep}{1pt}       \setlength{\partopsep}{0pt}
\setlength{\leftmargin}{1em} \setlength{\labelwidth}{1em}
\setlength{\labelsep}{0.5em} } }
\newcommand{\squishend}{  \end{list}  }

\ifodd 0
\newcommand{\rev}[1]{{\color{blue}#1}}
\newcommand{\com}[1]{\textbf{\color{red}(COMMENT: #1)}}
\newcommand{\clar}[1]{\textbf{\color{green}(NEED CLARIFICATION: #1)}}
\newcommand{\response}[1]{\textbf{\color{magenta}(RESPONSE: #1)}}
\else
\newcommand{\rev}[1]{#1}

\newcommand{\com}[1]{}
\newcommand{\clar}[1]{}
\newcommand{\response}[1]{}
\fi
\newcommand{\RNum}[1]{\uppercase\expandafter{\romannumeral #1\relax}}

\begin{document}

\icmltitlerunning{Peer Loss Functions}
\twocolumn[
\icmltitle{Peer Loss Functions:\\ Learning from Noisy Labels without Knowing Noise Rates}
\begin{icmlauthorlist}
\icmlauthor{Yang Liu}{ucsc}
\icmlauthor{Hongyi Guo}{sjtu}
\end{icmlauthorlist}

\icmlaffiliation{ucsc}{Computer Science and Engineering, UC Santa Cruz, Santa Cruz, CA, USA}
\icmlaffiliation{sjtu}{Computer Science and Engineering, Shanghai Jiao Tong University, China}

\icmlcorrespondingauthor{Yang Liu}{yangliu@ucsc.edu}
\icmlcorrespondingauthor{Hongyi Guo}{guohongyi@sjtu.edu.cn}

\icmlkeywords{}
\vskip 0.3in
]
\begin{NoHyper}
\printAffiliationsAndNotice{}
\end{NoHyper}


\begin{abstract}
Learning with noisy labels is a common challenge in supervised learning. Existing approaches often require practitioners to specify \emph{noise rates}, i.e., a set of parameters controlling the severity of label noises in the problem, and the specifications are either assumed to be given or estimated using additional steps. In this work, we introduce a new family of loss functions that we name as \emph{peer loss functions}, which enables learning from noisy labels and does not require a priori specification of the noise rates.
Peer loss functions work within the standard empirical risk minimization (ERM) framework. We show that, under mild conditions, performing ERM with peer loss functions on the noisy data leads to the optimal or a near-optimal classifier as if performing ERM over the clean training data, which we do not have access to.
We pair our results with an extensive set of experiments. Peer loss provides a way to simplify model development when facing potentially noisy training labels, and can be promoted as a robust candidate loss function in such situations.
\end{abstract}
\section{Introduction}
The quality of supervised learning models depends on the quality of the training dataset $\{(x_n,y_n)\}_{n=1}^N$.
In practice, label noise can arise due to a host of reasons. For instance, the observed labels $\Tilde{y}_n$s may represent human observations of a ground truth label. In this case, human annotators may observe the label imperfectly due to differing degrees of expertise or measurement error, see e.g., medical examples such as labeling MRI images from patients. There exist extensive prior works in the literature that aim to develop algorithms to learn models that are robust to label noise \citep{bylander1994learning,cesa1999sample,cesa2011online,ben2009agnostic,scott2013classification,natarajan2013learning,scott2015rate}. 
Typical solutions that have theoretical guarantees often require \emph{a priori} knowledge of \emph{noise rates}, i.e., a set of parameters that control the severity of label noise.
Working with unknown noise rates is difficult in practice:
Usually, one must estimate the noise rates from data, which may require additional data collection or requirement \citep{natarajan2013learning,scott2015rate,van2015learning} (e.g., a set of ground truth labels for tuning these parameters) and may introduce estimation error that can affect the final model in less predictable ways. Our main goal is to provide an alternative that does not require the specification of the noise rates, nor an additional estimation step for the noise. This target solution benefits the practitioner when he or she does not have access to reliable estimates of the noise rates (e.g., when the training data has a limited size for the estimation tasks, or when the training data is already collected in a form that makes the estimation hard to perform).

In this paper, we introduce a new family of loss functions, \emph{peer loss functions}, to empirical risk minimization (ERM), for a broad class of learning with noisy labels problems. Peer loss functions operate under different noise rates without requiring either \emph{a priori} knowledge of the embedded noise rates, or an estimation procedure. This family of loss functions builds on approaches developed in the \emph{peer prediction} literature \citep{MRZ:2005,dasgupta2013crowdsourced,shnayder2016informed}, which studies how to elicit information from self-interested agents without verification.
Results in the peer prediction literature focused on designing scoring functions to score each reported data using another noisy reference answer, without accessing ground truth information. We borrow this idea and the associated scoring functions via making a connection through treating each classifier's predictions as an agent's private information to be elicited and evaluated, and the noisy labels as imperfect reference answers reported from a ``noisy label agent". \rev{The specific form of peer loss evaluates classifiers' prediction using noisy labels on both the samples to-be-evaluated and carefully constructed ``peer" samples. The evaluation on the constructed peer sample encodes \emph{implicitly} the information about the noise as well as the underlying true labels, which helps us offset the effects of label noise. The peer sample evaluation returns us a favorable property that the expected risk of peer loss computed on the noisy distribution turns to be an affine transformation of the true risk of the classifier defined on the clean distribution. In other words, peer loss is invariant to label noise when optimizing with it. This effect helps us get rid of the estimation of noise rates.}

The main contributions of this work are:
\squishlist
    \item[1.] \rev{
    We propose a new family of loss functions that can easily adapt to existing ERM framework that i) is robust to \emph{asymmetric} label noise with formal theoretical guarantees and ii) requires no prior knowledge or estimation of the noise rates (\emph{no need for specifying noise rates}). We believe having the second feature above is non-trivial progress, and it features a promising solution to deploy in an unknown noisy training environment.}
    \item[2.] We present formal results showing that performing ERM with a peer loss function can recover an optimal, or a near-optimal classifier $f^*$ as if performing ERM on the clean data (Theorem \ref{THM:EQUAL}, \ref{THM:pneq}, \ref{THM:weightedPeer}). We also provide peer loss functions' risk guarantees (Theorem \ref{THM:converge1}, \ref{THM:GEN}).
    
    \item[3.] \rev{We present extensive experimental results to validate the usefulness of peer loss functions (Section \ref{sec:exp} and Appendix). This result is encouraging as it demonstrates the practical effectiveness in removing the requirement of error rates of noise before many of the existing training methods can be applied. We also provide preliminary results on how peer loss generalizes to multi-class classification problems.}
    
    \item[4.] Our implementation of peer loss functions is available at \hyperlink{https://github.com/gohsyi/PeerLoss}{https://github.com/gohsyi/PeerLoss}. 
    
\squishend

Due to space limit, the full version of this paper with all proof and experiment details can be found in \citep{Liu_2019}.

\subsection{Related Work}
We go through the most relevant works.\footnote{We provide more detailed discussions in the Appendix.} 

\paragraph{Learning from Noisy Labels}
Our work fits within a stream of research on learning with noisy labels.
A large portion of research on this topic works with the \emph{random classification noise} (RCN) model, where observed labels are flipped independently with probability $\in [0,\tfrac{1}{2}]$ 
\citep{bylander1994learning,cesa1999sample,cesa2011online,ben2009agnostic}. Recently, learning with asymmetric noisy data (or also referred as \emph{class-conditional} random classification noise (CCN)) for binary classification problems has been rigorously studied in \citep{stempfel2009learning,scott2013classification,natarajan2013learning,scott2015rate,van2015learning,menon2015learning}.

For RCN, where the noise parameters are symmetric, there exist works that show symmetric loss functions \citep{manwani2013noise,ghosh2015making,ghosh2017robust,van2015learning} are robust to the underlying noise, without specifying the noise rates.
Our focus departs from this line of works and we exclusively focus on asymmetric noise setting, and study the possibility of an approach that can ignore the knowledge of noise rates. Follow-up works include \citep{du2013clustering,van2015average,menon2015learning,charoenphakdee2019symmetric}.

\paragraph{More Recent Works}
 More recent developments include an importance re-weighting algorithm \citep{liu2016classification}, a noisy deep neural network learning setting \citep{sukhbaatar2014learning,han2018co,song2019selfie}, and learning from massive noisy data for image classification \citep{xiao2015learning,goldberger2016training,zhang2017mixup,jiang2017mentornet,jenni2018deep,yi2019probabilistic}, robust cross entropy loss for neural network \citep{zhang2018generalized}, loss correction \citep{patrini2017making}, among many others. Loss or sample correction has also been studied in the context of learning with unlabeled data with weak supervisions \citep{lu2018minimal}. Most of the above works either lacks theoretical guarantees of the proposed method against asymmetric noise rates \citep{sukhbaatar2014learning,zhang2018generalized}, or require estimating the noise rate (or transition matrix between the noisy and true labels) \citep{liu2016classification,xiao2015learning,patrini2017making,lu2018minimal}.
 
A recent work \citep{xu2019l_dmi} proposes an information theoretical loss, an idea adapted from an earlier theoretical contribution \citep{kong2018water}, which is also robust to asymmetric noise rates. We aimed for a simple-to-optimize loss function that can easily adapt to existing ERM solutions. 
\paragraph{Peer Prediction}
Our work builds on the literature of peer prediction \citep{prelec2004bayesian,MRZ:2005,witkowski2012robust,radanovic2013,Witkowski_hcomp13,dasgupta2013crowdsourced,shnayder2016informed,sub:ec17}.
Most relevant to us is \citep{dasgupta2013crowdsourced,shnayder2016informed} where a correlated agreement (CA) type of mechanism was proposed. CA evaluates a report's correlations with another reference agent - its specific form inspired our peer loss.

\section{Preliminaries}
Suppose $(X,Y) \in \mathcal X \times \mathcal Y$ are drawn from a joint distribution $\mathcal D$, with their marginal distributions denoted as $\PP_X, \PP_Y$. We assume $\mathcal X \subseteq \mathbb R^d$, and $\mathcal Y = \{-1,+1\}$, that is we consider a binary classification problem. Denote by $p:=\PP(Y=+1) \in (0,1)$. There are $N$ training samples $(x_1,y_1),...,(x_N, y_N)$ drawn i.i.d. from $\mathcal D$. For positive integer $n$, denote by $[n]:=\{1,2,...,n\}$. 

Instead of observing $y_n$s, the learner can only collect a noisy set of training labels $\tilde{y}_n$s, generated according to $y_n$s and a certain error rate model; that is we observe a dataset $\{(x_n,\tilde{y}_n)\}_{n=1}^N.$
We assume uniform error for all the training samples we collect, in that errors in $\tilde{y}_n$s follow the same error rate model: denoting the random variable for noisy labels as $\tilde{Y}$ and we define 
\begin{align*}
    &e_{+1} := \PP(\tilde{Y} = -1| Y = +1),~e_{-1} := \PP(\tilde{Y} = +1| Y = -1)
\end{align*}
Label noise is conditionally independent from the features, that is the error rate is uniform across $x_n$s:
$\PP(\tilde{Y} = y'| Y = y) =\PP(\tilde{Y} = y'| X, Y = y), \forall y,y' \in \{-1,+1\}.
$

We assume $0 \leq e_{+1} + e_{-1} <1$ - this condition is not unlike the ones imposed in the existing learning literature \citep{natarajan2013learning}, and it simply implies that the noisy labels are positively correlating with the true labels (informative about the true labels).  Denote the distribution of the noisy data $(X,\tilde{Y})$ as $\tilde{\mathcal D}$. 

$f: \mathcal X \rightarrow \mathbb R$ is a real-valued decision function, and its risk w.r.t. the 0-1 loss is defined as $\underline{\mathbb E_{(X,Y) \sim \mathcal D}[\BR(f(X),Y)]}$.
The Bayes optimal classifier $f^*$ is the one that minimizes the 0-1 risk: 
$
\underline{f^* = \text{argmin}_f ~\mathbb E_{(X,Y) \sim \mathcal D}[\BR(f(X),Y)]}
$.
Denote this optimal risk as $R^*$.
Instead of minimizing the above 0-1 risk, the learner often seeks a surrogate loss function $\ell:  \mathbb R \times \{-1,+1\} \rightarrow \mathbb R_{+}$, and finds a $f \in \F$ that minimizes the following error:
$
\underline{\mathbb E_{(X,Y) \sim \mathcal D}[\ell(f(X),Y)]}
$. $\mathcal F$ is the hypothesis space for $f$. Denote the following measures: $R_{\mathcal D}(f) = \mathbb E_{(X,Y) \sim \mathcal D}[\BR(f(X),Y)]$ and $ R_{\ell,\mathcal D}(f) = \mathbb E_{(X,Y) \sim \mathcal D}[\ell(f(X),Y)]$. 

 
 When there is no confusion, we will also short-hand $\mathbb E_{(X,Y) \sim \mathcal D}[\ell(f(X),Y)]$ as $\mathbb E_{\mathcal D}[\ell(f(X),Y)]$.
Denoting $D$ a dataset collected from distribution $\mathcal D$ (correspondingly $\tilde{D}:=\{(x_n,\tilde{y}_n)\}_{n \in [N]}$ from $\tilde{\mathcal D}$), the empirical risk measure for $f$ is defined as
\underline{$\hat{R}_{\ell,D}(f) = \frac{1}{|D|} \sum_{(x,y) \in D}\ell(f(x),y)~.
$} 

\subsection{Learning with Noisy Labels}

Typical methods for learning with noisy labels include developing noise correction surrogates loss function to learn with noisy data \citep{natarajan2013learning}. For instance, \cite{natarajan2013learning} tackles this problem by defining the following \emph{un-biased surrogate loss functions} over $\ell$ to help ``remove" noise in expectation:
$
\tilde{\ell}(t,y) := \frac{(1-e_{-y}) \cdot \ell (t,y) - e_{y}\cdot \ell(t,-y)}{1-e_{-1}-e_{+1}}, \forall t,y.
$
$\tilde{\ell}$ is identified such that when a prediction is evaluated against a noisy label using this surrogate loss function, the prediction is as if evaluated against the ground-truth label using $\ell$ in expectation. Hence the loss of the prediction is ``unbiased'', that is $\forall$ prediction $t$, $\mathbb E_{\tilde{Y}|y}[\tilde{\ell}(t,\tilde{Y})] = \ell(t,y)$ [Lemma 1, \citep{natarajan2013learning}].

One important note to make is most, if not all, existing solutions require the knowledge of the error rates $e_{-1},e_{+1}$. Previous works either assumed the knowledge of it, or needed additional assumptions, clean labels or redundant noisy labels to estimate them. This becomes the bottleneck of applying these great techniques in practice. Our work is also motivated by the desire to remove this limitation. 

\subsection{Peer Prediction}\label{sec:ca}

Peer prediction is a technique developed to truthfully elicit information when there is no ground truth verification.
Suppose we are interested in eliciting private observations about a binary event $Y \in \{-1,+1\}$ generated according to a random variable $Y$.
There are $K$ agents indexed by $[K]$. Each of them holds a noisy observation of the truth $Y=y$, denoted as $y^A \in \{-1,+1\},\, A \in [K]$. We would like to elicit the $y^A$s, but they are completely private and we will not observe $y$ to evaluate agents' reports. Denote by $r^A$ the reported data from each agent $A$. $r^A \neq y^A$ if agents are not compensated properly for their information. 

Results in \emph{peer prediction} have proposed scoring or reward functions that evaluate an agent's report using the reports of other peer agents. For example, a peer prediction mechanism may reward agent $A$ for her report $r^A$ using $S(r^A, r^B)$ where $r^B$ is the report of a randomly selected reference agent $B \in [K]\backslash \{A\}$.
The scoring function $S$ is designed so that truth-telling is a strict Bayesian Nash Equilibrium (implying other agents truthfully report their $y^B$), that is,
$
\mathbb E_{y^B}[S( y^A, y^B )|y^A] > 
\mathbb E_{y^B}[S(r^A, y^B)| y^A],~\forall r^A \neq y^A.
$

\textbf{Correlated Agreement} \citep{shnayder2016informed,dasgupta2013crowdsourced} (CA) is an established peer prediction mechanism for a multi-task setting\footnote{We provide other examples of peer prediction functions in the Appendix.}. CA is also the core and the focus of our subsequent sections on developing peer loss functions. This mechanism builds on a $\Delta$ matrix that captures the stochastic correlation between the two sources of predictions $y^A$ and $y^B$. 
Denote the following relabeling function: $g(1) = -1, g(2) = +1$, $\Delta \in \mathbb R^{2 \times 2}$ is a squared matrix with its entries defined as follows: $\forall ~k,l=1,2$
\begin{small}
\begin{align*}
    &\Delta_{k,l}  = \PP\bigl(y^A=g(k),y^B=g(l)\bigr)- \PP\bigl(y^A = g(k)\bigr) \PP\bigl(y^B= g(l)\bigr), 
\end{align*}
\end{small}
The intuition of above $\Delta$ matrix is that each $(k,l)$ entry of $\Delta$ captures the marginal correlation between the two predictions $y^A$ and $y^B$. When there is no confusion in the text, we will always follow this relabeling function to map a $-1$ label to $1$ and $+1$ to $2$ when defining or calling an entry in the $\Delta$ matrix without explicitly spelling out $g(\cdot)$, that is we will write
\begin{align*}
    &\Delta_{k,l}  = \PP\bigl(y^A=k,y^B=l\bigr)- \PP\bigl(y^A = k\bigr) \cdot \PP\bigl(y^B= l\bigr), 
\end{align*}
as well as 
$$
\Delta_{y,y'}  = \PP\bigl(y^A=y,y^B=y'\bigr)- \PP\bigl(y^A = y\bigr) \cdot \PP\bigl(y^B= y'\bigr).
$$
We further define $M: \{-1,+1\} \times \{-1,+1\} \rightarrow \{0,1\}$ as the sign matrix of $\Delta$: 
\begin{align}
M(y, y') =:  \text{Sgn}\left(\Delta_{y,y'}\right), \label{eqn:ms}
\end{align}
$\text{where}~\text{Sgn}(x)=1, x > 0; ~\text{Sgn}(x)=0,$ otherwise. 

CA requires each agent $A$ to perform multiple tasks: denote agent $A$'s predictions for the $N$ tasks as $y^A_1,\dots,y^A_N$. 
Ultimately the scoring function $S(\cdot)$ for each task $n$ that is shared between $A,B$ is defined as follows: randomly draw two tasks $n_1,n_2~, n_1 \neq n_2$,
\begin{align*}
S\bigl(y^A_n,y^B_n\bigr) :=& M\bigl(y^A_n, y^B_n\bigr) - M\bigl(y^A_{n_1},y^B_{n_2}\bigr).
\end{align*}
A key difference between the first and second $M(\cdot)$ terms is that the second term is defined for two independent peer tasks $n_1,n_2$ (as the reference answers). It was established in \citep{shnayder2016informed} that CA is truthful at a Bayesian Nash Equilibrium (Theorem 5.2, \cite{shnayder2016informed}.) \footnote{To be precise, it is an informed truthfulness. We refer interested readers to \citep{shnayder2016informed} for details.}; in particular, if $y^B$ is \emph{categorical} w.r.t. $y^A$:
$
\PP(y^B=y'|y^A = y) < \PP(y^B=y'), \forall A,B \in [K],~y' \neq y
$
then $S(\cdot) $ is strictly truthful (Theorem 4.4,  \cite{shnayder2016informed}). 

\section{Learning with Noisy Labels: a Peer Prediction Approach}\label{sec:pp}

In this section, we show that peer prediction scoring functions, when specified properly, will adopt Bayes optimal classifier as their maximizers (or minimizers for the corresponding loss form). 

\rev{\subsection{Learning with Noisy Labels as an Elicitation Problem}}
We first state our problem of learning with noisy labels as a peer prediction problem.
The connection is made by firstly rephrasing the two data sources, the classifiers' predictions and the noisy labels, from agents' perspective. For a task $Y \in \{-1, +1\}$, say $+1$ for example, denote the noisy labels $\tilde{Y}$ as $\R(X), X \sim \PP_{X|Y=1}$. In general, $\R(X)$ can be interpreted as the agent that ``observes" $\tilde{y}_1,...,\tilde{y}_N$ for a set of randomly drawn feature vectors $x_1,...,x_N$:
$
\tilde{y}_n \sim \R(X)$. Denote the following error rates for the agent's observations (similar to the definition of $e_{+1},e_{-}$):
$
    {\PP}(\R(X) = -1| Y = +1) = e_{+1},~
    {\PP}(\R(X) = +1| Y = -1) = e_{-1}.
$
There is another agent whose observations ``mimic'' the Bayes optimal classifier $f^*$. Again denote this optimal classifier agent as $\R^*(X):=f^*(X)$:
$
    {\PP}(\R^*(X) = -1| Y = +1) = e^*_{+1},~
{\PP}(\R^*(X) = +1| Y = -1) = e^*_{-1}.
$
\begin{figure}[!ht]
\centering
\includegraphics[width=0.7\linewidth]{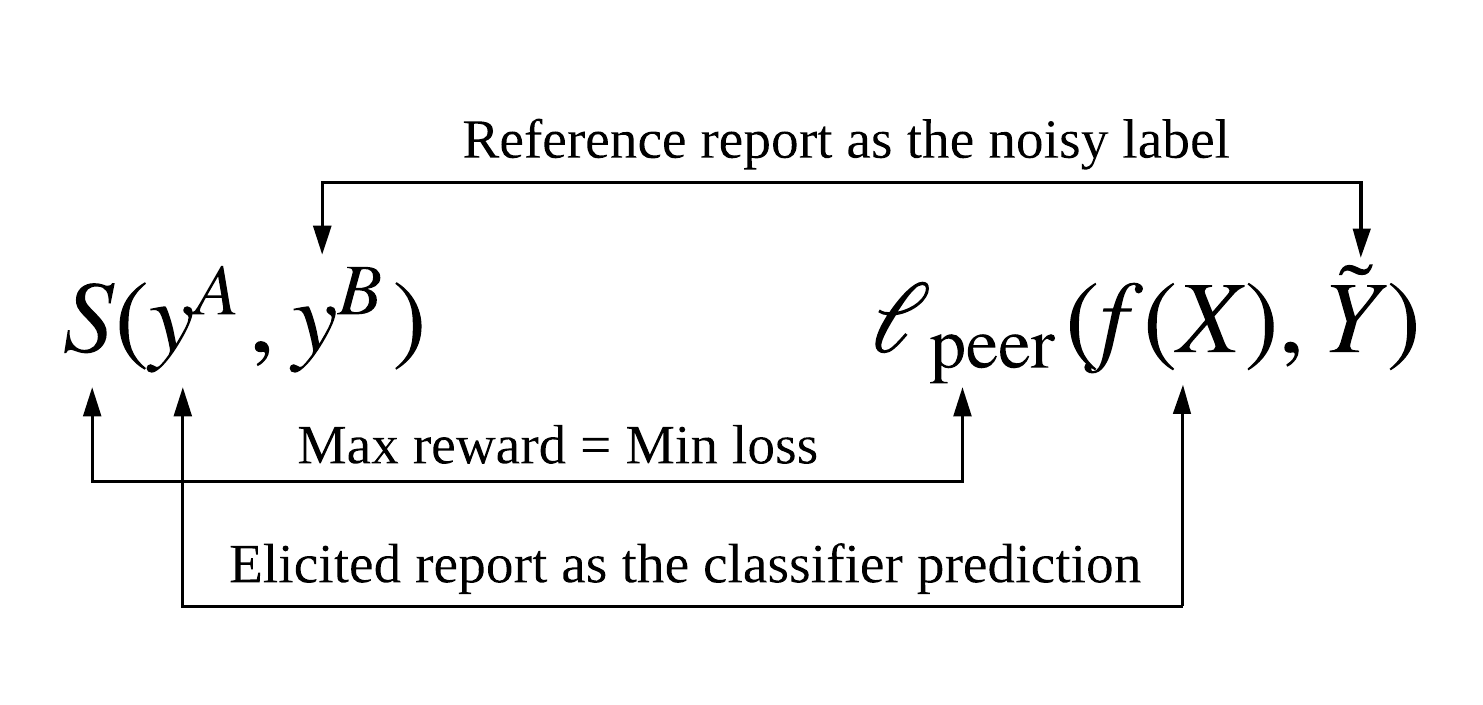}
\vspace{-0.1in}
\caption{$S$ is the peer prediction function; $\lPeer$ is to ``evaluate" a classifier's prediction using a noisy label.}\label{mechanism:ill}
\end{figure}

Suppose we would like to elicit predictions from the optimal classifier agent $\R^*$, while the reports from the noisy label agent $\R$ will serve as the reference reports. Both $\R$ and $\R^*$ are randomly assigned a task $X=x$, and each of them observes a signal $\R(x)$ and $\R^*(x)$ respectively. Denote the report from agent $\R^*$ as $r^*$. A scoring function $S:\mathbb R \times \mathbb R \rightarrow \mathbb R$ is called to induce truthfulness if the following fact holds: $\forall r^*(X) \neq \R^*(X)$,
\begin{align}
\mathbb E_{ X}\bigl [S\bigl(\R^*(X),\R(X)\bigr)\bigr] \geq \mathbb E_{X}\bigl [S\bigl (r^*(X),\R(X)\bigr)\bigr]. \label{eqn:truthful}
\end{align}
Taking the negative of $S(\cdot)$ (changing a reward score one aims to maximize to a loss to minimize) we also have
$$
\mathbb E_{ X}\bigl [-S\bigl(\R^*(X),\R(X)\bigr)\bigr] \leq \mathbb E_{X}\bigl [-S\bigl (r^*(X),\R(X)\bigr)\bigr],
$$ implying when taking $-S(\cdot)$ as the loss function, minimizing $-S(\cdot)$ w.r.t. $\R$ will the Bayes optimal classifier $f^*$. 

Our idea is summarized in Figure \ref{mechanism:ill}.

\rev{\subsection{Peer Prediction Mechanisms Induce Bayes Optimal Classifier}}
When there is no ambiguity, we will shorthand $\R(X), \R^*(X)$ as $\R,\R^*$, with keeping in mind that $\R,\R^*$ encode the randomness in $X$. In the elicitation setting, a potentially misreported classifier $f(X)$ only disagrees with $f^*(X)$ according to its local observation $f^*(X)$ but not $Y$ (unobservable), that is $\PP(f(X)\neq f^*(X)|f^*(X)=l,Y=+1)=\PP(f(X)\neq f^*(X)|f^*(X)=l,Y=-1), ~l \in \{-1,+1\}$. Denote this reporting space of $f$ as $\mathcal F_{\textsf{report}}$. Suppose $\R^*$ has the correct prior $p$ of $Y$. Then we have:
\begin{theorem}\label{THM:MAIN} 
 Suppose $S(\cdot)$ induces truthful $f^*$ (Eqn. (\ref{eqn:truthful})), that is $S(\cdot)$ is able to elicit the Bayes optimal classifier $f^*$ (agent $\R^*$) using $\R$. Then
$
f^* = \text{argmin}_{f \in \mathcal F_{\textsf{report}}}~ \mathbb E_{(X,\tilde{Y}) \sim \tilde{\mathcal D}}\bigl[-S(f(X),\tilde{Y})\bigr].
$
\end{theorem}
This proof can be done via showing that any non-optimal Bayes classifier corresponds to a non-truthful misreporting strategy.
We emphasize that it is not super restrictive to have a truthful peer prediction scoring function $S$. We provide discussions in Appendix. \rev{Theorem \ref{THM:MAIN} provides a conceptual connection and can serve as an anchor point when connecting a peer prediction score function to the problem of learning with noisy labels. So far we have not discussed a specific form of how we construct a loss function using ideas from peer prediction, and have not mentioned the requirement of knowing the noise rates. We will provide the detail about a particular \emph{peer loss} in the next section, and explain its independence of noise rates. }

\section{Peer Loss Function}\label{sec:main}
\rev{We now present peer loss, a family of loss functions inspired by a particular peer prediction mechanism, the correlated agreement (CA), as presented in Section \ref{sec:ca}. We are going to show that peer loss is able to induce the minimizer of a hypothesis space $\F$, under a broad set of non-restrictive conditions. In this Section, we do not restrict to Bayes optimal classifiers, nor do we impose any restrictions on the loss functions' elicitation power.}

\rev{
\subsection{Preparing CA for Noisy Learning Problem}}
To give a gentle start, we repeat the setting of CA for our classification problem. 

\paragraph{$\Delta$ and scoring matrix} First recall that $\Delta \in \mathbb R^{2 \times 2}$ is a squared matrix with entries defined between $\R^*$ (the $f^*$) and $\R$ (i.e., the noisy labels $\tilde{Y}$): $\forall k,l=1,2$
\begin{align*}
&\Delta_{k,l}  = \PP\bigl(f^*(X)=k,\tilde{Y}=l\bigr) -  \PP\bigl(f^*(X)= k\bigr)\PP\bigl(\tilde{Y} = l\bigr), ~
\end{align*}
\rev{$\Delta$ characterizes the ``marginal" correlations between the optimal classifier' prediction and the noisy label $\tilde{Y}$. Then the following scoring matrix $M$ is computed using $\text{Sgn}(\Delta)$, the sign matrix of $\Delta$.}
\begin{example}
\label{exa:delta}
Consider a binary class label case: $\PP(Y=-1) = 0.4, \PP(Y=+1) = 0.6$, the noise in the labels are $e_{-1} = 0.3, e_{+1} = 0.4$ and $e^*_{-1} = 0.2, e^*_{+1} = 0.3$. Then we have
\rev{
$
\Delta_{1,1} = 0.036, ~~\Delta_{1,2} = -0.036, ~~\Delta_{2,1} = -0.036, ~~ \Delta_{2,2} = 0.036.
$ The details of the calculation can be found in the Appendix.
And:
}
$\Delta = 
\begin{bmatrix}
    0.036     & -0.036\\
    -0.036    &  0.036  
\end{bmatrix}
\Rightarrow  \text{Sgn}(\Delta)  = \begin{bmatrix}
    1 & 0\\
    0  & 1 
\end{bmatrix}.
$ 
\end{example}

\paragraph{Peer samples} For each sample $(x_n,\tilde{y}_n)$, randomly draw another two samples $
(x_{n_1},\tilde{y}_{n_1}),(x_{n_2},\tilde{y}_{n_2})
$
such that $n_1\neq n_2.$ We will name $(x_{n_1},\tilde{y}_{n_1}),(x_{n_2},\tilde{y}_{n_2})$ as $n$'s peer samples. \rev{After pairing $x_{n_1}$ with $\tilde{y}_{n_2}$ (two independent instances), the scoring function $S(\cdot)$ for each sample point $x_n$ is defined as follows: 
$$
S(f(x_n),\tilde{y}_n))
=M\bigl(f(x_n),\tilde{y}_n\bigr) - M\bigl(f(x_{n_1}),\tilde{y}_{n_2}\bigr)
.
$$ 
Define loss function $\tilde{\ell}(\cdot)$ as the negative of $S(\cdot)$, which we will name as the (\textbf{Generic Peer Loss})
\begin{small}
\begin{align}
\tilde{\ell}\bigl(f(x_n),&\tilde{y}_n\bigr):= \bigl (1- M \bigl(f(x_n),\tilde{y}_n\bigr)\big) - \bigl (1- M\bigl(f(x_{n_1}),\tilde{y}_{n_2}\bigr)\bigr).\label{eqn:peerloss1}
\end{align}
\end{small}
The first term above evaluates the classifier's prediction on $x_n$ using noisy label $\tilde{y}_n$, and the second ``peer" term defined on two independent tasks $n_1,n_2$ ``punishes'' the classifier from overly agreeing with the noisy labels. We will see this effect more clearly.}



\rev{
\subsection{Peer Loss}

We need to know $\text{Sgn}(\Delta)$ in order to specify $M$ and $\tilde{\ell}$, which requires certain information about $f^*$ and $\tilde{Y}$. We show that Example 1 is not a special case, and for the scenarios that the literature is broadly interested in, $\text{Sgn}(\Delta)$ is simply the identify matrix:} 
\begin{lemma}
When $e_{-1}+e_{+1} < 1$, we have $\text{Sgn}(\Delta) = I_{2 \times 2}$, the identity matrix.\label{LEMMA:SGN}
\end{lemma}
\rev{The above implies that for $\Delta_{k,k}, k=1,2$, $f^*$ and $\tilde{Y}$ are positively correlated, 
so the marginal correlation is positive; while for off-diagonal entries, they are negatively correlated. 
}

\paragraph{Peer Loss} When $\text{Sgn}(\Delta) = I_{2 \times 2}$, $M(y,y') = 1$ if $y=y'$, and 0 otherwise. $\tilde{\ell}(\cdot)$ defined in Eqn. (\ref{eqn:peerloss1}) reduces to the following form:
\begin{align}
\BR_\text{peer}(f(x_n), \tilde{y}_n) = \BR(f(x_n), \tilde{y}_n) - \BR(f(x_{n_1}), \tilde{y}_{n_2})\label{eqn:pl}
\end{align}
To see this, for instance 
$1- M \bigl(f(x_n)=+1,\tilde{y}_n=+1\bigr) 
 = 1 - M(2,2) = 1-1=0 
 = \BR(f(x_n)=+1, \tilde{y}_n=+1)    
$.
Replacing $\BR(\cdot)$ with any generic loss $\ell(\cdot)$ we define:
\begin{align}
  \ell_\text{peer}(f(x_n), \tilde{y}_n) = \ell(f(x_n), \tilde{y}_n) - \mathbb \ell(f(x_{n_1}), \tilde{y}_{n_2})\label{def:lpeer}
\end{align}
We name the above loss as \emph{peer loss}. This strikingly simple form of $\ell_\text{peer}(f(x_n), \tilde{y}_n)$ implies that knowing $e_{-1}+e_{+1} < 1$ holds is all we need to specify $\ell_\text{peer}$.

Later we will show this particular form of loss is invariant under label noise, which gives peer loss the ability to drop the requirement noise rates. We will instantiate this argument formally with Lemma \ref{LEM:AFFINE} and establish a link between the above measure and the true risk of a classifier on the clean distribution. The rest of presentation focuses on $\ell_{\text{peer}}$ (Eqn. (\ref{def:lpeer})), but $\ell_{\text{peer}}$ recovers $\BR_\text{peer}$ via replacing $\ell$ with $\BR$.


\paragraph{ERM with Peer Loss} Performing ERM with peer loss returns us $\hat{f}^*_{\lPeer} $:
\begin{align}
    \hat{f}^*_{\lPeer} 
    &= \argmin_{f \in \F}\frac{1}{N}\sum_{n=1}^N \lPeer(f(x_n),\tilde{y}_n)
\end{align}

\rev{Note again that the definition of $\lPeer$ does not require the knowledge of either $e_{+1},e_{-1}$ or $e^*_{+1}, e^*_{-1}$.}

\subsection{Property of Peer Loss}
We now present a key property of peer loss, which shows that its risk over the noisy labels is simply an affine transformation of its true risk on clean data. We denote by $\mathbb E_{\mathcal D}[\ell_\text{peer}(f(X), Y)]$ the expected peer loss of $f$ when $(X,Y)$, as well as its peer samples, are drawn i.i.d. from distribution $\mathcal D$. 
\begin{lemma} \label{LEM:AFFINE}
Peer loss is invariant to label noise:
\[
\mathbb E_{\tilde{\mathcal D}}[\ell_\text{peer}(f(X), \tilde{Y})] = (1-e_{-1}-e_{+1}) \cdot \mathbb E_{\mathcal D}[\ell_\text{peer}(f(X), Y)].
\]
\end{lemma}

\rev{The above Lemma states that peer loss is invariant to label noise in expectation. We have also empirically observed this effect in our experiment. Therefore minimizing it over noisy labels is equivalent to minimizing over the true clean distribution. The theorems below establish the connection between $\mathbb E_{\mathcal D}[\ell_\text{peer}(f(X), Y)]$, the expected peer loss over clean data, with the true risk: 

Denote 
$
\tilde{f}^*_{\BPeer} = \argmin_{f \in \F} R_{\BPeer, \tilde{\mathcal D}}(f).$ With Lemma \ref{LEM:AFFINE}, we can easily prove the following:} 

\begin{theorem}\label{THM:EQUAL}
\rev{[Optimality guarantee with equal prior]} When $p = 0.5$, $ \tilde{f}^*_{\BPeer} \in \argmin_{f \in \F} R_{\mathcal D}(f)$.
\end{theorem}
The above theorem states that for a class-balanced dataset with $p=0.5$, peer loss induces the same minimizer as the one that minimizes the 0-1 loss on the clean data. Removing the constraint of $\F$, i.e., $\tilde{f}^*_{\BPeer} = \argmin_{f} R_{\BPeer, \tilde{\mathcal D}}(f) \Rightarrow \tilde{f}^*_{\BPeer}= f^*$. In practice we can balance the dataset s.t. $p\rightarrow 0.5$. 

When $p \neq 0.5$, denote $\delta_p = \mathbb P(Y=+1) - \mathbb P(Y=-1)$, we prove:
\begin{theorem}\label{THM:pneq}
\rev{[Approximate optimality guarantee with unequal prior]} When $p \neq 0.5$, 
$|R_{\mathcal D}(\tilde{f}^*_{\BPeer}) - \min_{f \in \mathcal F}R_{\mathcal D}(f)| \leq |\delta_p| $. 
\end{theorem}
When $|\delta_p|$ is small, i.e., $p$ is closer to $0.5$, this bound becomes tighter.

\paragraph{Multi-class extension} Our results in this section are largely generalizable to the multi-class classification setting. Suppose we have $K$ classes of labels, denoting as $\{1,2,...,K\}$. 
One can show that for many classes of noise matrices, the $M(\cdot)$ matrix is again an identify matrix. This above fact will help us reach the conclusion that minimizing peer loss leads to the same minimizer on the clean data. We provide experiment results for multi-class tasks in Section \ref{sec:exp}.

\paragraph{Why do we not need the knowledge of noise rates explicitly?} Both of the terms $\BR(f(x_n), \tilde{y}_n)$ and $\BR(f(x_{n_1}), \tilde{y}_{n_2})$ encoded the knowledge of noise rates \textbf{\emph{implicitly}}. The carefully constructed form as presented in Eqn. (\ref{eqn:pl}) allows peer loss to be invariant against noise (Lemma \ref{LEM:AFFINE}, a property we will explain later). For a preview, for example if we take expectation of $\BR_\text{peer}(f(x_n)=+1, \tilde{y}_n=+1)$ we will have
\begin{align*}
    &\E \left[\BR_\text{peer}(f(x_n)=+1, \tilde{y}_n=+1) \right] \\
    &= \PP(f(X)=+1,\tilde{Y}=+1) - \PP(f(X)=+1)\PP(\tilde{Y}=+1),
\end{align*}
the marginal correlation between $f$ and $\tilde{Y}$, which is exactly capturing the entries of $\Delta$ defined between $f$ and $\tilde{Y}$! The second term above is a product of marginals because of the independence of peer samples $n_1,n_2$. Using the constructed peer term is all we need to recover this information measure in expectation. In other words, both the joint and marginal product distribution terms encode the noise rate information in an implicit way.

\subsection{$\alpha$-weighted Peer Loss}

We take a further look at the case with $p \neq 0.5$. Denote by $R_{+1}(f) = \mathbb P(f(X)=-1|Y=+1), ~R_{-1}(f) = \mathbb P(f(X)=+1|Y=-1)$. It is easy to prove:
\begin{lemma}\label{LEM:alpha1}
Minimizing \(\mathbb E[\BR_\text{peer}(f(X), \tilde{Y})]\) is equivalent to 
minimizing \(R_{-1}(f)+R_{+1}(f)\).
\end{lemma}

However, minimizing the true risk \(R_{\mathcal D}(f)\) is equivalent to minimizing \(p\cdot R_{+1}(f) + (1-p)\cdot R_{-1}(f)\), a weighted sum of \(R_{+1}(f)\) and \(R_{-1}(f)\). The above observation and the failure to reproduce the strong theoretical guarantee when $p \neq 0.5$ motivated us to study a $\alpha$-weighted version of peer loss, to make peer loss robust to the case $p \neq 0.5$. We propose the following \emph{$\alpha$-weighted peer loss} via adding a weight \(\alpha \geq 0\) to the second term, the peer term: 
\begin{align*}
  \lAlphaPeer \bigl(f(x_n), \tilde{y}_n\bigr) = \ell(f(x_n), \tilde{y}_n)
  &- \alpha \cdot \mathbb \ell(f(x_{n_1}), \tilde{y}_{n_2})
\end{align*}
Denote $\AlphaPeer$ as $\lAlphaPeer $ when $\ell = \BR$, $
\tilde{f}^*_{\AlphaPeer} = \argmin_{f \in \F} R_{\AlphaPeer, \tilde{\mathcal D}}(f)
$ as the optimal classifier under $\AlphaPeer$, and $\delta_{\tilde{p}} = \mathbb P(\tilde{Y}=+1) - \mathbb P(\tilde{Y}=-1)$. 
Then when $\delta_{\tilde{p}} \neq 0$ (when this condition does not hold, we can perturb the training data by downsampling one of the two classes according to the noisy labels.), we prove:
\begin{theorem}\label{THM:weightedPeer}
Let
$
\alpha 
= 1 - (1-e_{-1}-e_{+1})\cdot\frac{\delta_p}{\delta_{\tilde{p}}}
$. We have $ \tilde{f}^*_{\AlphaPeer} \in  \argmin_{f \in \F} R_{\mathcal D}(f)$.
\end{theorem}
Denote $
\alpha^* 
:= 1 - (1-e_{-1}-e_{+1})\cdot\frac{\delta_p}{\delta_{\tilde{p}}}
$. Several remarks follow: 
(1) When $p= 0.5$, $\delta_p = 0$, we have \(\alpha^* = 1\), i.e. we recover the earlier definition of $\lPeer$. 
(2) When \(e_{-1} = e_{+1}\), \(\alpha^* = 0\) (see Appendix for details), we recover the $\ell$ for the clean learning setting, which has been shown to be robust under symmetric noise rates \citep{manwani2013noise,van2015learning}. 
(3) When the signs of \(\mathbb P(Y=1)-\mathbb P(Y=-1)\) and \(\mathbb P(\tilde{Y}=1)-\mathbb P(\tilde{Y}=-1)\) are the same, \(\alpha^* < 1\). Otherwise, \(\alpha^* > 1\). In other words, when the label noise changes the relative quantitative relationship of \(\mathbb P(Y=1)\) and \(\mathbb P(Y=-1)\), \(\alpha^* > 1\) and vice versa. 
(4) Knowing $\alpha^*$ requires a certain knowledge of $e_{+1},e_{-1}$ when $p \neq 0.5$. Though we do not claim this knowledge, this result implies tuning \(\alpha^*\) (using validation data) may improve the performance.

Theorem \ref{THM:EQUAL} and \ref{THM:weightedPeer} and sample complexity theories imply that performing ERM with $\AlphaStar$: $\hat{f}^*_{\AlphaStar} = \argmin_{f}\hat{R}_{\AlphaStar,\tilde{D}}(f)$ converges to $f^*$:
\begin{theorem}\label{THM:converge1}
With probability at least $1-\delta$, $$R_{\mathcal D}(\hat{f}^*_{\AlphaStar}) - R^* \leq \frac{1+\alpha^*}{1-e_{-1}-e_{+1}}\sqrt{\frac{2\log 2/\delta}{N}}.$$
\end{theorem}
\subsection{Calibration and Generalization}
So far our results focused on minimizing 0-1 losses, which is hard in practice.
We provide evidence of $\lPeer$'s, and $\lAlphaPeer$'s in general, calibration and convexity for a generic and differentiable calibrated loss. We consider a $\ell$ that is classification calibrated, convex and $L$-Liptchitz. 

 \textbf{Classification calibration} describes the property that the excess risk when optimizing using a loss function $\ell$ would also guarantee a bound on the excessive 0-1 loss:
 \begin{definition}\label{def:cc}
 $\ell$ is classification calibrated if there $\exists$ a convex, invertible, nondecreasing transformation $\Psi_{\ell}$ with $\Psi_{\ell}(0) = 0$ s.t. 
$
\Psi_{\ell}(R_{\mathcal D}(\tilde{f})-R^*) \leq R_{\ell,\mathcal D}(\tilde{f}) - \min_f R_{\ell,\mathcal D}(f), \forall \tilde{f}.
$
 \end{definition}
Denote $f^*_{\ell} \in \argmin_f R_{\ell,\mathcal D}(f)$. Below we provide sufficient conditions for $\lAlphaPeer$ to be calibrated. 
\begin{theorem}\label{THM:calibration}
$\ell_{\text{$\alpha$-peer}}$ is classification calibrated when either of the following two conditions holds: (1) $\alpha=1$ (i.e., $\lAlphaPeer = \lPeer$), $p = 0.5$, and $f^*_{\ell}$ satisfies the following:
$
\E[\ell(f^*_{\ell}(X),-Y)] \geq \E[\ell(f(X),-Y)],~\forall f.
$ (2) $\alpha < 1, \max\{e_{+1},e_{-1}\} < 0.5$, $\ell''(t,y) = \ell''(t,-y), \forall t$, and $\alpha (1-2p) (1-e_{+1}-e_{-1})  = (1-\alpha) (e_{+1}-e_{-1}) $.
\end{theorem}
(1) states that $f^*_{\ell}$ not only achieves the smallest risk over the clean distribution $(X,Y)$ but also performs the worst on the ``opposite" distribution with flipped labels $-Y$. (2) $\ell''(t,y) = \ell''(t,-y)$ is satisfied by some common loss function, such as square and logistic losses, as noted in \citep{natarajan2013learning}, 

Under the calibration condition, and denote the corresponding calibration transformation function for $\lAlphaPeer$ as $\Psi_{\lAlphaPeer}$.
Denote by 
\begin{small}
$$\hat{f}^*_{\lAlphaPeer} = \argmin_{f \in \F}~\hat{R}_{\lAlphaPeer,\tilde{D}}(f):=\frac{1}{N}\sum_{n=1}^N \lAlphaPeer(f(x_n),\tilde{y}_n).$$
\end{small}
Consider a bounded $\ell$ with $\bar{\ell}, \underline{\ell}$ denoting its max and min value. We have the following generalization bound:
\begin{theorem}\label{THM:GEN}
With probability at least $1-\delta$: 
\begin{small}
\begin{align*}
&R_{\mathcal D}(\hat{f}^*_{\lAlphaStar}) - R^*\leq \frac{1}{1-e_{-1}-e_{+1}}\cdot \\
&~~~~~~ \Psi^{-1}_{\lAlphaStar}\biggl(\min_{f \in \mathcal F}R_{\lAlphaStar,\tilde{\mathcal D}}(f)-\min_f R_{\lAlphaStar,\tilde{\mathcal D}}(f)\\
&+4(1+\alpha^*) L\cdot \Re(\F)+2\sqrt{\frac{\log 4/\delta}{2N}}\left(1+(1+\alpha^*)(\bar{\ell}-\underline{\ell})\right)\biggr)
\end{align*}
\end{small}
 where $\Re(\F)$ is Rademacher complexity of $\F$.
\end{theorem}

\subsection{Convexity} In experiments, we use neural networks which are more robust to non-convex loss functions. 
Nonetheless, despite the fact that $\lAlphaPeer(\cdot)$ is not convex in general, Lemma 5 in \citep{natarajan2013learning} informs us that as long as $\hat{R}_{\lAlphaPeer,\tilde{D}}(f)$ is close to some convex function, mirror gradient type of algorithms will converge to a small neighborhood of the optimal point when performing ERM with $\lAlphaPeer$. A natural candidate for this convex function is the expectation of $\hat{R}_{\lAlphaPeer,\tilde{D}}(f)$ as
$
\hat{R}_{\lAlphaPeer,\tilde{D}}(f) \rightarrow R_{\lAlphaPeer, \tilde{\mathcal D}}(f) 
$ when $N \rightarrow \infty$.

\begin{lemma}\label{LEM:CONVEX}
When $\alpha < 1, \max\{e_{+1},e_{-1}\} < 0.5$, $\ell''(t,y) = \ell''(t,-y), \forall t$, and $\alpha (1-2p) (1-e_{+1}-e_{-1}) = (1-\alpha)(e_{+1}-e_{-1}) $, $R_{\lAlphaPeer, \tilde{\mathcal D}}(f) $ is convex.
\end{lemma}
This is the same condition as specified in (2) of Theorem 6.

\section{Experiments}\label{sec:exp}
\setlength{\tabcolsep}{5pt}

\begin{table*}[ht]
\small
\centering
\begin{tabular}{|c|c|c|c|c|c|c|c|c|c|c|c|}
\hline
\multicolumn{2}{|c|}{Task}  & \multicolumn{5}{c|}{With Prior Equalization $p = 0.5$}          & \multicolumn{5}{c|}{Without Prior Equalization $p \neq 0.5$}       \\ \hline
$(d,N_+, N_-)$ & $e_{-1}, e_{+1}$  & Peer  & Surr  & \rev{Symm}  & \rev{DMI}   & NN    & Peer  & Surr  & \rev{Symm}  & \rev{DMI}   & NN    \\ \hline\hline
               & 0.1, 0.3  & \textbf{0.977} & \textbf{0.968} & \textbf{0.969} & \textbf{0.974} & \textbf{0.964} & \textbf{0.977} & \textbf{0.968} & \textbf{0.969} & \textbf{0.974} & \textbf{0.964} \\ 
Twonorm        & 0.2, 0.4  & \textbf{0.976} & 0.919 & \textbf{0.959} & \textbf{0.966} & 0.911 & \textbf{0.976} & 0.919 & \textbf{0.959} & \textbf{0.966} & 0.911 \\ 
(20,3700,3700) & 0.4, 0.4  & \textbf{0.973} & 0.934 & \textbf{0.958} & 0.936 & 0.883 & \textbf{0.973} & 0.934 & \textbf{0.958} & 0.936 & 0.883 \\ \hline
               & 0.1, 0.3  & \textbf{0.919} & 0.878 & 0.851 & 0.875 & 0.811 & \textbf{0.925} & 0.885 & 0.868 & 0.889 & 0.809 \\ 
Splice         & 0.2, 0.4  & \textbf{0.901} & 0.832 & 0.757 & 0.801 & 0.714 & \textbf{0.912} & 0.84  & 0.782 & 0.81  & 0.725 \\ 
(60,1527,1648) & 0.4, 0.4  & \textbf{0.819} & 0.754 & 0.657 & 0.66  & 0.626 & \textbf{0.822} & 0.755 & 0.674 & 0.647 & 0.601 \\ \hline
               & 0.1, 0.3  & \textbf{0.833} & 0.78  & 0.777 & 0.797 & 0.756 & \textbf{0.856} & 0.802 & 0.803 & 0.83  & 0.75  \\ 
Diabetes       & 0.2, 0.4  & \textbf{0.755} & 0.681 & 0.634 & 0.682 & 0.596 & \textbf{0.739} & 0.705 & 0.695 & 0.707 & 0.672 \\ 
(8,268,500)    & 0.4, 0.4  & \textbf{0.719} & 0.645 & 0.619 & 0.637 & 0.551 & 0.651 & \textbf{0.685} & \textbf{0.68}  & 0.633 & 0.583 \\ \hline
               & 0.1, 0.3  & \textbf{0.639} & 0.563 & 0.507 & 0.529 & 0.519 & \textbf{0.727} & 0.645 & \textbf{0.709} & 0.666 & 0.648 \\ 
German         & 0.2, 0.4  & \textbf{0.664} & 0.59  & 0.6   & 0.618 & 0.572 & \textbf{0.676} & \textbf{0.681} & 0.537 & 0.573 & 0.535 \\ 
(23,300,700)   & 0.4, 0.4  & \textbf{0.606} & 0.55  & 0.573 & 0.573 & 0.556 & \textbf{0.654} & 0.632 & 0.549 & 0.611 & 0.553 \\ \hline
               & 0.1, 0.3  & \textbf{0.89}  & \textbf{0.895} & \textbf{0.892} & 0.856 & 0.868 & \textbf{0.893} & \textbf{0.898} & \textbf{0.883} & 0.785 & 0.863 \\
Waveform       & 0.2, 0.4  & \textbf{0.881} & \textbf{0.89}  & 0.828 & 0.835 & 0.81  & \textbf{0.884} & \textbf{0.884} & 0.745 & 0.761 & 0.837 \\ 
(21,1647,3353) & 0.4, 0.4  & \textbf{0.87}  & \textbf{0.866} & \textbf{0.867} & 0.773 & 0.835 & \textbf{0.853} & \textbf{0.852} & \textbf{0.852} & 0.672 & 0.828 \\ \hline
               & 0.1, 0.3  & \textbf{0.906} & \textbf{0.9}   & \textbf{0.89}  & 0.87  & 0.909 & \textbf{0.943} & 0.909 & 0.897 & 0.811 & \textbf{0.93}  \\ 
Image          & 0.2, 0.4  & 0.836 & \textbf{0.862} & 0.719 & \textbf{0.845} & 0.832 & 0.672 & 0.755 & 0.722 & \textbf{0.86}  & 0.599 \\ 
(18,1320,990)  & 0.4, 0.4  & 0.741 & 0.72  & \textbf{0.788} & 0.763 & 0.732 & \textbf{0.806} & \textbf{0.803} & \textbf{0.823} & 0.762 & 0.8   \\ \hline
\end{tabular}
\caption{\small{Experiment results on 6 UCI Benchmarks (The full table of all details on 10 UCI Benchmarks are deferred to Appendix; $N_+, N_-$ are the numbers of positive and negative samples). \rev{Surr: surrogate loss method \citep{natarajan2013learning}; DMI: \citep{xu2019l_dmi}; Symm: symmetric loss method \citep{ghosh2015making}.}
Entries within 2\% from the best in each row are highlighted in bold. 
All results are averaged across 8 random seeds.
Neural-network-based methods (Peer, Surrogate, NN, \rev{Symmetric, DMI}) use the same hyper-parameters.
}}
\label{tab:exp1}
\end{table*}
We implemented a two-layer ReLU Multi-Layer Perceptron (MLP) for classification tasks on 10 UCI Benchmarks and applied our peer loss to update their parameters. We show the robustness of peer loss with increasing rates of label noise on 10 real-world datasets. \rev{We compare the performance of our peer loss based method with surrogate loss method \citep{natarajan2013learning} (unbiased loss correction with known error rates), symmetric loss method \citep{ghosh2015making}, DMI \citep{xu2019l_dmi}}, C-SVM \citep{Liu:2003:BTC:951949.952139} and PAM \citep{Khardon:2007:NTV:1248659.1248667}, which are state-of-the-art methods for dealing with random binary-classification noise, as well as a neural network baseline solution with binary cross entropy loss (NN). We use a cross-validation set to tune the parameters specific to the algorithms.
 For surrogate loss, we use the true $e_{-1}$ and $e_{+1}$ instead of learning them separately. Thus, surrogate loss could be considered a favored and advantaged baseline method. 
Accuracy of a classification algorithm is defined as the fraction of examples in the test set classified correctly with respect to the clean and true label. 
For given noise rates $e_{+1}$ and $e_{-1}$, labels of the training data are flipped accordingly.

\begin{figure}[!ht]
    \centering
    \includegraphics[width=0.45\textwidth]{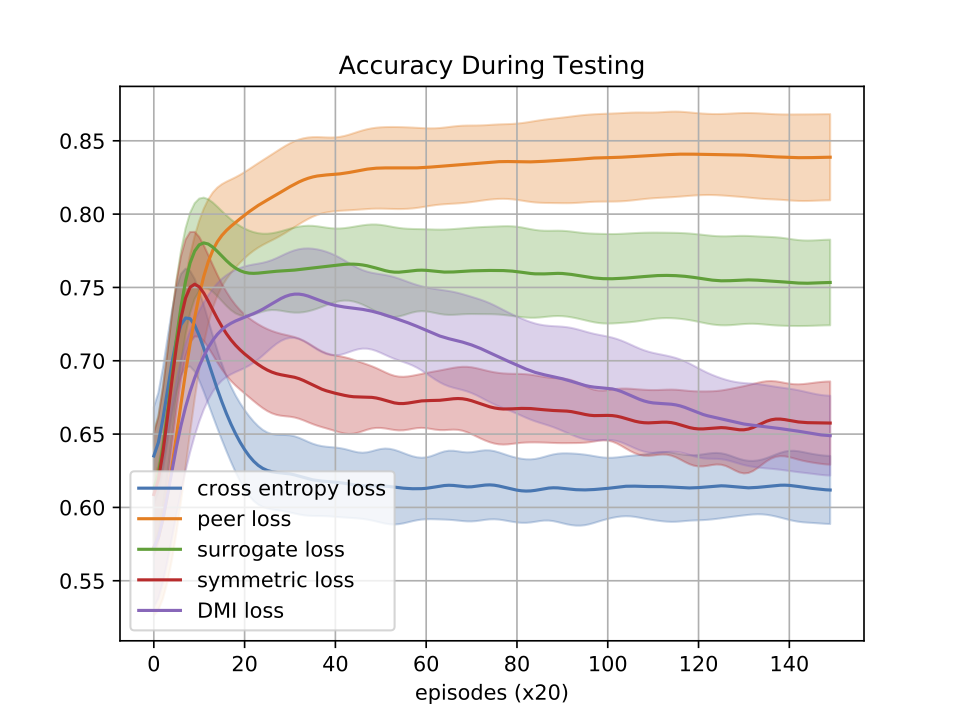}
    \caption{\rev{Accuracy on test set during training}. Splice ($e_{-1}=0.4$, $e_{+1}=0.4$). More examples can be found in Appendix.}
    \label{fig:test1}
\end{figure}

A subset of the experiment results is shown in Table \ref{tab:exp1}. A full table with all details can be found in the Appendix. \textit{Equalized Prior} means that we balance the dataset to guarantee $p=0.5$. For this case we used $\lPeer$ (i.e., $\alpha = 1$ as in $\lAlphaPeer$). For $p \neq 0.5$, we use validation dataset (still with noisy labels) to tune $\alpha$. \rev{Our method is competitive across all datasets and is even able to outperform the surrogate loss method with access to the true noise rates in a number of datasets, as well as the symmetric loss functions (which does not require the knowledge of noise rates when error rates are symmetric) and the recently proposed information theoretical loss \citep{xu2019l_dmi}.} Figure \ref{fig:test1} shows that peer loss can prevent over-fitting when facing noisy labels. 


\paragraph{A closer look at our decision boundary} To have a better understanding of peer loss, we visualize the decision boundary returned by peer loss with a 2D synthetic experiment: the outer circle of randomly places points correspond to one class and the inner one is the other class. From Figure \ref{CE:db} we observe that when using cross entropy for training, the decision boundary is sharp on clean data but becomes much less so on noisy data (we have more examples with higher noise rate in the Appendix). Peer loss returns sharp boundaries even under a high noise rate (Figure \ref{PL:db}).

\begin{figure}[!ht]
\centering
\includegraphics[width=0.5\linewidth]{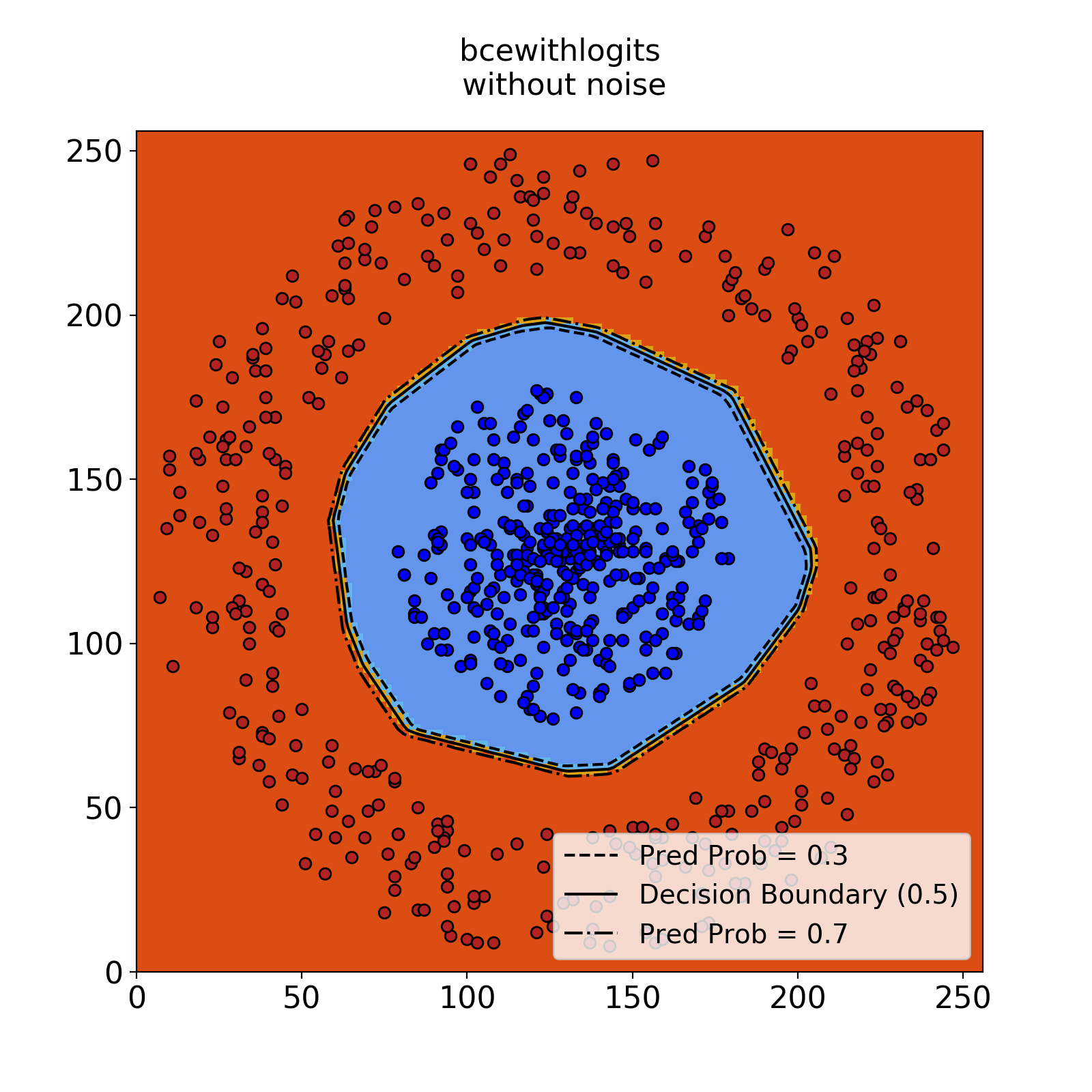}\hspace{-0.2in}
\includegraphics[width=0.5\linewidth]{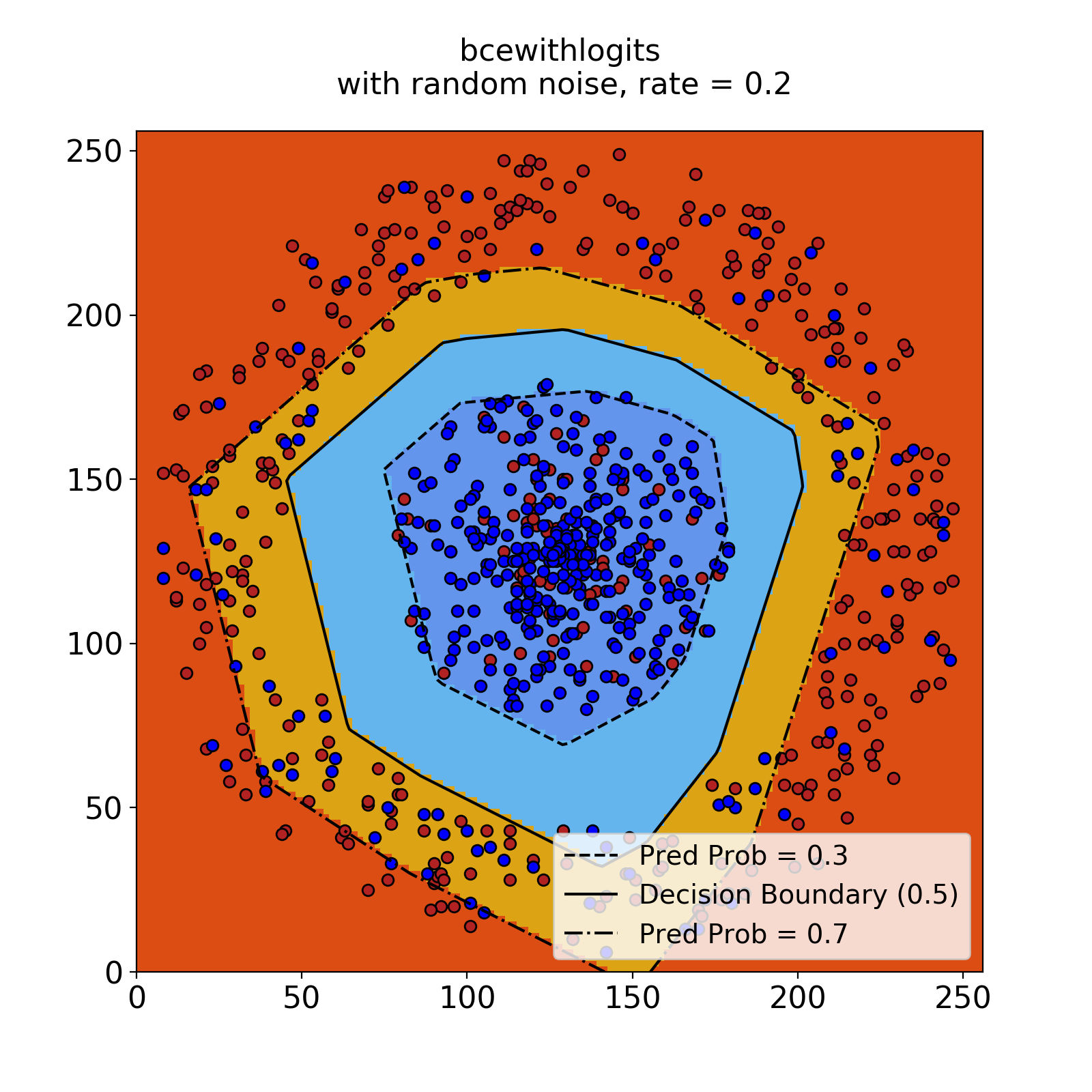}
\caption{Decision boundary for cross entropy. Left: trained on clean data. Right: trained on noisy labels, $e_{+1}=e_{-1} = 0.2$.}\label{CE:db}
\end{figure}

\begin{figure}[!ht]
\centering
\includegraphics[width=0.5\linewidth]{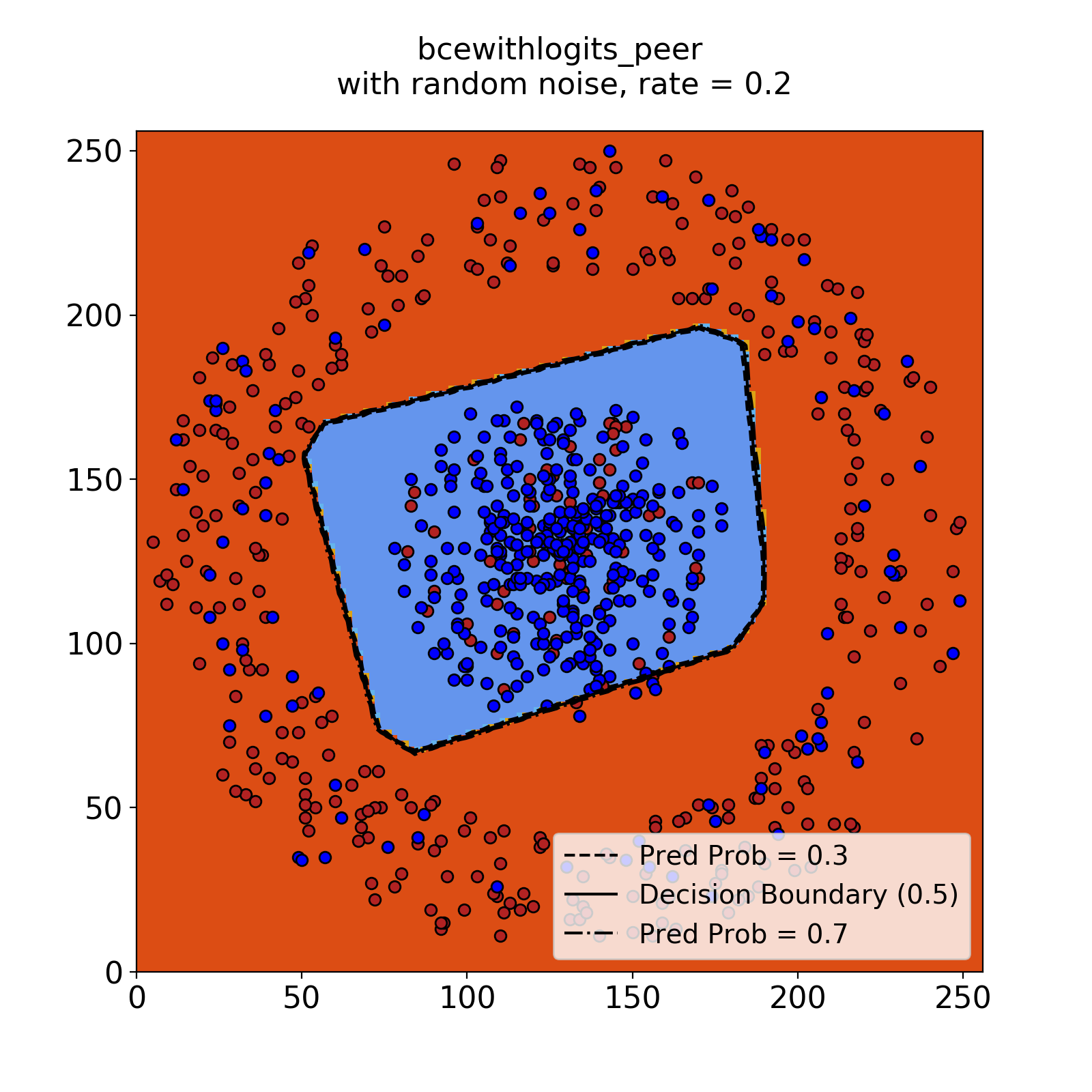}\hspace{-0.2in}
\includegraphics[width=0.5\linewidth]{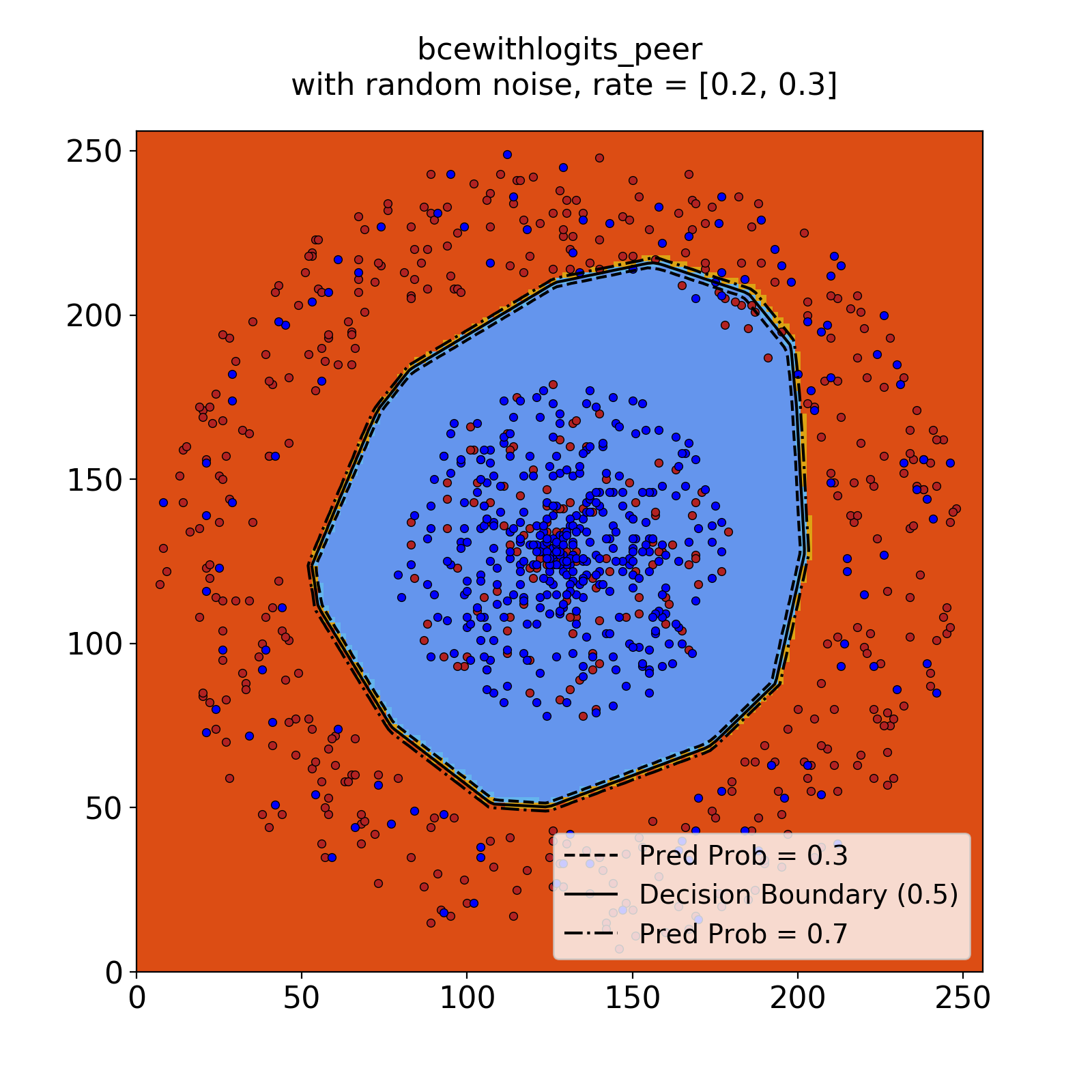}
\caption{Decision boundary for peer loss. Left: $e_{+1}=e_{-1} = 0.2$. Right: $e_{+1}= 0.2, e_{-1} = 0.3$.}\label{PL:db}
\end{figure}

\rev{\paragraph{Preliminary results on multi-class classification} We provide preliminary results on CIFAR-10 \citep{krizhevsky2009learning} in Table \ref{tab:cifar}. We followed the setup in \citep{xu2019l_dmi} and used ResNet \citep{he2016deep} as the underlying optimization solution. However, different from settings in \citep{xu2019l_dmi} where label noise only exists between specific class pairs, our noise is universal across classes. For each class, we flip the label to any other label with a probability of $\epsilon/9$, where $\epsilon$ is the error rate and $9$ is the number of other classes. We do show peer loss is competitive against cross entropy and DMI \citep{xu2019l_dmi}. 
}


\begin{table}[ht]
\small
    \centering
    \begin{tabular}{|c|c|c|}
    \hline
        Model & Error Rate $\epsilon = 0.2$ & Error Rate $\epsilon = 0.4$ \\ \hline\hline
       cross entropy  & 86.67 & 82.09 \\ \hline
        DMI \citep{xu2019l_dmi}  & 85.11 & 81.67 \\ \hline
        Peer Loss  & \textbf{87.72} & \textbf{83.81} \\ \hline
    \end{tabular}
    \caption{\rev{Accuracy on CIFAR-10.}}
    \label{tab:cifar}
\end{table}

\section{Conclusion and Discussion} 
This paper introduces peer loss, a family of loss functions that enables training a classifier over noisy labels, but without using explicit knowledge of the noise rates of labels. 


Peer loss had made the assumption that label noise is homogeneous across training data instances. 
Future extensions of this work includes extension to instance based \citep{cheng2017learning,xia2020partsdependent} and margin based \citep{amid2019robust} label noise. We are also interested in exploring the application of peer loss in differentially private ERM \citep{chaudhuri2011differentially}, as well as in semi-supervised learning.

\section*{Proof for Lemma \ref{LEM:AFFINE}}

\begin{proof}
We sketch the main steps. 
We denote by $X_{n_1},  \tilde{Y}_{n_2}$ the random variable corresponding to the peer samples $x_{n_1}, \tilde{y}_{n_2}$. 

First we have
\begin{align}
  \E[\ell_\text{peer}(f(X), \tilde{Y})] = \E[\ell(f(X), \tilde{Y})] - \mathbb E[\ell(f(X_{n_1}), \tilde{Y}_{n_2})] \label{eqn:pl0}  
\end{align}
Consider the two terms on the RHS separately.
\begin{align}
 & \mathbb E[\ell(f(X), \tilde{Y})]\nonumber \\
=& \mathbb E_{X,Y=-1} \bigl [\mathbb P(\tilde{Y}=-1 | Y=-1) \cdot \ell(f(X), -1) \nonumber\\
&~~~~+
                       \mathbb P(\tilde{Y}=+1 | Y=-1) \cdot \ell(f(X), +1)\bigr]\nonumber\\
 &+\mathbb E_{X,Y=+1} \bigl [\mathbb P(\tilde{Y}=+1 | Y=+1) \cdot \ell(f(X), +1) \nonumber\\
 &~~~~+
                       \mathbb P(\tilde{Y}=-1 | Y=+1) \cdot \ell(f(X), -1)\bigr]\tag{Independence between $\tilde{Y}$ and $X$ given $Y$}\nonumber\\
=& \mathbb E_{X,Y=-1} \bigl [(1-e_{-1})  \ell(f(X), -1) + e_{-1}  \ell(f(X), +1)\bigr]\nonumber\\
 +&\mathbb E_{X,Y=+1} \bigl [(1-e_{+1}) \ell(f(X), +1) + e_{+1}  \ell(f(X), -1)\bigr ] \label{eqn:pl1}
\end{align}
 The above is done mostly via law of total probability and using the assumption that $\tilde{Y}$ is conditionally (on $Y$) independent of $X$. Subtracting and adding $ e_{+1}\cdot\ell(f(X), -1) $ and $e_{-1}\cdot\ell(f(X), +1)$ to the two expectation terms separately we have
 \begin{align*}
\text{Eqn. (\ref{eqn:pl1})}=& \mathbb E_{X,Y=-1} \bigl [(1-e_{-1}-e_{+1}) \cdot \ell(f(X), -1) \\
                       &+ e_{+1}\cdot\ell(f(X), -1) 
                       + e_{-1} \cdot \ell(f(X), +1)\bigr] \\
 &+\mathbb E_{X,Y=+1} \bigl [(1-e_{-1}-e_{+1}) \cdot \ell(f(X), +1) \\
                       &+ e_{-1}\cdot\ell(f(X), +1)
                       + e_{+1} \cdot \ell(f(X), -1)\bigr]\\
=& (1-e_{-1}-e_{+1}) \cdot \mathbb E_{X,Y} \bigl[\ell(f(X), Y) \bigr] \\
&+ 
   \mathbb E_X \bigl [e_{+1} \cdot \ell (f(X), -1) + e_{-1} \cdot \ell(f(X), +1)\bigr]
\end{align*}
And consider the second term:
\begin{align*}
 &\mathbb E[\ell(f(X_{n_1}), \tilde{Y}_{n_2})]\\
=& \mathbb E_X [\ell(f(X), -1)] \cdot \mathbb P(\tilde{Y}=-1) \\
&+ 
   \mathbb E_X [\ell(f(X), +1)] \cdot \mathbb P(\tilde{Y}=+1) \tag{Independence between $n_1$ and $n_2$} \\
=& \mathbb E_X \bigl[ (e_{+1} \cdot p + (1-e_{-1})(1-p)) \cdot \ell(f(X), -1) \\
&+
                  \left((1-e_{+1}) p + e_{-1}(1-p)\right) \cdot \ell(f(X), +1)\bigr]\tag{Expressing $\PP(\tilde{Y})$ using $p$ and $e_{+1}, e_{-1}$}\\
=& \mathbb E_X \bigl[ (1-e_{-1}-e_{+1})(1-p) \cdot \ell(f(X), -1) \\
&+                (1-e_{-1}-e_{+1})p \cdot \ell(f(X), +1) \bigr]\\
 &+\mathbb E_X \bigl[ (e_{+1} \cdot p + e_{+1} (1-p)) \cdot \ell(f(X), -1) \\
 &+ 
                 (e_{-1} (1-p) + e_{-1} p) \cdot \ell(f(X), +1)\bigr]\\
=& (1-e_{-1}-e_{+1}) \cdot \mathbb E[\ell(f(X_{n_1}), Y_{n_2})] \\
&+ 
   \mathbb E_X\bigl[e_{+1} \cdot \ell(f(X), -1) + e_{-1} \cdot \ell (f(X), +1)\bigr]
\end{align*}
Subtracting the first and second term on RHS of Eqn. (\ref{eqn:pl0}):
\begin{align}
\mathbb E[\ell_\text{peer}&(f(X), \tilde{Y})] 
=\E[\ell(f(X), \tilde{Y})] - \mathbb E[\ell(f(X_{n_1}), \tilde{Y}_{n_2})]\nonumber \\
=& (1-e_{-1}-e_{+1}) \cdot \mathbb E[\ell_\text{peer}(f(X), Y)]
\end{align}
\end{proof}
\section*{Acknowledgement}
Yang Liu would like to thank Yiling Chen for inspiring early discussions on this problem. The authors thank Tongliang Liu, Ehsan Amid and Manfred Warmuth for constructive comments and conversations, and Nontawat Charoenphakdee for his comments on related works. The authors also would like to thank Xingyu Li, Zhaowei Zhu and Jiaheng Wei for detailed discussions, suggestions and help with generating Figures 3 and 4.

This work is partially funded by
the Defense Advanced Research Projects Agency (DARPA) and Space and Naval Warfare Systems Center Pacific (SSC Pacific) under Contract No. N66001-19-C-4014. The views and conclusions contained herein
are those of the authors and should not be interpreted as necessarily representing the official policies, either
expressed or implied, of DARPA, SSC Pacific or the U.S. Government. The U.S. Government is authorized to reproduce and distribute reprints for governmental purposes notwithstanding any
copyright annotation therein.

\bibliographystyle{icml2020}
\bibliography{noise_learning,library,myref,cvpr}
\newpage
\onecolumn
\newpage
\setcounter{section}{0}
\setcounter{figure}{0}
\setcounter{table}{0}
\setcounter{lemma}{0}
\setcounter{theorem}{0}

\makeatletter 
\renewcommand{\thefigure}{A\arabic{figure}}
\setcounter{table}{0}
\renewcommand{\thetable}{A\arabic{table}}
\setcounter{algorithm}{0}
\renewcommand{\thealgorithm}{A\arabic{algorithm}}
\setcounter{definition}{0}
\renewcommand{\thedefinition}{A\arabic{definition}}

\newpage
\appendix
\setcounter{page}{1}
\resetlinenumber

\begin{center}
\vskip 0.15in
\LARGE{\bf Appendix}  
\end{center}

\section*{Related work with more details}

\paragraph{Learning from Noisy Labels}

Our work fits within a stream of research on learning with noisy labels.
A large portion of research on this topic works with the \emph{random classification noise} (RCN) model, where observed labels are flipped independently with probability $e \in [0,\tfrac{1}{2}]$ 
\citep{bylander1994learning,cesa1999sample,cesa2011online,ben2009agnostic}. Recently, learning with asymmetric noisy data (or also referred as \emph{class-conditional} random classification noise (CCN)) for binary classification problems has been rigorously studied in \citep{stempfel2009learning,scott2013classification,natarajan2013learning,scott2015rate,van2015learning,menon2015learning}.

\rev{
\paragraph{Symmetric loss} For RCN, where the noise parameters are symmetric, there exist works that show symmetric loss functions \citep{manwani2013noise,ghosh2015making,ghosh2017robust,van2015learning} are robust to the underlying noise, without specifying the noise rates. It was also shown that under certain conditions, the proposed loss functions are able to handle asymmetric noise. Our focus departs from this line of works and we exclusively focus on asymmetric noise setting, and study the possibility of an approach that can ignore the knowledge of noise rates.

Follow-up works \citep{du2013clustering,van2015average,menon2015learning,charoenphakdee2019symmetric} have looked into leveraging symmetric conditions and 0-1 loss with asymmetric noise, and with more evaluation metrics, such as balanced error rate and AUROC. In particular, experimental evidence is reported in \citep{charoenphakdee2019symmetric} on the importance of symmetricity when learning with noisy labels.

\paragraph{More recent works}
 More recent developments include an importance re-weighting algorithm \citep{liu2016classification}, a noisy deep neural network learning setting \citep{sukhbaatar2014learning,han2018co,song2019selfie}, and learning from massive noisy data for image classification \citep{xiao2015learning,goldberger2016training,zhang2017mixup,jiang2017mentornet,jenni2018deep,yi2019probabilistic}, robust cross entropy loss for neural network \citep{zhang2018generalized}, loss correction \citep{patrini2017making}, among many others. Loss or sample correction has also been studied in the context of learning with unlabeled data with weak supervisions \citep{lu2018minimal}. Most of the above works either lacks theoretical guarantee of the proposed method against asymmetric noise rates \citep{sukhbaatar2014learning,zhang2018generalized}, or require estimating the noise rate or transition matrix between noisy and true labels \citep{liu2016classification,xiao2015learning,patrini2017making,lu2018minimal}. A good number of the recent works can be viewed as derivatives or extension of the unbiased surrogate loss function idea introduced in \citep{natarajan2013learning}, therefore they would naturally require the knowledge of the noise rates or transition matrix. We do provide thorough comparisons between peer loss and the unbiased surrogate loss methods.
 
 A recent work \citep{xu2019l_dmi} proposes an information theoretical loss (an idea adapted from an earlier theoretical contribution \citep{kong2018water}) that is also robust to asymmetric noise rate. We aimed for a simple-to-optimize loss function that can easily adapt to existing ERM solutions. \citep{xu2019l_dmi} involves estimating a joint distribution matrix between classifiers and noisy labels, and then invokes computing a certain information theoretical measure based on this matrix. Therefore, its sample complexity requirement and the sensitivity to noise in this estimation are not entirely clear to us (not provided in the paper either). We do provide calibration guarantees, generalization bounds, and conditions under which the loss functions are convex. In general, we do think computationally peer loss functions are easy to optimize with, in comparison to information theoretical measures. Experiments comparing with \citep{xu2019l_dmi} are also given in Section \ref{sec:exp}. 
}

\vspace{-0.1in}
\paragraph{Peer Prediction}

Our work builds on the literature for peer prediction \citep{prelec2004bayesian,MRZ:2005,witkowski2012robust,radanovic2013,Witkowski_hcomp13,dasgupta2013crowdsourced,shnayder2016informed,sub:ec17}. \citep{MRZ:2005} established that strictly proper scoring rule \citep{Gneiting:07} could be adopted to elicit truthful reports from self-interested agents.  Follow-up works that have been done to relax the assumptions imposed \citep{witkowski2012robust,radanovic2013,Witkowski_hcomp13,radanovic2016incentives,sub:ec17}.
Most relevant to us is \citep{dasgupta2013crowdsourced,shnayder2016informed} where a correlated agreement (CA) type of mechanism was proposed. CA evaluates a report's correlations with another reference agent - its specific form inspired our peer loss.

\section*{Illustration of our implementation of peer loss}

\begin{figure*}[!h]
\begin{center}
\includegraphics[width=0.9\textwidth]{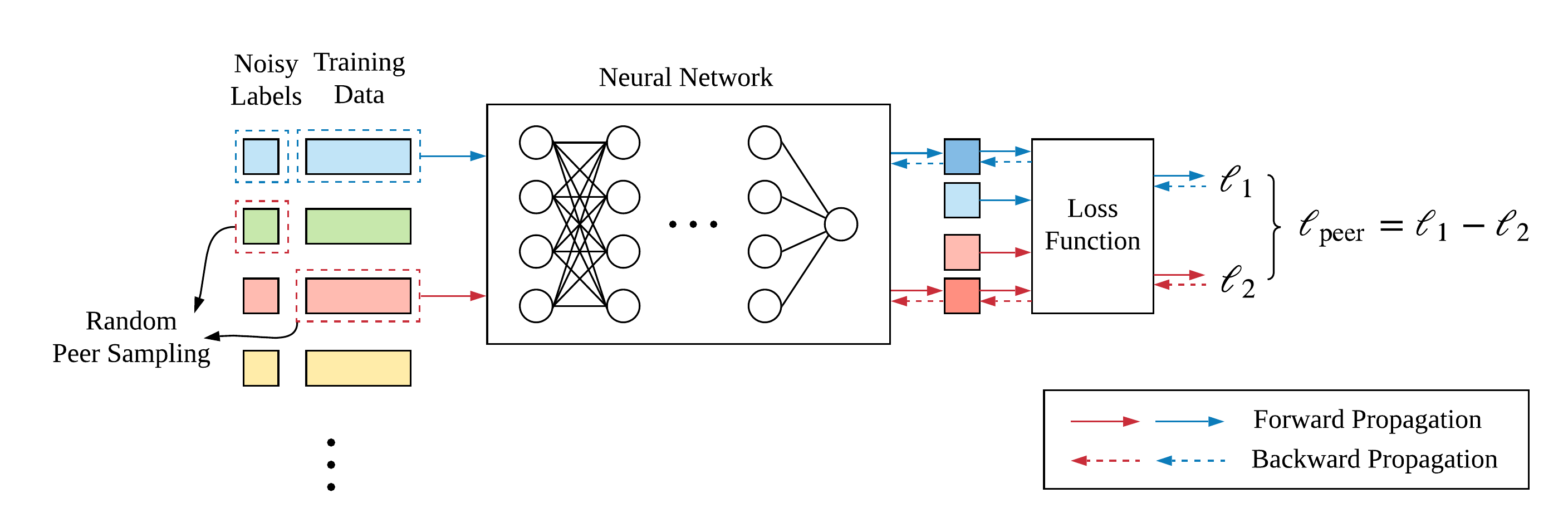}
\caption{Illustration of our peer loss implementation.}
\end{center}
\label{mechanism:ill_plf} 
\end{figure*}
We illustrate our peer loss method in Figure 5. 

\section*{Other peer prediction functions}

 Other notable examples include quadratic and logarithmic scoring function, defined as follows:
\begin{example}
Quadratic scoring function:
\[
S\bigl( r^A, r^B\bigr) := 2{\PP}\bigl (y^B = r^B| y^A = r^A\bigr)  - \sum_{s \in \{-1,+1\}}{\PP}\bigl(y^B = s| y^A = r^A\bigr) ^2,
\]
\end{example}

\begin{example}
Logarithmic scoring function: 
\[
S\bigl( r^A, r^B\bigr)  := \log {\PP}\bigl (y^B = r^B| y^A = r^A\bigr). \]
\end{example}
We know the following is true:
\begin{lemma}[\cite{MRZ:2005}]
$S$ defined in Example 1 \& 2 induce strict truthfulness when $y^A$ and $y^B$ are \emph{stochastically relevant}.
\end{lemma}
with defining stochastic relevance as follows:
\begin{definition}
$y^A$ and $y^B$ are stochastically relevant if
 ~$\exists ~s \in \{-1,+1\}$ s.t. 
$$
 \PP\bigl(y^B=s|y^A=+1\bigr) \neq \PP\bigl(y^B=s|y^A = -1\bigr).
$$
\end{definition}

Similarly we conclude that when $y^A$ and $y^B$ are stochastic relevant, the correlated agreement scoring rule, quadratic scoring rule and logarithmic scoring rule are strictly truthful.

\section*{Proof for Theorem \ref{THM:MAIN}}

\begin{proof}
Note that proving $f^* = \text{argmin}_{f }~ \mathbb E_{(X,\tilde{Y}) \sim \tilde{\mathcal D}}\bigl[-S(f(X),\tilde{Y})\bigr]$ is equivalent with proving $
f^* = \text{argmax}_{f }~ \mathbb E_{(X,\tilde{Y}) \sim \tilde{\mathcal D}}\bigl[S(f(X),\tilde{Y})\bigr].
$ 

First note that the expected score of a classifier over the data distribution further writes as follows:
\begin{align*}
\mathbb E_{\tilde{\mathcal D}}\bigl[S(f(X), \tilde{Y})\bigr]& =  p\cdot    \mathbb E_{\tilde{\mathcal D}|Y=+1}\bigl [S(f(X),\tilde{Y})\bigr] + (1-p)\cdot    \mathbb E_{\tilde{\mathcal D}|Y=-1}\bigl [S(f(X),\tilde{Y}]\bigr).
\end{align*}
$S(\cdot)$ is able to elicit the Bayes optimal classifier $f^*$ using $\tilde{Y}$ implies that 
\begin{align*}
 &p \cdot \mathbb E_{ \tilde{\mathcal D}|Y=+1}\bigl [S(f^*(X), \tilde{Y})\bigr] + (1-p) \cdot       \mathbb E_{\tilde{\mathcal D}|Y=-1}\bigl [S(f^*(X),\tilde{Y})\bigr]  \\
    &> p\cdot \mathbb E_{\tilde{\mathcal D}|Y=+1}\bigl [S\bigl(\R'(X), \tilde{Y}\bigr)\bigr]+(1-p)\cdot \mathbb E_{\tilde{\mathcal D}|Y=-1}\bigl [S(\R'(X),\tilde{Y})\bigr]
    ,~\forall \R' \neq f^*~.
\end{align*}

Denote by $f'$ a sub-optimal classifier that disagrees with $f^*$ on set $\mathcal X^{+}_{\text{dis}|Y=+1} = \{X|Y=+1: f'(X) \neq f^*(X), f^*(X)=+1\}$. Denote $\epsilon^+:={\PP}(X \in \mathcal X^+_{\text{dis}|Y=+1})$. Construct the following reporting strategy for $f^*(X)=+1$: 
  \[ \R'(X) =\left\{
                \begin{array}{ll}
                f^*(X), ~\text{w.p.}~1-\epsilon^+\\
                 -f^*(X), ~\text{w.p.}~\epsilon^+
                \end{array}
              \right.
 \]
 We can similarly construct the above reporting strategy for  $\mathcal X^{-}_{\text{dis}|Y=+1} = \{X|Y=+1: f'(X) \neq f^*(X), f^*(X)=-1\}$ with parameter $\epsilon^-:={\PP}(X \in \mathcal X^-_{\text{dis}|Y=+1})$:
   \[ \R'(X) =\left\{
                \begin{array}{ll}
                f^*(X), ~\text{w.p.}~1-\epsilon^-\\
                 -f^*(X), ~\text{w.p.}~\epsilon^-
                \end{array}
              \right.
 \]
 By definition of sub-optimality of $f'$ in reporting we know that
$\max\{\epsilon^+, \epsilon^-\}>0$, as a zero measure mis-reporting strategy does not affect its optimality. Not hard to check that 
 \begin{align*}
 &\mathbb E_{ \tilde{\mathcal D}|Y=+1}\bigl [S(f'(X),\tilde{Y})\bigr] \\
 &=\PP(f^*(X)=+1|Y=+1) \cdot \mathbb E_{ \tilde{\mathcal D}|f^*(X)=+1,Y=+1}\bigl [S(f'(X),\tilde{Y})\bigr]\\
 &+\PP(f^*(X)=-1|Y=+1) \cdot \mathbb E_{ \tilde{\mathcal D}|f^*(X)=-1,Y=+1}\bigl [S(f'(X),\tilde{Y})\bigr]~.
 \end{align*}
  Each of the two conditional expectation terms further derives:
 \begin{align*}
     &\mathbb E_{ \tilde{\mathcal D}|f^*(X)=+1,Y=+1}\bigl [S(f'(X),\tilde{Y})\bigr] \\
     &=\left(1-{\PP}(X \in \mathcal X^+_{\text{dis}|Y=+1})\right)\cdot \mathbb E_{\tilde{\mathcal D}|f^*(X)=+1,Y=+1,f'(X) = f^*(X)}[S(f'(X)=+1,\tilde{Y})]\\
     &+{\PP}(X \in \mathcal X^+_{\text{dis}|Y=+1}) \cdot \mathbb E_{\tilde{\mathcal D}|f^*(X)=+1,Y=+1,f'(X) \neq  f^*(X)}[S(f'(X)=-1,\tilde{Y})]\\
     &=(1-\epsilon^+)\cdot \mathbb E_{\tilde{\mathcal D}|f^*(X)=+1,Y=+1}[S(f'(X)=+1,\tilde{Y})]\\
     &+\epsilon^+ \cdot \mathbb E_{\tilde{\mathcal D}|f^*(X)=+1,Y=+1}[S(f'(X)=-1,\tilde{Y})]\\
     &=\mathbb E_{\tilde{\mathcal D}|f^*(X)=+1,Y=+1}[S(\R'(X),\tilde{Y})]~.
 \end{align*}
In the second equality above, the dropping of conditions $f'(X) \neq f(X)$ in $\E$ is due to the fact that $\tilde{Y}$ is conditionally independent of $X$ given $Y$. The above argument repeats for $ \mathbb E_{ \tilde{\mathcal D}|f^*(X)=-1,Y=+1}\bigl [S(f'(X),\tilde{Y})\bigr]$. Therefore we conclude
 \[
\mathbb E_{ \tilde{\mathcal D}|Y=+1}\bigl [S(f'(X),\tilde{Y})\bigr]  =   \mathbb E_{\tilde{\mathcal D}|Y=+1}\left [S\left(\R'(X),\tilde{Y}\right)\right]~.
 \]
 
 Yet we have the following fact that
 \begin{align}
\mathbb E_{\tilde{\mathcal D}|Y=+1}\left [S\left(\R'(X),\tilde{Y}\right)\right]&= \PP(f^*(X)=+1|Y=+1) \cdot \biggl((1-\epsilon^+) \cdot \mathbb E_{\tilde{\mathcal D}|Y=+1,f^*(X)=+1}\bigl [S(f^*(X),\tilde{Y})\bigr] \nonumber \\
  &~~~~~+ \epsilon^+ \cdot \mathbb E_{\tilde{\mathcal D}|Y=+1,f^*(X)=+1}\bigl [S(-f^*(X),\tilde{Y})\bigr]\biggr)\nonumber \\
  &+ \PP(f^*(X)=-1|Y=+1) \cdot \biggl((1-\epsilon^-) \cdot \mathbb E_{\tilde{\mathcal D}|Y=+1,f^*(X)=-1}\bigl [S(f^*(X),\tilde{Y})\bigr] \nonumber \\
  &~~~~+ \epsilon^- \cdot \mathbb E_{\tilde{\mathcal D}|Y=+1,f^*(X)=-1}\bigl [S(-f^*(X),\tilde{Y})\bigr]\biggr) \label{eqn:thm1:0}
  \end{align}
  Subtracting and adding $\epsilon^- \cdot  \mathbb E_{\tilde{\mathcal D}|Y=+1,f^*(X)=+1}\bigl [S(f^*(X),\tilde{Y})\bigr]$ and $\epsilon^+ \cdot  \mathbb E_{\tilde{\mathcal D}|Y=+1,f^*(X)=-1}\bigl [S(f^*(X),\tilde{Y})\bigr]$ in above partial sums in $\left(\cdot \right)$, and combining the first $(1-\epsilon^+-\epsilon^-)$ terms, we have 
  \begin{align}
 \text{Eqn. (\ref{eqn:thm1:0})} &= (1-\epsilon^+-\epsilon^-) \cdot \mathbb E_{ \tilde{\mathcal D}|Y=+1}\bigl [S(f^*(X),\tilde{Y})\bigr] \nonumber \\
  &+ \epsilon^+  \cdot\biggl(\PP(f^*(X)=+1|Y=+1) \cdot \mathbb E_{\tilde{\mathcal D}|Y=+1,f^*(X)=+1}\bigl [S(-f^*(X),\tilde{Y})\bigr] \nonumber \\
  &~~~+ \PP(f^*(X)=-1|Y=+1) \cdot  \mathbb E_{\tilde{\mathcal D}|Y=+1,f^*(X)=-1}\bigl [S(f^*(X),\tilde{Y})\bigr] \biggr)  \nonumber \\
  &+ \epsilon^-  \cdot\biggl(\PP(f^*(X)=+1|Y=+1) \cdot \mathbb E_{\tilde{\mathcal D}|Y=+1,f^*(X)=+1}\bigl [S(f^*(X),\tilde{Y})\bigr]\nonumber \\
  &~~~+ \PP(f^*(X)=-1|Y=+1) \cdot  \mathbb E_{\tilde{\mathcal D}|Y=+1,f^*(X)=-1}\bigl [S(-f^*(X),\tilde{Y})\bigr] \biggr) \nonumber \\
  &= (1-\epsilon^+-\epsilon^-) \cdot \mathbb E_{ \tilde{\mathcal D}|Y=+1}\bigl [S(f^*(X),\tilde{Y})\bigr] \nonumber \\
  &~~~~+ \epsilon^+ \cdot  \mathbb E_{ \tilde{\mathcal D}|Y=+1}\bigl [S(\R^\circ  \equiv -1,\tilde{Y})\bigr]\nonumber \\
  &~~~~+\epsilon^-  \cdot \mathbb E_{ \tilde{\mathcal D}|Y=+1}\bigl [S(\R^\circ  \equiv +1,\tilde{Y})\bigr]~.
 \end{align}
 Note that $\mathbb E_{ \tilde{\mathcal D}|Y=+1}\bigl [S(\R^\circ  \equiv -1,\tilde{Y})\bigr]$ corresponds to a reporting strategy that always reports $-1$, and $\mathbb E_{ \tilde{\mathcal D}|Y=+1}\bigl [S(\R^\circ  \equiv +1,\tilde{Y})\bigr]$ is the one to always report $+1$. 


 Similarly for $ \mathbb E_{ \tilde{\mathcal D}|Y=-1}\bigl [S(f'(X),\tilde{Y})\bigr]$: Suppose $f'$ disagrees with $f^*$ on set $\mathcal X^{+}_{\text{dis}|Y=-1} = \{X|Y=-1: f'(X) \neq f^*(X), f^*(X)=+1\}$. 
Note  
 because of the assumption 
$\PP(f(X)\neq f^*(X)|f^*(X)=+1,Y=+1)=\PP(f(X)\neq f^*(X)|f^*(X)=+1,Y=-1)$, in above we have the same $\epsilon^+$ as in the $Y=+1$. Construct the following reporting strategy for $f^*(X)=+1$
  \[ \R'(X) =\left\{
                \begin{array}{ll}
                f^*(X), ~\text{w.p.}~1-\epsilon^+\\
                 -f^*(X), ~\text{w.p.}~\epsilon^+
                \end{array}
              \right.
 \]
 We can similarly construct the above reporting strategy for  $\mathcal X^{-}_{\text{dis}|Y=-1} = \{X|Y=-1: f'(X) \neq f^*(X), f^*(X)=-1\}$ with parameter $\epsilon^-:={\PP}(X \in \mathcal X^-_{\text{dis}|Y=-1})$:
   \[ \R'(X) =\left\{
                \begin{array}{ll}
                f^*(X), ~\text{w.p.}~1-\epsilon^-\\
                 -f^*(X), ~\text{w.p.}~\epsilon^-
                \end{array}
              \right.
 \]
 Similarly we claim that (details not repeated)
  \begin{align*}
 &\mathbb E_{ \tilde{\mathcal D}|Y=-1}\bigl [S(f'(X),\tilde{Y})\bigr] =   \mathbb E_{\tilde{\mathcal D}|Y=-1}\left [S\left(\R'(X),\tilde{Y}\right)\right]~. 
 \end{align*}
Again we note that
  \begin{align}
 \mathbb E_{ \tilde{\mathcal D}|Y=-1}\left [S\left(\R'(X),\tilde{Y}\right)\right]&= \PP(f^*(X)=+1|Y=-1) \cdot \biggl((1-\epsilon^+) \cdot \mathbb E_{\tilde{\mathcal D}|Y=-1,f^*(X)=+1}\bigl [S(f^*(X),\tilde{Y})\bigr] \nonumber \\
  &~~~~~+ \epsilon^+ \cdot \mathbb E_{\tilde{\mathcal D}|Y=-1,f^*(X)=+1}\bigl [S(-f^*(X),\tilde{Y})\bigr]\biggr)\nonumber \\
  &+ \PP(f^*(X)=-1|Y=-1) \cdot \biggl((1-\epsilon^-) \cdot \mathbb E_{\tilde{\mathcal D}|Y=-1,f^*(X)=-1}\bigl [S(f^*(X),\tilde{Y})\bigr] \nonumber \\
  &~~~~+ \epsilon^- \cdot \mathbb E_{\tilde{\mathcal D}|Y=-1,f^*(X)=-1}\bigl [S(-f^*(X),\tilde{Y})\bigr]\biggr)~. \label{eqn:thm1:0}
  \end{align}
 And by further rearranging terms as done above we have (with only difference being replacing $Y=+1$ with $Y=-1$; we omit the details)
 \begin{align*}
     &\mathbb E_{ \tilde{\mathcal D}|Y=-1}\left [S\left (\R'(X),\tilde{Y}\right)\right]\\
       &= (1-\epsilon^+-\epsilon^-) \cdot \mathbb E_{ \tilde{\mathcal D}|Y=-1}\bigl [S(f^*(X),\tilde{Y})\bigr] \nonumber \\
  &~~~~+\epsilon^+  \cdot \mathbb E_{ \tilde{\mathcal D}|Y=-1}\bigl [S(\R^\circ  \equiv -1,\tilde{Y})\bigr]\\
    &~~~~+ \epsilon^- \cdot  \mathbb E_{ \tilde{\mathcal D}|Y=-1}\bigl [S(\R^\circ  \equiv +1,\tilde{Y})\bigr].\nonumber 
 \end{align*}
 Therefore 
 \begin{align*}
& \mathbb E_{ \tilde{\mathcal D}}\bigl [S(f'(X),\tilde{Y})\bigr]\\
&=p\cdot \mathbb E_{ \tilde{\mathcal D}|Y=+1}\bigl [S(f'(X),\tilde{Y})\bigr]+(1-p) \cdot \mathbb E_{ \tilde{\mathcal D}|Y=-1}\bigl [S(f'(X),\tilde{Y})\bigr]\\
     &= p \cdot \mathbb E_{ \tilde{\mathcal D}|Y=+1}\bigl [S(\R'(X),\tilde{Y})\bigr] + (1-p) \cdot \mathbb E_{ \tilde{\mathcal D}|Y=-1}\bigl [S(\R'(X),\tilde{Y})\bigr]\\
     &= p \cdot (1-\epsilon^+-\epsilon^-) \cdot \mathbb E_{ \tilde{\mathcal D}|Y=+1}\bigl [S(f^*(X),\tilde{Y})\bigr] + (1-p) \cdot (1-\epsilon^+-\epsilon^-) \cdot \mathbb E_{ \tilde{\mathcal D}|Y=-1}\bigl [S(f^*(X),\tilde{Y})\bigr] \\
     &+\epsilon^+ \cdot  \left(p\cdot \mathbb E_{ \tilde{\mathcal D}|Y=+1}\bigl [S(\R^\circ  \equiv -1,\tilde{Y})\bigr] + (1-p) \cdot \mathbb E_{ \tilde{\mathcal D}|Y=-1}\bigl [S(\R^\circ  \equiv -1,\tilde{Y})\bigr]\right)\\
     &+\epsilon^- \cdot  \left(p\cdot \mathbb E_{ \tilde{\mathcal D}|Y=+1}\bigl [S(\R^\circ  \equiv +1,\tilde{Y})\bigr] + (1-p) \cdot \mathbb E_{ \tilde{\mathcal D}|Y=-1}\bigl [S(\R^\circ  \equiv +1,\tilde{Y})\bigr]\right)\\
     &=(1-\epsilon^+-\epsilon^-) \cdot \mathbb E_{ \tilde{\mathcal D}}\bigl [S(f^*(X),\tilde{Y})\bigr] + \epsilon^+ \cdot  \mathbb E_{ \tilde{\mathcal D}}\bigl [S(\R^\circ  \equiv -1,\tilde{Y})\bigr]+\epsilon^- \cdot   \mathbb E_{ \tilde{\mathcal D}}\bigl [S(\R^\circ  \equiv +1,\tilde{Y})\bigr]~. 
 \end{align*}
  Due to the truthfulness of $S$ (unless $f^*$ is predicting all $+1$ or all $-1$ labels), 
  \[
  \mathbb E_{ \tilde{\mathcal D}}\bigl [S(\R^\circ  \equiv -1,\tilde{Y})\bigr] <  \mathbb E_{ \tilde{\mathcal D}}\bigl [S(f^*(X),\tilde{Y})\bigr],~~  \mathbb E_{ \tilde{\mathcal D}}\bigl [S(\R^\circ  \equiv +1,\tilde{Y})\bigr] <  \mathbb E_{ \tilde{\mathcal D}}\bigl [S(f^*(X),\tilde{Y})\bigr],~~
  \]
Using the fact that $\max\{\epsilon^+, \epsilon^-\}>0$:
\begin{align}
    (1-\epsilon^+-\epsilon^-) \cdot \mathbb E_{ \tilde{\mathcal D}}\bigl [S(f^*(X),\tilde{Y})\bigr] &+ \epsilon^+ \cdot  \mathbb E_{ \tilde{\mathcal D}}\bigl [S(\R^\circ  \equiv -1,\tilde{Y})\bigr]+\epsilon^- \cdot   \mathbb E_{ \tilde{\mathcal D}}\bigl [S(\R^\circ  \equiv +1,\tilde{Y})\bigr] 
    < \mathbb E_{ \tilde{\mathcal D}}\bigl [S(f^*(X),\tilde{Y})\bigr].
    \label{eqn:pp1}
\end{align}
Therefore we proved the optimality of $f^*$.

\end{proof}

 According to Theorem \ref{THM:MAIN} and Theorem 4.4,  \cite{shnayder2016informed}, minimizing $-S$ defined in CA is going to find the Bayes optimal classifier, if $\tilde{Y}$ and $f^*$ are categorical, which is easily satisfied:
\begin{lemma}\label{LEMMA:CATE}
When $e_{-1}+e_{+1} < 1$, $\tilde{Y}$ and $f^*$ are categorical. 
\end{lemma}

\begin{proof}
Being categorical means
\[
\PP(\tilde{Y}=-y|f^*(X) = y) < \PP(\tilde{Y}=-y),~y \in \{-1,+1\}
\]
which further implies
\[
\PP(\tilde{Y}=-y, f^*(X) = y) < \PP(\tilde{Y}=-y)\PP(f^*(X)=y),~y \in \{-1,+1\}
\]
and
\[
\PP(\tilde{Y}=y, f^*(X) = y) > \PP(\tilde{Y}=y)\PP(f^*(X)=y),~y \in \{-1,+1\}.
\]
Consider the following fact
\begin{align*}
&\PP(\tilde{Y}=+1,f^*(X)=+1) \\
=&\PP(Y=+1)\PP(\tilde{Y}=+1,f^*(X)=+1|Y=+1) \\
&~~~~+ \PP(Y=-1)\PP(\tilde{Y}=+1,f^*(X)=+1|Y=-1) \\
=&\PP(Y=+1)\PP(\tilde{Y}=+1|f^*(X)=+1,Y=+1)\\
&~~~~\cdot\PP(f^*(X)=+1|Y=+1) \\
+&\PP(Y=-1)\PP(\tilde{Y}=+1|f^*(X)=+1,Y=-1)\\
&\cdot\PP(f^*(X)=+1|Y=-1)
\end{align*}
Since $f^*(X)$ is a function of $X$ and $Y$, due to conditional independence between $\tilde{Y}$ and $X$ (conditional on $Y$) we have
\begin{align*}
&\PP(\tilde{Y}=+1|f^*(X)=+1,Y=+1)=\PP(\tilde{Y}=+1|Y=+1) = 1-e_{+1},\\
&\PP(\tilde{Y}=+1|f^*(X)=+1,Y=-1) = \PP(\tilde{Y}=+1|Y=-1) = e_{-1}
\end{align*}
Therefore
\begin{align*}
&\PP(\tilde{Y}=+1,f^*(X)=+1) = \PP(Y=+1) (1-e_{+1})(1-e^*_{+1}) + \PP(Y=-1)\cdot e_{-1} \cdot e^*_{-1}
\end{align*}
We also have
\begin{align*}
&\PP(\tilde{Y}=+1) = \PP(Y=+1) (1-e_{+1})+\PP(Y=-1) \cdot e_{-1}\\ &\PP(f^*(X)=+1) 
= \PP(Y=+1) (1-e^*_{+1})+\PP(Y=-1) \cdot e^*_{-1} 
\end{align*}
Then we have
\begin{align*}
&\PP(\tilde{Y}=+1,f^*(X)=+1)  - \PP(\tilde{Y}=+1) \PP(f^*(X)=+1) \\
=&\PP(Y=+1)\PP(Y=-1)(1-e_{+1}-e_{-1})(1-e^*_{+1}-e^*_{-1}) \\
>& 0, ~\text{when $1> e^*_{+1}+e^*_{-1}$.}
\end{align*}
$ e^*_{-1}+e^*_{+1} < 1$ means that the Bayes' optimal classifier is at least informative (\citep{sub:ec17}) - if otherwise, we can flip the classifier's output to obtain one, which contradicts the optimality of Bayes optimal classifier.

\end{proof}

\section*{Details for Example \ref{exa:delta}}

First of all, we compute the marginals of $f^*$ and $\tilde Y$:
\begin{align*}
     \PP\left(f^{*}(X)=-1\right)&= \PP\left(f^{*}(X)=-1 | Y=-1\right)  \PP(Y=-1)+ \PP\left(f^{*}(X)=-1 | Y=+1\right)  \PP(Y=+1)\\
    &=\left(1-e_{-1}^{*}\right) \cdot 0.4+e_{+1}^{*} \cdot 0.6=0.5
\end{align*}
And easily
\begin{align*}
    & \PP\left(f^{*}(X)=+1\right)=1- \PP\left(f^{*}(X)=-1\right)=0.5
\end{align*}
For noisy labels:
\begin{align*}
    \PP(\tilde{Y}=-1)&= \PP(\tilde{Y}=-1 | Y=-1)  \PP(Y=-1)+ \PP(\tilde{Y}=-1 | Y=+1)  \PP(Y=+1)\\
    &=\left(1-e_{-1}\right) \cdot 0.4+e_{+1} \cdot 0.6=0.52,
\end{align*}
and
\begin{align*}
     \PP\left(f^{*}(X)=+1\right)=1- \PP\left(f^{*}(X)=-1\right)=0.48
\end{align*}
For the joint distribution,
\begin{align*}
    & \PP\left(f^{*}(X)=-1, Y=-1\right)\\
    =& \PP\left(f^{*}(X)=-1, Y=-1 | Y=-1\right)  \PP(Y=-1)+ \PP\left(f^{*}(X)=-1, Y=-1 | Y=+1\right)\\
    =&\left(1-e_{-1}^{*}\right)\left(1-e_{-1}\right) \cdot 0.4+e_{+1}^{*} \cdot e_{+1} \cdot 0.6=0.296
\end{align*}
\begin{align*}
     \PP\left(f^{*}(X) = -1, \tilde{Y}=+1\right)= \PP\left(f^{*}(X)=-1\right)- \PP\left(f^{*}(X)=-1, Y=-1\right)=0.264\end{align*}
Further, 
\begin{align*}
    &\PP\left(f^{*}(X)=+1, \tilde{Y}=-1\right)=\PP(\tilde{Y}=-1)-\PP\left(f^{*}(X)=-1, Y=-1\right)=0.224\\
    &\PP\left(f^{*}(X)=+1, Y=+1\right)=\PP\left(f^{*}(X)=+1\right)-\PP\left(f^{*}(X)=+1, Y=-1\right)=0.216
\end{align*}
With above, the entries in Delta can be computed easily, for instance
\begin{align*}
    \Delta_{1,1}=\PP\left(f^{*}(X)=-1, Y=-1\right)-\PP\left(f^{*}(X)=-1\right) \cdot \PP(Y=-1)=0.296-0.5 \cdot 0.52=0.036
\end{align*}

\section*{Proof for Lemma \ref{LEMMA:SGN} }

\begin{proof}
Again recall that 
\begin{align*}
&\PP(\R^*=+1, \R=+1) = \PP(Y=+1) (1-e_{+1})(1-e^*_{+1}) + \PP(Y=-1) e_{-1} \cdot e^*_{-1}
\end{align*}
\begin{align*}
&\PP(\R=+1) = \PP(Y=+1) (1-e_{+1})+\PP(Y=-1)\cdot e_{-1}\\ 
&\PP(\R^*=+1) 
= \PP(Y=+1) (1-e^*_{+1})+\PP(Y=-1)\cdot e^*_{-1} \end{align*}
Then we have
\begin{align*}
&\PP(\R^*=+1,\R=+1)  -  \PP(\R^*=+1)\PP(\R=+1)  \\
=&\PP(Y=+1)\PP(Y=-1)(1-e_{+1}-e_{-1})(1-e^*_{+1}-e^*_{-1}) > 0
\end{align*}
when $1-e_{+1}-e_{-1}>0, ~1-e^*_{+1}-e^*_{-1} >0$. Interestingly this coincides with the condition imposed in \citep{natarajan2013learning}. Similarly we can prove that 
\begin{align}
&\PP(\R^*=+1, \R=-1)  - \PP(\R^*=+1) \PP(\R=-1) \nonumber  \\
=&-\PP(Y=+1)\PP(Y=-1)(1-e_{+1}-e_{-1})(1-e^*_{+1}-e^*_{-1}) <0\label{eqn:cross} 
\end{align}
The other entries for $\PP(\R^*=-1,\R=-1)  -  \PP(\R^*=-1)\PP(\R=-1)$ and $\PP(\R^*=-1,\R=+1)  -  \PP(\R^*=-1)\PP(\R=+1)$ are symmetric.
Therefore the sign matrix of above score matrix is exactly the diagonal matrix.
\end{proof}

\section*{Proof for Lemma \ref{LEM:AFFINE}}
\begin{proof}
We denote by $X_{n_1},  \tilde{Y}_{n_2}$ the random variable corresponding to the peer samples $x_{n_1}, \tilde{y}_{n_2}$. 

First we have
\begin{align}
  \E[\ell_\text{peer}(f(X), \tilde{Y})] = \E[\ell(f(X), \tilde{Y})] - \mathbb E[\ell(f(X_{n_1}), \tilde{Y}_{n_2})] \label{eqn:pl0}  
\end{align}
Consider the two terms on the RHS separately.
\begin{align}
 & \mathbb E[\ell(f(X), \tilde{Y})]\nonumber \\
=& \mathbb E_{X,Y=-1} \bigl [\mathbb P(\tilde{Y}=-1 |X, Y=-1) \cdot \ell(f(X), -1) +
                       \mathbb P(\tilde{Y}=+1 | X, Y=-1) \cdot \ell(f(X), +1)\bigr]\nonumber\\
 &+\mathbb E_{X,Y=+1} \bigl [\mathbb P(\tilde{Y}=+1 |X, Y=+1) \cdot \ell(f(X), +1) +
                       \mathbb P(\tilde{Y}=-1 |X, Y=+1) \cdot \ell(f(X), -1)\bigr]\nonumber\\
=& \mathbb E_{X,Y=-1} \bigl [\mathbb P(\tilde{Y}=-1 | Y=-1) \cdot \ell(f(X), -1) +
                       \mathbb P(\tilde{Y}=+1 | Y=-1) \cdot \ell(f(X), +1)\bigr]\nonumber\\
 &+\mathbb E_{X,Y=+1} \bigl [\mathbb P(\tilde{Y}=+1 | Y=+1) \cdot \ell(f(X), +1) +
                       \mathbb P(\tilde{Y}=-1 | Y=+1) \cdot \ell(f(X), -1)\bigr]\tag{Independence between $\tilde{Y}$ and $X$ given $Y$}\nonumber\\
=& \mathbb E_{X,Y=-1} \bigl [(1-e_{-1})  \ell(f(X), -1) + e_{-1}  \ell(f(X), +1)\bigr]\nonumber\\
 &+\mathbb E_{X,Y=+1} \bigl [(1-e_{+1}) \ell(f(X), +1) + e_{+1}  \ell(f(X), -1)\bigr ] \label{eqn:pl1}
\end{align}
 The above is done mostly via law of total probability and using the assumption that $\tilde{Y}$ is conditionally (on $Y$) independent of $X$. Subtracting and adding $ e_{+1}\cdot\ell(f(X), -1) $ and $e_{-1}\cdot\ell(f(X), +1)$ to the two expectation terms separately we have
 \begin{align*}
\text{Eqn. (\ref{eqn:pl1})}=& \mathbb E_{X,Y=-1} \bigl [(1-e_{-1}-e_{+1}) \cdot \ell(f(X), -1) + e_{+1}\cdot\ell(f(X), -1) 
                       + e_{-1} \cdot \ell(f(X), +1)\bigr] \\
 &+\mathbb E_{X,Y=+1} \bigl [(1-e_{-1}-e_{+1}) \cdot \ell(f(X), +1) + e_{-1}\cdot\ell(f(X), +1)
                       + e_{+1} \cdot \ell(f(X), -1)\bigr]\\
=& (1-e_{-1}-e_{+1}) \cdot \mathbb E_{X,Y} \bigl[\ell(f(X), Y) \bigr] + 
   \mathbb E_X \bigl [e_{+1} \cdot \ell (f(X), -1) + e_{-1} \cdot \ell(f(X), +1)\bigr]
\end{align*}
And consider the second term:
\begin{align*}
 &\mathbb E[\ell(f(X_{n_1}), \tilde{Y}_{n_2})]\\
=& \mathbb E_X [\ell(f(X), -1)] \cdot \mathbb P(\tilde{Y}=-1) + 
   \mathbb E_X [\ell(f(X), +1)] \cdot \mathbb P(\tilde{Y}=+1) \tag{Independence between $n_1$ and $n_2$} \\
=& \mathbb E_X \bigl[ (e_{+1} \cdot p + (1-e_{-1})(1-p)) \cdot \ell(f(X), -1) \\
&+
                  \left((1-e_{+1}) p + e_{-1}(1-p)\right) \cdot \ell(f(X), +1)\bigr]\tag{Expressing $\PP(\tilde{Y})$ using $p$ and $e_{+1}, e_{-1}$}\\
=& \mathbb E_X \bigl[ (1-e_{-1}-e_{+1})(1-p) \cdot \ell(f(X), -1) \\
&+                (1-e_{-1}-e_{+1})p \cdot \ell(f(X), +1) \bigr]\\
 &+\mathbb E_X \bigl[ (e_{+1} \cdot p + e_{+1} (1-p)) \cdot \ell(f(X), -1) \\
 &+ 
                 (e_{-1} (1-p) + e_{-1} p) \cdot \ell(f(X), +1)\bigr]\\
=& (1-e_{-1}-e_{+1}) \cdot \mathbb E[\ell(f(X_{n_1}), Y_{n_2})] \\
&+ 
   \mathbb E_X\bigl[e_{+1} \cdot \ell(f(X), -1) + e_{-1} \cdot \ell (f(X), +1)\bigr]
\end{align*}
Subtracting the first and second term on RHS of Eqn. (\ref{eqn:pl0}):
\begin{align}
\mathbb E[\ell_\text{peer}(f(X), \tilde{Y})] 
=&\E[\ell(f(X), \tilde{Y})] - \mathbb E[\ell(f(X_{n_1}), \tilde{Y}_{n_2})]\nonumber \\
=& (1-e_{-1}-e_{+1}) \cdot \mathbb E[\ell_\text{peer}(f(X), Y)]
\end{align}
\end{proof}

\paragraph{Multi-class extension for 0-1 loss}
\begin{proof}

\rev{

 We denote by $Q$ a transition matrix that characterizes the relationships between noisy label $\tilde{Y}$ and the true label $Y$. The $(i,j)$ entry of $Q$ is defined as $Q_{ij} = \PP(\tilde{Y} = j| Y=i)$. We write $Q_{ij} = q_{ij}$.

 Consider the following case: suppose the noisy labels have the same probability of flipping to a specific wrong class, 
 that is, we pose the following conditions: 
$q_{ij} = q_{kj}$, for all $j\ne k\ne i$. 
This condition allows us to define $K$ new quantities:
\begin{align}
 e_j = q_{ij} ~~~\text{for all}~~i\ne j,~~~ q_{ii} = 1 - \sum_{j\ne i} e_j\label{eqn:qii}
\end{align}
Note that this condition is easily satisfied for the binary case since there is only one other class to be flipped to wrongly.

We show that $M(\cdot)$ is a diagonal matrix when $\sum_{j=1}^K e_{j} < 1$, a similar condition as $e_{-1} + e_{+1} < 1$. 

Notice the following facts:
\[
 \mathbb E[\BR(f(X), \tilde{Y})]  - \mathbb E[\BR(f(X_{n_1}), \tilde{Y}_{n_2})] =\PP(f(X) \neq \tilde{Y})- \PP(f(X_{n_1}) \neq \tilde{Y}_{n_2}) = - \PP(f(X) = \tilde{Y}) + \PP(f(X_{n_1}) = \tilde{Y}_{n_2})
\]
and using Eqn. \ref{eqn:qii} we have
\[
\sum_{k=1}^K\PP(Y=k) \cdot q_{kj} = \PP(Y=j) \left(1-\sum_{k\neq j} e_k\right) + (1-\PP(Y=j)) e_j = \left(1-\sum_k e_k\right) \PP(Y=j) + e_j
\]
Then
\begin{align*}
&\PP(f(X) = \tilde{Y}) \\
=& \sum_{k=1}^K \PP(Y=k) \sum_{j=1}^K \PP(f(X)=j, \tilde{Y}=j|Y=k) \\
=& \sum_{k=1}^K \PP(Y=k) \sum_{j=1}^K \PP(f(X)=j|Y=k) \cdot \PP(\tilde{Y}=j|Y=k) \tag{Conditional independence}\\
=& \sum_{k=1}^K \PP(Y=k) \sum_{j=1}^K \PP(f(X)=j|Y=k) \cdot q_{kj}\\
=&\sum_{j=1}^K\sum_{k=1}^K \PP(f(X)=j|Y=k)\PP(Y=k) \cdot q_{kj}\\
=&\sum_{j=1}^K \PP(f(X)=j|Y=j)\PP(Y=j)\left(1-\sum_{k\neq j}e_k\right) + \sum_{j=1}^K\sum_{k\neq j}\PP(f(X)=j|Y=k)\PP(Y=k) e_j \tag{Eqn. \ref{eqn:qii}}\\
=&\sum_{j=1}^K \PP(f(X)=j|Y=j)\PP(Y=j)\left(1-\sum_{k\neq j}e_k\right)+ \sum_{j=1}^K e_j \cdot\left(\PP(f(X) = j)-\PP(f(X)=j|Y=j)\PP(Y=j)\right)\\
=&\left(1-\sum_{k} e_k\right) \sum_{j=1}^K \PP(f(X)=j|Y=j)\PP(Y=j) + \sum_{j=1}^K e_j \cdot \PP(f(X) = j)
 \end{align*}
 Now consider the following 
 \begin{align*}
     &\PP(f(X_{n_1}) = \tilde{Y}_{n_2})\\
     =&\sum_{j=1}^K \PP(f(X)=j) \PP(\tilde{Y} = j)\\
     =&\sum_{j=1}^K \PP(f(X)=j)  \sum_{k=1}^K\PP(Y=k)\cdot q_{kj}\\
     =&\sum_{j=1}^K \PP(f(X)=j)\left((1-\sum_k e_k) \PP(Y=j) + e_j\right) \tag{Eqn. (\ref{eqn:2nd})}
 \end{align*}
Therefore 
\begin{align*}
 &\mathbb E[\BR(f(X), \tilde{Y})]  - \mathbb E[\BR(f(X_{n_1}), \tilde{Y}_{n_2})]\\
 =& -\PP(f(X) = \tilde{Y})+\PP(f(X_{n_1}) = \tilde{Y}_{n_2}) \\
 =&- \left(1-\sum_k e_k\right) \sum_{j=1}^K\left(\PP(f(X)=j|Y=j)\PP(Y=j)-\PP(f(X)=j)\PP(Y=j) \right)\\
 =& \left(1-\sum_k e_k\right)\sum_{j=1}^K \PP(Y=j) \left( \PP(f(X) = j) - \PP(f(X) = j|Y=j) \right)
\end{align*}
For clean distribution we have
\begin{align*}
  \mathbb E[\BR(f(X), Y)] &= \PP(f(X) \neq Y) = 1- \sum_{j=1}^K  \PP(f(X) = j|Y=j)\PP(Y=j) 
 \end{align*}
For the second term above we have
\begin{align*}
    \mathbb E[\BR(f(X_{n_1}), Y_{n_2})] = \PP(X_{n_1}) \neq Y_{n_2}) = 1-\sum_{j=1}^K \PP(f(X) = j)\PP(Y=j) 
\end{align*}
Therefore 
\begin{align*}
    &\mathbb E[\BR(f(X), Y)] -     \mathbb E[\BR(f(X_{n_1}), Y_{n_2})]  \\
    =& \sum_{j=1}^K \PP(Y=j) \left(  \PP(f(X) = j)- \PP(f(X) = j|Y=j)\right)
\end{align*}
and the above concludes 
\[
\mathbb E[\BR(f(X), \tilde{Y})]  - \mathbb E[\BR(f(X_{n_1}), \tilde{Y}_{n_2})] = \left(1-\sum_k e_k\right)(\mathbb E[\BR(f(X), Y)] -     \mathbb E[\BR(f(X_{n_1}), Y_{n_2})] ),
\]
where the RHS above is the peer loss computed on the clean distribution. 

Again when we have balanced label distribution that $\PP(Y=j) = 1/K$, 
\begin{align}
\sum_{j=1}^K \PP(Y=j) \left( \PP(f(X) = j)-\PP(f(X) = j|Y=j) \right) =\frac{1}{K} - \underbrace{\sum_{j=1}^K \PP(Y=j)  \PP(f(X) = j|Y=j)}_{\text{Accuracy = 1 - 0-1 risk}}  \label{eqn:2nd}
\end{align}
Therefore minimizing peer loss on the clean distribution returns the same minimizer as for the true and clean 0-1 risk.}
\end{proof}

\section*{Proof for Theorem \ref{THM:EQUAL}}
\begin{proof}
From Lemma \ref{LEM:AFFINE} we know
\begin{align*}
&\mathbb E[\ell_\text{peer}(f(X), \tilde{Y})]\\
=& (1-e_{-1}-e_{+1}) \cdot \mathbb E[\ell_\text{peer}(f(X), Y)] \tag{Lemma \ref{LEM:AFFINE}}\\
=& (1-e_{-1}-e_{+1}) \cdot \biggl(\mathbb E[\ell(f(X), Y)] - \mathbb E[\ell(f(X_{n_1}), Y_{n_2})]\biggr) \\
=& (1-e_{-1}-e_{+1}) \cdot  \mathbb E\biggl[\ell(f(X), Y)]
   -0.5 \cdot \mathbb E_X [\ell(f(X), -1)]
   -0.5 \cdot \mathbb E_X [\ell(f(X), +1)]\biggr) \tag{Independence between $n_1$ and $n_2$, and equal prior}
   \end{align*}
  When $\ell$ is the 0-1 loss we have $\ell(f(X), -1)+\ell(f(X), +1)=1, \forall x$, and therefore
 \begin{align*}
\mathbb E[\ell_\text{peer}(f(X), \tilde{Y})] = (1-e_{-1}-e_{+1}) \cdot \biggl(\mathbb E[\ell(f(X), Y)] -1 \biggr)
\end{align*}
With above we proved $\tilde{f}^*_{\BPeer} \in \argmin_{f \in \F} R_{\mathcal D}(f)$.
\end{proof}
\section*{Proof for Theorem \ref{THM:pneq}}

\begin{proof}
Apply Lemma \ref{LEM:AFFINE} we know
\begin{align*}
    \E[\lPeer(f(X),\tilde{Y})] = (1-e_{+1}-e_{-1}) \E[\lPeer(f(X),Y)],
\end{align*}
Denote by $f^*_{\mathcal F} \in \argmin_{f \in \mathcal F} R_{\mathcal D}(f)$. From the optimality of $\tilde{f}^*_{\BPeer}$ we have
\[
 \E[\lPeer(\tilde{f}^*_{\BPeer}(X),\tilde{Y})] \leq  \E[\lPeer(f^*_{\mathcal F}(X),\tilde{Y})] \Leftrightarrow  \E[\lPeer(\tilde{f}^*_{\BPeer}(X),Y)] \leq  \E[\lPeer(f^*_{\mathcal F}(X),Y)]
\]
i.e., 
\begin{align}
  & R_{\mathcal D}(\tilde{f}^*_{\BPeer}) - p \cdot \mathbb E_X [\ell(\tilde{f}^*_{\BPeer}(X), +1)]
-(1-p)  \cdot \mathbb E_X [\ell(\tilde{f}^*_{\BPeer}(X), -1)]  \nonumber \\
& \leq   R_{\mathcal D}(f^*_{\mathcal F})-p  \cdot \mathbb E_X [\ell(f^*_{\mathcal F}(X), +1)]
-(1-p)  \cdot \mathbb E_X [\ell(f^*_{\mathcal F}(X), -1)], \label{eqn:pneq1}
\end{align}
Note $\forall f$:
\begin{align}
&\biggl |p  \cdot \mathbb E_X [\ell(f(X), +1)]
+ (1-p )  \cdot \mathbb E_X [\ell(f(X), -1)]\nonumber \\
&~~~-0.5 \cdot \mathbb E_X [\ell(f(X), +1)]-0.5 \cdot \mathbb E_X [\ell(f(X), -1)] \biggr | \nonumber \\
 =& |p -0.5|  \cdot \bigl  |\mathbb E_X [\ell(f(X), +1)] - \mathbb E_X [\ell(f(X), -1)]\bigr | \nonumber \\
 \leq & |p -0.5|. 
 \label{eqn:pneg2}
\end{align}
Then we have
\begin{align*}
& R_{\mathcal D}(\tilde{f}^*_{\BPeer}) - 0.5\\
&= R_{\mathcal D}(\tilde{f}^*_{\BPeer}) - 0.5 \cdot \mathbb E_X [\ell(\tilde{f}^*_{\BPeer}(X), +1)]-0.5 \cdot \mathbb E_X [\ell(\tilde{f}^*_{\BPeer}(X), -1)] \tag{$\ell(\tilde{f}^*_{\BPeer}(X), +1)+ \ell(\tilde{f}^*_{\BPeer}(X), -1)=1$}\\
&\leq   R_{\mathcal D}(\tilde{f}^*_{\BPeer}) - p \cdot \mathbb E_X [\ell(\tilde{f}^*_{\BPeer}(X), +1)]-(1-p)  \cdot \mathbb E_X [\ell(\tilde{f}^*_{\BPeer}(X), -1)]  + |p -0.5|   \tag{Eqn. (\ref{eqn:pneg2})}\\
&\leq  R_{\mathcal D}(f^*_{\mathcal F})-p \cdot \mathbb E_X [\ell(f^*_{\mathcal F}(X), +1)]
-(1-p)  \cdot \mathbb E_X [\ell(f^*_{\mathcal F}(X), -1)] + |p -0.5|  
\tag{Eqn. (\ref{eqn:pneq1})}\\
&\leq  R_{\mathcal D}(f^*_{\mathcal F})-0.5  \cdot \mathbb E_X [\ell(f^*_{\mathcal F}(X), +1)]
-0.5 \cdot \mathbb E_X [\ell(f^*_{\mathcal F}(X), -1)] + 2|p -0.5|   
\tag{Eqn. (\ref{eqn:pneg2})}\\
&=  R_{\mathcal D}(f^*_{\mathcal F}) - 0.5 + 2|p -0.5| 
\tag{$\ell(f^*_{\mathcal F}(X), +1)+\ell(f^*_{\mathcal F}(X), -1)=1$}
\end{align*}
Therefore
\[
R_{\mathcal D}(\tilde{f}^*_{\BPeer}) -R_{\mathcal D}(f^*_{\mathcal F}) \leq 2|p -0.5| = |\delta_p|.
\]
\end{proof}

\section*{Proof for Lemma \ref{LEM:alpha1}}

\begin{proof}

\begin{align}
 & \mathbb E[\BR_\text{peer}(f(X), \tilde{Y})] \nonumber \\
=& (1-e_{-1}-e_{+1}) \cdot \mathbb E[\ell_\text{peer}(f(X), Y)] \tag{Lemma \ref{LEM:AFFINE}}\\
=& (1-e_{-1}-e_{+1}) \cdot (\PP(f(X) \neq Y) - \PP(f(X_{n_1}) \neq Y_{n_2}))\nonumber \\
=& (1-e_{-1}-e_{+1}) \cdot ( \mathbb P(f(X)=-1, Y=+1) + \mathbb P(f(X)=+1, Y=-1) \nonumber \\
                     &~~~~-\mathbb P(f(X)=-1) \mathbb P(Y=+1) - \mathbb P(f(X)=+1) \mathbb P(Y=-1))\tag{Independence between $n_1$ and $n_2$} \\
=& (1-e_{-1}-e_{+1}) \cdot ( p \cdot R_{+1} + (1-p) \cdot R_{-1}\nonumber  \\
&~~~~- p \cdot \mathbb P(f(X)=-1) 
                                                 - (1-p) \cdot \mathbb P(f(X)=+1)) \tag{Law of total probability}\\
=& (1-e_{-1}-e_{+1}) \cdot ( p \cdot R_{+1} + (1-p) \cdot R_{-1} \nonumber \\
&~~~~- p \cdot \bigl (p\cdot R_{+1} + (1-p) \cdot (1-R_{-1})\bigr) 
                                                 - (1-p) \cdot \bigl (p\cdot (1-R_{+1}) + (1-p)\cdot R_{-1}) \bigr)\nonumber \\
=& 2 (1-e_{-1}-e_{+1}) \cdot p(1-p) \cdot (R_{-1}+R_{+1}-1) \label{eqn:lemma3}
\end{align}
Again since $ \mathbb E[\BR_\text{peer}(f(X), \tilde{Y})]$ is an affine transform of $R_{-1}+R_{+1}$ we conclude the proof.
\end{proof}

\section*{Proof for Theorem \ref{THM:weightedPeer}}

\begin{proof}

\begin{align*}
 & \mathbb E[\AlphaPeer(f(X), \tilde{Y})]  \\
=& \mathbb E[\BR(f(X), \tilde{Y})] - \alpha \cdot \mathbb E[\BR(f(X_{n_1}), \tilde{Y}_{n_2})] \\
=& \mathbb E[\BR_\text{peer} (f(X), \tilde{Y})] + (1-\alpha) \cdot \mathbb E[\BR(f(X_{n_1}), \tilde{Y}_{n_2})] \tag{Subtracting and adding $\mathbb E[\BR(f(X_{n_1}), \tilde{Y}_{n_2})]$ to the first and second term}\\
=& \mathbb E[\BR_\text{peer} (f(X), \tilde{Y})] +(1- \alpha) \cdot \PP(f(X_{n_1}) \neq \tilde{Y}_{n_2})\\
=& \mathbb E[\BR_\text{peer} (f(X), \tilde{Y})] + 
             (1-\alpha) \cdot \bigl(\mathbb P (f(X)=+1) \cdot \mathbb P(\tilde{Y}=-1) + 
                               \mathbb P (f(X)=-1) \cdot \mathbb P(\tilde{Y}=+1)\bigr)  
\end{align*}
Again the last equality is due to the independence between $n_1$ and $n_2$. Replace $\mathbb P (f(X)=+1)$ and $\mathbb P (f(X)=-1)$ as functions of $p, R_{+1}, R_{-1}$:
\begin{align*}
    &\mathbb P (f(X)=+1) = p\cdot (1-R_{+1}) + (1-p)\cdot R_{-1}\\
    &\mathbb P (f(X)=-1) = p \cdot R_{+1} + (1-p)(1-R_{-1}),
\end{align*}
we further have
\begin{align*}
 \mathbb E[\AlphaPeer(f(X), \tilde{Y})] =& \mathbb E[\BR_\text{peer} (f(X), \tilde{Y})] + 
   (1-\alpha) \cdot \biggl (\bigl (p\cdot (1-R_{+1}) + (1-p)\cdot R_{-1}\bigr) \cdot \mathbb P(\tilde{Y}=-1)\\ &+ 
                     \bigl(p \cdot R_{+1} + (1-p)(1-R_{-1})\bigr) \cdot \mathbb P(\tilde{Y}=+1)\biggr) \\
=& \mathbb E[\BR_\text{peer} (f(X), \tilde{Y})] + 
   (1-\alpha) \cdot (\mathbb P(\tilde{Y}=+1) - \mathbb P(\tilde{Y}=-1)) \cdot (p\cdot R_{+1}-(1-p)\cdot R_{-1}) + C \tag{$C$ is a constant: $C = (1-\alpha) \cdot \left((1-p)\cdot\mathbb P(\tilde{Y}=+1) + p\cdot\mathbb P(\tilde{Y}=-1)\right)$}\\
=& 2 (1-e_{-1}-e_{+1}) \cdot p(1-p) \cdot (R_{-1}+R_{+1}-1) \tag{Eqn. (\ref{eqn:lemma3}), Proof of Lemma \ref{LEM:alpha1}}\\
 & + (1-\alpha) \cdot (\mathbb P(\tilde{Y}=+1) - \mathbb P(\tilde{Y}=-1)) \cdot \bigl(p \cdot R_{+1}-(1-p) \cdot R_{-1}\bigr) + C \\
=& R_{+1} \cdot \biggl(2(1-e_{-1}-e_{+1}) \cdot p(1-p) + (1-\alpha) p \cdot (\mathbb P(\tilde{Y}=+1) - \mathbb P(\tilde{Y}=-1))\biggr)\\
 &+R_{-1} \cdot \biggl(2(1-e_{-1}-e_{+1}) \cdot p(1-p) - (1-\alpha) (1-p) \cdot (\mathbb P(\tilde{Y}=+1) - \mathbb P(\tilde{Y}=-1))\biggr) + C',
\end{align*}
where $C'$ is a constant:
\begin{align*}
C' = C - 2 (1-e_{-1}-e_{+1}) \cdot p(1-p)
\end{align*}

Let 
\[
\frac{p}{1-p} = 
\frac{2(1-e_{-1}-e_{+1}) \cdot p(1-p) + (1-\alpha) \cdot p \cdot (\mathbb P(\tilde{Y}=+1) - \mathbb P(\tilde{Y}=-1))}
     {2(1-e_{-1}-e_{+1}) \cdot p(1-p) - (1-\alpha) \cdot(1-p)\cdot (\mathbb P(\tilde{Y}=+1) - \mathbb P(\tilde{Y}=-1))}.
\] 
that
\[
\alpha
= 1 - (1-e_{-1}-e_{+1})\cdot\frac{\delta_p}{\delta_{\tilde{p}}}.
\]
we obtain that
\begin{align}
     \mathbb E[\AlphaPeer(f(X), \tilde{Y})] \propto (1-e_{-1}-e_{+1})\mathbb E[\BR(f(X), Y)] + \text{const.},\label{eqn:alphapeer}
\end{align}

concluding our proof. The last equation Eqn.(\ref{eqn:alphapeer}) also implies the following proposition: 
\begin{proposition}\label{prop:alpha}
For any $f,f'$, we have
$$ \mathbb E_{\tilde{\mathcal D}}[\AlphaPeer(f(X), \tilde{Y})] -  \mathbb E_{\tilde{\mathcal D}}[\AlphaPeer(f'(X), \tilde{Y})] \propto (1-e_{-1}-e_{+1})\bigl(\mathbb E[\BR(f(X), Y)]-\mathbb E[\BR(f'(X), Y)]\bigr).$$
\end{proposition}

\noindent $\mathbf{e_{+1} = e_{-1} = e}$: ~~When $e_{+1} = e_{-1} = e$, we have
\begin{align*}
    &\mathbb P(\tilde{Y}=+1) - \mathbb P(\tilde{Y}=-1)\\
    =& p \cdot \PP(\tilde{Y}=+1|Y=+1) + (1-p)  \cdot \PP(\tilde{Y}=+1|Y=-1) \\
    &- p \cdot \PP(\tilde{Y}=-1|Y=+1) - (1-p)  \cdot \PP(\tilde{Y}=-1|Y=-1)\\
    =& p \cdot (1-e) + (1-p) \cdot e- p \cdot e - (1-p)\cdot (1-e)\\
    =& p (1-2e) - (1-p) (1-2e)\\
    =&(p-(1-p)) \cdot (1-2e).
\end{align*}
That is 
\[
(1-e_{+1}-e_{-1})\frac{\delta_p}{\delta_{\tilde{p}}} = (1-2e) \cdot \frac{p-(1-p)}{\mathbb P(\tilde{Y}=+1) - \mathbb P(\tilde{Y}=-1)} = 1
\]
Therefore $\alpha=0$.
\end{proof}

\section*{Proof for Theorem \ref{THM:converge1}}
\begin{proof}

$\forall f$, using Hoeffding's inequality with probability at least $1-\delta$ 
\begin{align}
&\left|\hat{R}_{\AlphaPeer,\tilde{D}}(f) - R_{\AlphaPeer,\tilde{\mathcal D}}(f)\right|\nonumber \\
\leq & \sqrt{\frac{\log 2/\delta}{2N}}\left(\overline{\BR_{\alpha-\text{peer}}}-\underline{\BR_{\alpha-\text{peer}}}\right)\tag{$\overline{\BR_{\alpha-\text{peer}}} = 1, \underline{\BR_{\alpha-\text{peer}}} = -\alpha$}\nonumber\\
= &(1+\alpha) \sqrt{\frac{\log 2/\delta}{2N}}, \label{eqn:sample:1}
\end{align}
where $\overline{\BR_{\alpha-\text{peer}}} = 1, \underline{\BR_{\alpha-\text{peer}}} = -\alpha$ denote the upper and lower bound on $\BR_{\alpha-\text{peer}}$. 

Note we also have the following:
\begin{align*}
  & R_{\AlphaPeer,\tilde{\mathcal D}}(\hat{f}^*_{\AlphaPeer}) - R_{\AlphaPeer,\tilde{\mathcal D}}  (f^*_{\AlphaPeer})\\
  \leq & \hat{R}_{\AlphaPeer,\tilde{D}}(\hat{f}^*_{\AlphaPeer})-\hat{R}_{\AlphaPeer,\tilde{D}}(f^*_{\AlphaPeer})+\left(R_{\AlphaPeer,\tilde{\mathcal D}}(\hat{f}^*_{\AlphaPeer})-\hat{R}_{\AlphaPeer,\tilde{D}}(\hat{f}^*_{\AlphaPeer})\right)\\
  &~~~~+\left(\hat{R}_{\AlphaPeer,\tilde{D}}(f^*_{\AlphaPeer})-R_{\AlphaPeer,\tilde{\mathcal D}}  (f^*_{\AlphaPeer})\right)\\
  \leq &0+2 \max_{f}|\hat{R}_{\AlphaPeer,\tilde{D}}(f) - R_{\AlphaPeer,\tilde{\mathcal D}}(f)|
\end{align*}
Now we show
\begin{align*}
&R_{\mathcal D}(\hat{f}^*_{\AlphaStar}) -  R^*\\
=&R_{\mathcal D}(\hat{f}^*_{\AlphaStar}) - R_{\mathcal D}  (f^*_{\AlphaStar}) \tag{Theorem \ref{THM:weightedPeer}} \\
=& \frac{1}{1-e_{-1}-e_{+1}} \bigl(R_{\AlphaStar,\tilde{\mathcal D}}(\hat{f}^*_{\AlphaStar}) - R_{\AlphaStar,\tilde{\mathcal D}}  (f^*_{\AlphaStar})\bigr) \tag{Proposition \ref{prop:alpha}} \\
\leq &\frac{2}{1-e_{-1}-e_{+1}} \max_{f }\left|\hat{R}_{\AlphaStar,\tilde{D}}(f) - R_{\AlphaStar,\tilde{\mathcal D}}(f)\right|\\
\leq &
\frac{2(1+\alpha^*)}{1-e_{-1}-e_{+1}}\sqrt{\frac{\log 2/\delta}{2N}}.\tag{Eqn. (\ref{eqn:sample:1})}
\end{align*}
We conclude the proof.
\end{proof}

\section*{Proof for Theorem \ref{THM:calibration}}

\begin{proof}
We start with condition (1). From Lemma \ref{LEM:AFFINE},
\begin{align*}
&\mathbb E[\ell_\text{peer}(f(X), \tilde{Y})]\\
=&(1-e_{-1}-e_{+1}) \cdot \mathbb E[\ell_{\text{peer}}(f(X), Y)]\\
=& (1-e_{-1}-e_{+1}) \cdot \biggl(\mathbb E[\ell(f(X), Y)]
   -0.5 \cdot \mathbb E[\ell(f(X), -1)]
   -0.5 \cdot \mathbb E[\ell(f(X), +1)]\biggr) 
   \end{align*}
The above further derives as
\begin{align*}
&\mathbb E[\ell_\text{peer}(f(X), \tilde{Y})]\\
=& (1-e_{-1}-e_{+1}) \cdot \biggl(\mathbb E[\ell(f(X), Y)]
   -0.5 \cdot \mathbb E[\ell(f(X), Y)]
   -0.5 \cdot \mathbb E[\ell(f(X), -Y)]\biggr) \\
   =& \frac{1-e_{-1}-e_{+1}}{2} \cdot \bigl(\mathbb E[\ell(f(X), Y)]
   -\mathbb E[\ell(f(X), -Y)]\bigr)
   \end{align*}

Denote by $c:=\frac{2}{1-e_{-1}-e_{+1}}$ we have 
\[
\E[\ell(f(X), Y)] = c\cdot \mathbb E[\ell_\text{peer}(f(X), \tilde{Y})] + \mathbb E[\ell(f(X), -Y)]
\]
Then
\begin{align*}
     &\E[\ell(f(X), Y)] - \E[\ell(f^*_{\ell}(X), Y)] - ( \mathbb E[\ell(f(X), -Y)]-\E[\ell(f^*_{\ell}(Y), -Y))] \\
     =& c\cdot (\mathbb E[\ell_\text{peer}(f(X), \tilde{Y})] - \mathbb E[\ell_\text{peer}(f^*_{\ell}(X), \tilde{Y})])\\
     \leq &  c\cdot (\mathbb E[\ell_\text{peer}(f(X), \tilde{Y})] - \mathbb E[\ell_\text{peer}(f^*_{\ell_{\text{peer}}}(X), \tilde{Y})])
 \end{align*}
Further by our conditions we know
\begin{align*}
\E[\ell(f(X), Y)] - &\E[\ell(f^*_{\ell}(X), Y)] - ( \mathbb E[\ell(f(X), -Y)]-\E[\ell(f^*_{\ell}(Y), -Y))] \\
&\geq \E[\ell(f(X), Y)] - \E[\ell(f^*_{\ell}(X), Y)].
\end{align*}
Therefore we have proved
\[
\mathbb E[\ell_\text{peer}(f(X), \tilde{Y})] - \mathbb E[\ell_\text{peer}(f^*_{\ell_{\text{peer}}}(X), \tilde{Y})] \geq \frac{1}{c}\bigl( \E[\ell(f(X), Y)] - \E[\ell(f^*_{\ell}(X), Y)]\bigr).
\]
Since $\ell(\cdot)$ is calibrated, and according to Proposition \ref{prop:alpha} and Theorem \ref{THM:EQUAL}:
\begin{align*}
   &\mathbb E_{\tilde{\mathcal D}}[\AlphaPeer(f(X), \tilde{Y})] -  \mathbb E_{\tilde{\mathcal D}}[\AlphaPeer(f^*_{\ell}(X), \tilde{Y})] \\
   =&(1-e_{-1}-e_{+1})\bigl(\mathbb E[\BR(f(X), Y)]-\mathbb E[\BR(f^*_{\ell}(X), Y)]\bigr)\\
   \leq& (1-e_{-1}-e_{+1})\cdot \Psi^{-1}_{\ell}\left(\E[\ell(f(X), Y)] - \E[\ell(f^*_{\ell}(X), Y)]\right)\\
   \leq &(1-e_{-1}-e_{+1})\cdot \Psi^{-1}_{\ell}\left(c\cdot (\mathbb E[\ell_\text{peer}(f(X), \tilde{Y})] - \mathbb E[\ell_\text{peer}(f^*_{\ell_{\text{peer}}}(X), \tilde{Y})])\right).
\end{align*}
Therefore $\Psi_{\lPeer}(x) = \frac{1}{c}\Psi_{\ell}(\frac{x}{1-e_{-1}-e_{+1}})$, where $\Psi_{\ell}$ is the calibration transformation function for $\ell$. It's straight-forward to verify that $\Psi_{\lPeer}(x)$ satisfies the conditions in Definition \ref{def:cc}, when $\Psi_{\ell}$ satisfied it, $c > 0$ and $1-e_{-1}-e_{+1} >0$. We conclude the proof.

Now we check condition (2). 
Denote $\tilde{p}_y = p_y (1-e_y) + (1-p_y) e_{-y}
$ (marginal distribution of the noisy label), where $p_{+1}=p, p_{-1}=1-p$, then we have :
\begin{align*}
 &\mathbb E[\lAlphaPeer(f(X), \tilde{Y})] \\
 =& \mathbb E[\ell(f(X), \tilde{Y}) - \alpha \cdot \ell(f(X_{n_1}), \tilde{Y}_{n_2})]\\
 =& \E\biggl [ (1-e_{Y})\ell(f(X), Y) + e_{Y} \cdot \ell(f(X), -Y) - \alpha \cdot \tilde{p}_Y\cdot \ell(f(X), Y) - \alpha \cdot (1-\tilde{p}_Y) \ell(f(X), -Y)\biggr]\\
 =&\E\biggl[ (1-e_{Y}-\alpha \cdot \tilde{p}_Y) \ell(f(X), Y) + (e_{Y}-\alpha \cdot (1-\tilde{p}_Y))\ell(f(X), -Y)\biggr]
\end{align*}
$\alpha \cdot \tilde{p}_Y\cdot \ell(f(X), Y) + \alpha \cdot (1-\tilde{p}_Y) \ell(f(X), -Y)$ encodes the expectation of the peer term: due to the random sampling of $n_1$, each $X$ the same chance $\PP(X=x)$ being paired with other samples. Regardless of the realization of $Y$, the two terms are exactly one $\tilde{p}_{+1}\ell(f(X),+1)$ and one $\tilde{p}_{-1} \ell(f(X),-1)$ - this is due to the independence between $n_1$ and $n_2$.

Let $\phi(f(X)\cdot Y):=\ell(f(X), Y) $, we have
\begin{align*}
 &\mathbb E [\lAlphaPeer(f(X), \tilde{Y})]  =\E\biggl[ (1-e_{Y}-\alpha \cdot \tilde{p}_Y)\phi(f(X)\cdot Y) + (e_{Y}-\alpha \cdot (1-\tilde{p}_Y))\phi(-f(X)\cdot Y)\biggr].
\end{align*}
When
\[
e_{+1}-\alpha \cdot (1-\tilde{p}_{+1}) = e_{-1}-\alpha \cdot (1-\tilde{p}_{-1})
\]
we also know that
\[
1-e_{+1}-\alpha \cdot \tilde{p}_{+1} = 1-e_{-1}-\alpha \cdot \tilde{p}_{-1}
\]
This is because
\begin{align*}
    1-e_{+1}-\alpha \cdot \tilde{p}_{+1} + e_{+1}-\alpha \cdot (1-\tilde{p}_{+1}) = 1-\alpha
\end{align*}
and
\begin{align*}
    1-e_{-1}-\alpha \cdot \tilde{p}_{-1} + e_{-1}-\alpha \cdot (1-\tilde{p}_{-1}) = 1-\alpha
\end{align*}
From $e_{+1}-\alpha \cdot (1-\tilde{p}_{+1}) = e_{-1}-\alpha \cdot (1-\tilde{p}_{-1})
$ we obtain
\begin{align*}
    & e_{+1} - e_{-1} = \alpha (\tilde{p}_{-1} - \tilde{p}_{+1}) = \alpha (2\tilde{p}_{-1}-1)\\
    \Leftrightarrow &e_{+1} - e_{-1} = \alpha (2(1-p) (1-e_{-1})+2p \cdot e_{+1} - 1) \tag{$\tilde{p}_{-1} = (1-p) (1-e_{-1})+p \cdot e_{+1}$ }\\
    \Leftrightarrow &e_{+1} - e_{-1} = \alpha((1-2p)(1-e_{+1}-e_{-1})+e_{+1}-e_{-1})\\
    \Leftrightarrow & (1-\alpha) (e_{+1} - e_{-1}) =\alpha (1-2p)(1-e_{+1}-e_{-1})
\end{align*}
But when $1-e_{Y}-\alpha \cdot \tilde{p}_Y$ and $e_{Y}-\alpha \cdot (1-\tilde{p}_Y)$ are constants that are independent of $Y$, we can further denote
\begin{align*}
\varphi(f(X)\cdot Y)&:= (1-e_{Y}-\alpha \cdot \tilde{p}_Y)\phi(f(X)\cdot Y) + (e_{Y}-\alpha \cdot (1-\tilde{p}_Y))\phi(-f(X)\cdot Y)\\
&:=c_1 \cdot \phi(f(X)\cdot Y) + c_2 \cdot \phi(-f(X)\cdot Y)
\end{align*}
That is $\E[\varphi(f(X)\cdot Y)]=\E[\lAlphaPeer(f(X), \tilde{Y})]$. Therefore proving calibration for $\E[\lAlphaPeer(f(X), \tilde{Y})]$ is equivalent with proving the calibration property for $\E[\varphi(f(X)\cdot Y)]$.

We now introduce a theorem:
\begin{theorem}[Theorem 6, \citep{bartlett2006convexity}]
    Let $\varphi$ be convex. Then $\varphi$ is classification-calibrated if and only if it is differentiable at $0$ and $\varphi' < 0$.
\end{theorem}
We now show that $\varphi$ is convex:
\begin{align*}
\varphi''(\beta) =& c_1 \cdot \phi''(\beta) +c_2 \cdot \phi''(-\beta)\\
=&(1-e_{Y}-\alpha\cdot  \tilde{p}_Y)\cdot \phi''(\beta) +(e_{Y}-\alpha \cdot (1-\tilde{p}_Y))\phi''(\beta)\\
=&   (1-e_{Y}-\alpha \cdot  \tilde{p}_Y +e_{Y}-\alpha \cdot (1-\tilde{p}_Y)) \phi''(\beta)\\
=& (1-\alpha)\phi''(\beta)>0
\end{align*}
when $\alpha < 1$. The last inequality is due to the fact that $\ell$ is convex.

Secondly we show the first derivative of $\varphi$ is negative at $0$: $\varphi'(0) < 0$:
\begin{align}
\varphi'(0) =& c_1 \cdot \phi'(0) -c_2 \cdot \phi'(0) \nonumber \\
=&(1-e_{Y}-\alpha \cdot  \tilde{p}_Y)\cdot \phi'(0) -(e_{Y}-\alpha \cdot (1-\tilde{p}_Y))\phi'(0) \nonumber \\
=& (1- 2e_{Y} + \alpha(1-2\tilde{p}_Y) ) \phi'(0) \label{eqn:varder}
\end{align}
Recall that
$
\tilde{p}_y = p_y (1-e_y) + (1-p_y) e_{-y}
$. Plug into to Eqn. (\ref{eqn:varder}) we have
\begin{align}
\varphi'(0)
=& \bigl (1- 2e_{Y} + \alpha(1-2\tilde{p}_Y) \bigr ) \phi'(0) \nonumber \\
=& \biggl( (1-\alpha \cdot  p_Y) (1-2e_Y)+\alpha(1-p_Y) (1-e_{-Y}) \biggr)\phi'(0)
\end{align}
Since $(1-\alpha \cdot p_Y) (1-2e_Y)+\alpha(1-p_Y) (1-e_{-Y}) > 0$ and $\phi'(0) < 0$ (due to calibration property of $\ell$, Theorem 6 of \cite{bartlett2006convexity}), we proved that $\varphi'(0) < 0$. Then based on Theorem 6 of \cite{bartlett2006convexity}, we know $\ell_{\lAlphaPeer}$ is classification calibrated.
\end{proof}

\section*{Proof for Theorem \ref{THM:GEN}}

\begin{proof}

We first prove the following Rademacher complexity bound:
 \begin{lemma}\label{LEM:RAD}
 Let $\Re(\F)$ denote the Rademacher complexity of $\F$. $L$ denote the Lipschitz constant of $\ell$. Then with probability at least $1-\delta$, $\max_{f \in \F} |\hat{R}_{\lAlphaPeer,\tilde{D}}(f) - R_{\lAlphaPeer,\tilde{\mathcal D}}(f)| \leq 
2 (1+\alpha) L\cdot \Re(\F)+  \sqrt{\frac{\log 4/\delta}{2N}}\left(1+(1+\alpha)(\bar{\ell}-\underline{\ell})\right).$ 
\end{lemma}
Note we also have the following $\forall \alpha$:
\begin{align*}
  & R_{\lAlphaPeer,\tilde{\mathcal D}}(\hat{f}^*_{\lAlphaPeer}) - R_{\lAlphaPeer,\tilde{\mathcal D}}  (f^*_{\lAlphaPeer})\\
  \leq & \hat{R}_{\lAlphaPeer,\tilde{D}}(\hat{f}^*_{\lAlphaPeer})-\hat{R}_{\lAlphaPeer,\tilde{D}}(f^*_{\lAlphaPeer})\\
  &~~~~+(R_{\lAlphaPeer,\tilde{\mathcal D}}(\hat{f}^*_{\lAlphaPeer})-\hat{R}_{\lAlphaPeer,\tilde{D}}(\hat{f}^*_{\lAlphaPeer}))\\
  &~~~~+(\hat{R}_{\lAlphaPeer,\tilde{D}}(f^*_{\lAlphaPeer})-R_{\lAlphaPeer,\tilde{\mathcal D}}  (f^*_{\lAlphaPeer}))\\
  \leq &0+2 \max_{f \in \mathcal F}|\hat{R}_{\lAlphaPeer,\tilde{D}}(f) - R_{\lAlphaPeer,\tilde{\mathcal D}}(f)|
\end{align*}
Then apply the calibration condition we have
\begin{align*}
&R_{\mathcal D}(\hat{f}^*_{\lAlphaStar}) - R^* \\
=& \frac{1}{1-e_{-1}-e_{+1}}\bigl(R_{\AlphaStar,\tilde{\mathcal D}}(\hat{f}^*_{\lAlphaStar}) -  R_{\AlphaStar,\tilde{\mathcal D}}(f^*)\bigr) \tag{Proposition \ref{prop:alpha}}\\
=& \frac{1}{1-e_{-1}-e_{+1}}\bigl(R_{\AlphaStar,\tilde{\mathcal D}}(\hat{f}^*_{\lAlphaStar}) -  R_{\AlphaStar,\tilde{\mathcal D}}(\tilde{f}^*_{\AlphaStar})\bigr) \tag{Theorem \ref{THM:pneq}}\\
\leq & \frac{1}{1-e_{-1}-e_{+1}}\Psi^{-1}_{\lAlphaStar}\biggl(\min_{f \in \mathcal F}R_{\lAlphaStar,\tilde{\mathcal D}}(f)-\min_f R_{\lAlphaStar,\tilde{\mathcal D}}(f) ~~~~~\tag{Calibration of $\AlphaStar$}\\
&+R_{\lAlphaStar,\tilde{\mathcal D}}(\hat{f}^*_{\lAlphaStar}) - R_{\lAlphaStar,\tilde{\mathcal D}}  (f^*_{\lAlphaStar})\biggr)\\
\leq &\frac{1}{1-e_{-1}-e_{+1}}\Psi^{-1}_{\lAlphaStar}\biggl(\min_{f \in \mathcal F}R_{\lAlphaStar,\tilde{\mathcal D}}(f)-\min_f R_{\lAlphaStar,\tilde{\mathcal D}}(f)\\
&+2 \max_{f \in \mathcal F}|\hat{R}_{\lAlphaStar, \tilde{D}}(f) - R_{\lAlphaStar,\tilde{\mathcal D}}(f)|\\
\leq &\frac{1}{1-e_{-1}-e_{+1}}\Psi^{-1}_{\lAlphaStar}\biggl(\min_{f \in \mathcal F}R_{\lAlphaStar,\tilde{\mathcal D}}(f)-\min_f R_{\lAlphaStar,\tilde{\mathcal D}}(f)~~~~~\tag{Lemma \ref{LEM:RAD}}\\
&+4(1+\alpha^*) L\cdot \Re(\F)+2\sqrt{\frac{\log 4/\delta}{2N}}\left(1+(1+\alpha^*)(\bar{\ell}-\underline{\ell})\right) 
\biggr),
\end{align*}
with probability at least $1-\delta$.
\end{proof}
\section*{Proof for Lemma \ref{LEM:RAD}}

\begin{proof}
 Define $\tilde{p}_y = p_y (1-e_y) + (1-p_y) e_{-y} \in (0,1)
$ (marginal distribution of the noisy label), where $p_{+1}=p, p_{-1}=1-p$; and
define the following loss function:
\[
\tilde{\ell}(x_n,\tilde{y}_n):=\ell(f(x_n),\tilde{y}_n) - \alpha \cdot \tilde{p}_{\tilde{y}_n}\ell(f(x_n),\tilde{y}_n) - \alpha \cdot (1-\tilde{p}_{\tilde{y}_n})\ell(f(x_n),-\tilde{y}_n)
\]
Due to the random sampling of $n_1,n_2$ for the peer term, we have
\[
\E_{n_1,n_2}\left[ \frac{1}{N}\sum_{n=1}^N \lAlphaPeer(f(x_n),\tilde{y}_n)\right] = \frac{1}{N}\sum_{n=1}^N \tilde{\ell}(f(x_n),\tilde{y}_n).
\]
In above,  $\alpha \cdot \tilde{p}_{\tilde{y}_n}\ell(f(x_n),\tilde{y}_n) + \alpha \cdot (1-\tilde{p}_{\tilde{y}_n})\ell(f(x_n),-\tilde{y}_n)$ encodes the expectation of the peer term (similar to the arguments in Theorem 6) each $x_n$ has $\frac{1}{N}$ chance being paired with each of the training samples - so the expected number of count is 1. Regardless of $\tilde{y}_n$, the two terms are exactly one $\tilde{p}_{+1}\ell(f(x_n),+1)$ and one $\tilde{p}_{-1} \ell(f(x_n),-1)$ - this is due to the independence between $n_1$ and $n_2$.

Then via Hoeffding inequality, with probability at least $1-\delta$ (over randomness of $n_1,n_2$), 
\begin{align}
    \biggl |\frac{1}{N}\sum_{n=1}^N &\lAlphaPeer(f(x_n),\tilde{y}_n) - \frac{1}{N}\sum_{n=1}^N \tilde{\ell}(f(x_n),\tilde{y}_n) \biggr| \leq \sqrt{\frac{\log 2/\delta}{2N}} \cdot (\overline{\ell_{\alpha-\text{peer}}}-\underline{\ell_{\alpha-\text{peer}}}) \label{eqn:samplemean2}
\end{align}
$\overline{\ell_{\alpha-\text{peer}}}, \underline{\ell_{\alpha-\text{peer}}}$ denote the upper and lower bound of $\ell_{\alpha-\text{peer}}$ respectively. Further we know that 
\begin{align}
\E[\lAlphaPeer(f(X),\tilde{Y})] = \E\left[ \E_{n_1,n_2}[\lAlphaPeer(f(X),\tilde{Y})] \right] =  \E[\tilde{\ell}(f(X),\tilde{Y})].\label{eqn:equal}
\end{align}


Via Rademacher bound on the maximal deviation we have with probability at least $1-\delta$
\begin{align}
    \max_{f \in \F} \bigl |\hat{R}_{\tilde{\ell},\tilde{D}}(f)-R_{\tilde{\ell},\tilde{\mathcal D}}(f)\bigr | \leq 2  \Re(\tilde{\ell} \circ \F) + 
    \sqrt{\frac{\log 1/\delta}{2N}}~\label{eqn:maxf}
\end{align}
Since $\ell$ is $L$-Lipschitz, due to the fact that  $\tilde{\ell}$ is linear in $\ell$
, $\tilde{\ell}$ is $(1+\alpha) L$-Lipschitz. Based on the Lipschitz composition of Rademacher averages, we have
\[
\Re(\tilde{\ell} \circ \F) \leq (1+\alpha) L\cdot \Re(\F)
\]
Therefore, via union bound (events in Eqn. (\ref{eqn:samplemean2}) and Eqn. (\ref{eqn:maxf})), we know with probability at least $1-2\delta$:
\begin{align*}
    & \biggl |\frac{1}{N}\sum_{n=1}^N \lAlphaPeer(f(x_n),\tilde{y}_n) - R_{\lAlphaPeer,\tilde{\mathcal D}}(f)\biggr |\\
    =&\biggl |\frac{1}{N}\sum_{n=1}^N \lAlphaPeer(f(x_n),\tilde{y}_n) -\hat{R}_{\tilde{\ell},\tilde{D}}(f)+\hat{R}_{\tilde{\ell},\tilde{D}}(f)-R_{\lAlphaPeer,\tilde{\mathcal D}}(f)\biggr|\\
    \leq &\biggl |\frac{1}{N}\sum_{n=1}^N \lAlphaPeer(f(x_n),\tilde{y}_n) -\hat{R}_{\tilde{\ell},\tilde{D}}(f)\biggr| +\biggl |\hat{R}_{\tilde{\ell},\tilde{D}}(f)-R_{\lAlphaPeer,\tilde{\mathcal D}}(f)\biggr|\\
    \leq & \sqrt{\frac{\log 2/\delta}{2N}}\cdot (\overline{\ell_{\alpha-\text{peer}}}-\underline{\ell_{\alpha-\text{peer}}}) + |\hat{R}_{\tilde{\ell},\tilde{D}}(f)-R_{\tilde{\ell},\tilde{\mathcal D}}(f)\bigr| \tag{Eqn. (\ref{eqn:samplemean2}) and $R_{\lAlphaPeer,\tilde{D}}(f)=R_{\tilde{\ell},\tilde{D}}(f)$: Eqn. (\ref{eqn:equal})}\\
    \leq & \sqrt{\frac{\log 2/\delta}{2N}}\cdot (\overline{\ell_{\alpha-\text{peer}}}-\underline{\ell_{\alpha-\text{peer}}})+2(1+\alpha) L\cdot \Re(\F)+
    \sqrt{\frac{\log 1/\delta}{2N}}\tag{Eqn. (\ref{eqn:maxf})}\\
    \leq &2(1+\alpha) L\cdot \Re(\F)+ \sqrt{\frac{\log 2/\delta}{2N}}\cdot \bigl(1+\overline{\ell_{\alpha-\text{peer}}}-\underline{\ell_{\alpha-\text{peer}}}\bigr)
\end{align*}
In above $R_{\lAlphaPeer,\tilde{D}}(f)=R_{\tilde{\ell},\tilde{D}}(f)$ because $\lAlphaPeer$ and $\tilde{\ell}$ share the same expected risk over $\tilde{\mathcal D}$ by construction.
 Plug in the fact that
  $\lAlphaPeer$ is linear in $\ell$:
 $$
\overline{\ell_{\alpha-\text{peer}}} \leq \bar{\ell} - \alpha \cdot \underline{\ell},~~~ \underline{\ell_{\alpha-\text{peer}}}  \geq  \underline{\ell}-\alpha \cdot \bar{\ell}
 $$
 and an easy consequence that
\[
\overline{\ell_{\alpha-\text{peer}}}-\underline{\ell_{\alpha-\text{peer}}} \leq (1+\alpha)(\bar{\ell}-\underline{\ell}).
\]
Let $\delta :=\delta/2$, we conclude the proof.

\end{proof}

\section*{Proof for Lemma \ref{LEM:CONVEX}}

\begin{proof}
This was proved in the proof for Theorem \ref{THM:calibration}, when proving the classification calibration property of $\lAlphaPeer$ under condition (2).
\end{proof}

\section*{Experiment}

\subsection*{Implementation Details}
On each benchmark, we use the same hyper-parameters for all  neural network based methods.
For C-SVM, we fix one of the weights to 1, and tune the other. For PAM, we tune the margin.

\subsection*{Results}

\begin{table}[ht]
\small
\centering
\begin{tabular}{|c|c|c|c|c|c|c|c|c|c|c|c|c|c|}

\hline
\multicolumn{2}{|c|}{Task}  & \multicolumn{6}{c|}{With Prior Equalization $p = 0.5$}          & \multicolumn{6}{c|}{Without Prior Equalization $p \neq 0.5$}       \\ \hline
$(d,N_+, N_-)$ & $e_{-1}, e_{+1}$ & Peer  & Surr & \rev{Symm} & \rev{DMI}   & NN    & C-SVM & Peer  & Surr & \rev{Symm} & \rev{DMI}   & NN    & C-SVM \\ \hline\hline
               & 0.1, 0.3   & \textbf{0.977} & \textbf{0.968}     & \textbf{0.969}     & \textbf{0.974} & \textbf{0.964} & \textbf{0.966}
               & \textbf{0.977} & \textbf{0.968}     & \textbf{0.969}     & \textbf{0.974} & \textbf{0.964} & \textbf{0.966} \\
               & 0.2, 0.2   & \textbf{0.977} & \textbf{0.969}     & \textbf{0.974}     & \textbf{0.976} & \textbf{0.972} & \textbf{0.969}
               & \textbf{0.977} & \textbf{0.969}     & \textbf{0.974}     & \textbf{0.976} & \textbf{0.972} & \textbf{0.969} \\
Twonorm        & 0.1, 0.4   & \textbf{0.976} & \textbf{0.964}     & \textbf{0.956}     & \textbf{0.974} & 0.911 & 0.95
               & \textbf{0.976} & \textbf{0.964}     & 0.956     & \textbf{0.974} & 0.911 & 0.95  \\
(20,3700,3700) & 0.2, 0.4   & \textbf{0.976} & 0.919     & \textbf{0.959}     & \textbf{0.966} & 0.911 & 0.935
               & \textbf{0.976} & 0.919     & \textbf{0.959}     & \textbf{0.966} & 0.911 & 0.935 \\
               & 0.4, 0.4   & \textbf{0.973} & 0.934     & \textbf{0.958}     & 0.936 & 0.883 & 0.875
               & \textbf{0.973} & 0.934     & \textbf{0.958}     & 0.936 & 0.883 & 0.875 \\ \hline
               & 0.1, 0.3   & \textbf{0.919} & 0.878     & 0.851     & 0.875 & 0.811 & \textbf{0.928}
               & \textbf{0.925} & 0.885     & 0.868     & 0.889 & 0.809 & \textbf{0.933} \\
               & 0.2, 0.2   & \textbf{0.918} & 0.874     & 0.879     & 0.888 & 0.819 & \textbf{0.931}
               & \textbf{0.927} & 0.876     & 0.906     & 0.885 & 0.812 & \textbf{0.941} \\
Splice         & 0.1, 0.4   & \textbf{0.914} & 0.86      & 0.757     & 0.842 & 0.743 & 0.891 
               & \textbf{0.925} & 0.862     & 0.777     & 0.852 & 0.754 & 0.898 \\
(60,1527,1648) & 0.2, 0.4   & \textbf{0.901} & 0.832     & 0.757     & 0.801 & 0.714 & 0.807 
               & \textbf{0.912} & 0.84      & 0.782     & 0.81  & 0.725 & 0.824 \\
               & 0.4, 0.4   & \textbf{0.819} & 0.754     & 0.657     & 0.66  & 0.626 & 0.767 
               & \textbf{0.822} & 0.755     & 0.674     & 0.647 & 0.601 & 0.76  \\ \hline
               & 0.1, 0.3   & \textbf{0.833} & 0.78      & 0.777     & 0.797 & 0.756 & 0.753 
               & \textbf{0.856} & 0.802     & 0.803     & 0.83  & 0.75  & 0.788 \\ 
               & 0.2, 0.2   & \textbf{0.821} & 0.762     & 0.795     & \textbf{0.801} & 0.75  & 0.717 
               & \textbf{0.856} & 0.813     & 0.793     & 0.826 & 0.769 & 0.796 \\
Heart          & 0.1, 0.4   & \textbf{0.827} & 0.777     & 0.714     & 0.779 & 0.717 & 0.744 
               & \textbf{0.859} & 0.815     & 0.725     & 0.814 & 0.723 & 0.677 \\
(13,165,138)   & 0.2, 0.4   & 0.812 & 0.768     & 0.717     & 0.788 & 0.679 & 0.714 
               & \textbf{0.856} & 0.758     & 0.725     & 0.797 & 0.693 & 0.704 \\
               & 0.4, 0.4   & \textbf{0.75}  & 0.729     & 0.654     & 0.69  & 0.595 & 0.688 
               & \textbf{0.785} & 0.728     & 0.686     & 0.711 & 0.554 & 0.698 \\ \hline
               & 0.1, 0.3   & \textbf{0.745} & 0.707     & 0.674     & 0.72  & 0.667 & 0.67  
               & \textbf{0.778} & 0.75      & 0.738     & 0.729 & 0.727 & 0.726 \\ 
               & 0.2, 0.2   & \textbf{0.755} & 0.708     & 0.72      & 0.729 & 0.671 & \textbf{0.745}
               & \textbf{0.759} & 0.736     & 0.753     & \textbf{0.743} & 0.706 & \textbf{0.759} \\
Diabetes       & 0.1, 0.4   & \textbf{0.745} & 0.682     & 0.612     & 0.701 & 0.627 & 0.568 
               & \textbf{0.777} & 0.724     & 0.694     & 0.713 & 0.71  & 0.688 \\
(8,268,500)    & 0.2, 0.4   & \textbf{0.755} & 0.681     & 0.634     & 0.682 & 0.596 & 0.59  
               & \textbf{0.739} & 0.705     & 0.695     & 0.707 & 0.672 & 0.7   \\
               & 0.4, 0.4   & \textbf{0.719} & 0.645     & 0.619     & 0.637 & 0.551 & 0.654 
               & 0.651 & \textbf{0.685}     & 0.68      & 0.633 & 0.583 & \textbf{0.702} \\ \hline
               & 0.1, 0.3   & \textbf{0.639} & 0.563     & 0.507     & 0.529 & 0.519 & 0.529 
               & \textbf{0.727} & 0.645     & \textbf{0.709}     & 0.666 & 0.648 & 0.698 \\ 
               & 0.2, 0.2   & \textbf{0.659} & 0.606     & 0.537     & 0.548 & 0.534 & 0.615 
               & \textbf{0.698} & 0.661     & 0.655     & 0.627 & 0.623 & \textbf{0.695} \\
Breast         & 0.1, 0.4   & \textbf{0.587} & \textbf{0.577}     & 0.504     & 0.504 & 0.519 & 0.553 
               & \textbf{0.735} & 0.654     & 0.685     & 0.621 & 0.66  & 0.698 \\
(9,85,201)     & 0.2, 0.4   & \textbf{0.63}  & 0.534     & 0.482     & 0.496 & 0.538 & 0.538 
               & \textbf{0.73}  & 0.674     & 0.666     & 0.58  & 0.672 & 0.698 \\
               & 0.4, 0.4   & \textbf{0.596} & 0.519     & 0.504     & 0.526 & 0.471 & 0.51  
               & 0.677 & 0.628     & 0.545     & 0.537 & 0.529 & \textbf{0.698} \\ \hline
               & 0.1, 0.3   & \textbf{0.928} & \textbf{0.922}     & \textbf{0.924}     & \textbf{0.934} & 0.873 & \textbf{0.924} 
               & \textbf{0.956} & \textbf{0.949}     & \textbf{0.943}     & \textbf{0.954} & 0.92  & \textbf{0.943} \\ 
               & 0.1, 0.4   & \textbf{0.932} & \textbf{0.938}     & \textbf{0.937}     & \textbf{0.944} & 0.83  & 0.85  
               & \textbf{0.951} & 0.929     & \textbf{0.946}     & \textbf{0.941} & 0.898 & 0.929 \\
Breast         & 0.2, 0.2   & 0.928 & 0.904     & 0.835     & 0.897 & 0.887 & \textbf{0.961} 
               & \textbf{0.952} & \textbf{0.952}     & 0.897     & \textbf{0.942} & \textbf{0.955} & \textbf{0.946} \\
(30,212,357)   & 0.2, 0.4   & \textbf{0.93}  & 0.885     & 0.844     & 0.89  & 0.844 & 0.865 
               & \textbf{0.933} & 0.898     & 0.898     & \textbf{0.918} & 0.831 & 0.862 \\
               & 0.4, 0.4   & \textbf{0.928} & 0.867     & 0.819     & 0.746 & 0.824 & 0.855 
               & \textbf{0.908} & 0.839     & 0.817     & 0.795 & 0.673 & 0.866 \\ \hline
               & 0.1, 0.3   & \textbf{0.701} & 0.624     & 0.614     & 0.637 & 0.581 & 0.611 
               & \textbf{0.68}  & \textbf{0.693}     & 0.603     & 0.605 & 0.6   & 0.671 \\ 
               & 0.2, 0.2   & \textbf{0.689} & 0.65      & 0.647     & 0.623 & 0.611 & 0.664 
               & 0.702 & 0.693     & 0.704     & 0.62  & 0.6   & \textbf{0.738} \\
German         & 0.1, 0.4   & \textbf{0.696} & 0.642     & 0.587     & 0.63  & 0.562 & 0.55  
               & \textbf{0.667} & \textbf{0.693}     & 0.54      & 0.594 & 0.54  & 0.553 \\
(23,300,700)   & 0.2, 0.4   & \textbf{0.664} & 0.59      & 0.6       & 0.618 & 0.572 & 0.469 
               & \textbf{0.676} & \textbf{0.681}     & 0.537     & 0.573 & 0.535 & 0.581 \\
               & 0.4, 0.4   & \textbf{0.606} & 0.55      & 0.573     & 0.573 & 0.556 & 0.572 
               & 0.654 & 0.632     & 0.549     & 0.611 & 0.553 & \textbf{0.696} \\ \hline
               & 0.1, 0.3   & \textbf{0.89}  & \textbf{0.895}     & \textbf{0.892}     & 0.856 & 0.868 & 0.862 
               & \textbf{0.893} & \textbf{0.898}     & \textbf{0.883}     & 0.785 & 0.863 & \textbf{0.878} \\ 
               & 0.2, 0.2   & \textbf{0.883} & \textbf{0.899}     & \textbf{0.9}       & 0.861 & \textbf{0.894} & \textbf{0.886} 
               & \textbf{0.901} & \textbf{0.899}     & \textbf{0.894}     & 0.792 & \textbf{0.898} & \textbf{0.897} \\
Waveform       & 0.1, 0.4   & \textbf{0.884} & \textbf{0.893}     & 0.762     & 0.856 & 0.771 & 0.804 
               & \textbf{0.888} & \textbf{0.894}     & 0.703     & 0.778 & 0.821 & 0.821 \\
(21,1647,3353) & 0.2, 0.4   & \textbf{0.881} & \textbf{0.89}      & 0.828     & 0.835 & 0.81  & 0.795 
               & \textbf{0.884} & \textbf{0.884}     & 0.745     & 0.761 & 0.837 & 0.837 \\
               & 0.4, 0.4   & \textbf{0.87}  & \textbf{0.866}     & \textbf{0.867}     & 0.773 & 0.835 & 0.776 
               & \textbf{0.853} & \textbf{0.852}     & \textbf{0.852}     & 0.672 & 0.828 & \textbf{0.848} \\ \hline
               & 0.1, 0.3   & \textbf{0.906} & \textbf{0.9}       & 0.89      & 0.87  & \textbf{0.909} & 0.881 
               & \textbf{0.943} & 0.909     & 0.897     & 0.811 & \textbf{0.93}  & \textbf{0.924} \\ 
               & 0.2, 0.2   & \textbf{0.913} & \textbf{0.894}     & \textbf{0.907}     & \textbf{0.897} & \textbf{0.899} & \textbf{0.918 }
               & 0.905 & 0.905     & 0.905     & 0.91  & \textbf{0.936} & \textbf{0.936} \\
Thyroid        & 0.1, 0.4   & \textbf{0.875} & \textbf{0.862}     & 0.834     & 0.784 & \textbf{0.88}  & \textbf{0.869 }
               & 0.902 & \textbf{0.924}     & 0.856     & 0.75  & \textbf{0.919} & \textbf{0.917} \\
(5,65,150)     & 0.2, 0.4   & \textbf{0.863} & \textbf{0.862}     & \textbf{0.85}      & 0.784 & 0.822 & 0.781 
               & \textbf{0.905} & 0.898     & 0.865     & 0.759 & 0.881 & \textbf{0.92}  \\
               & 0.4, 0.4   & 0.762 & 0.738     & \textbf{0.859}     & 0.788 & 0.764 & 0.781 
               & 0.769 & 0.818     & \textbf{0.876}     & 0.738 & 0.738 & 0.837 \\ \hline
               & 0.1, 0.3   & 0.856 & 0.875     & 0.843     & \textbf{0.896} & 0.866 & \textbf{0.892} 
               & 0.796 & 0.835     & \textbf{0.903}     & 0.896 & 0.878 & \textbf{0.892} \\ 
               & 0.2, 0.2   & \textbf{0.9}   & 0.835     & \textbf{0.911}     & 0.894 & \textbf{0.908} & \textbf{0.912} 
               & \textbf{0.931} & 0.896     & 0.917     & 0.883 & \textbf{0.934} & 0.908 \\
Image          & 0.1, 0.4   & 0.723 & 0.841     & 0.705     & \textbf{0.881} & 0.799 & 0.785 
               & 0.717 & 0.806     & 0.679     & \textbf{0.888} & 0.825 & 0.808 \\
(18,1320,990)  & 0.2, 0.4   & 0.836 & \textbf{0.862}     & 0.719     & \textbf{0.845} & 0.832 & 0.802 
               & 0.672 & 0.755     & 0.722     & \textbf{0.86}  & 0.599 & 0.825 \\
               & 0.4, 0.4   & 0.741 & 0.72      & 0.788     & 0.763 & 0.732 & \textbf{0.834} 
               & 0.806 & 0.803     & 0.823     & 0.762 & 0.8   & \textbf{0.86}  \\ \hline
\end{tabular}
\caption{Experiment Results on 10 UCI Benchmarks. Entries within 2\% from the best in each row are in bold. \rev{Surr: surrogate loss method \citep{natarajan2013learning}; DMI: \citep{xu2019l_dmi}; Symm: symmetric loss method \citep{ghosh2015making}.} All method-specific parameters are estimated through cross-validation. The proposed method (Peer) are competitive across all the datasets. Neural-network-based methods (Peer, Surrogate, NN, \rev{Symmetric, DMI}) use the same hyper-parameters. All the results are averaged across 8 random seeds.}
\label{tab:uci}
\end{table}

The full experiment results are shown in Table.\ref{tab:uci}. \textit{Equalized Prior} indicates that in the corresponding experiments, we resample to make sure $\PP(Y=+1) = \PP(Y=-1)$ and we fix $\alpha=1$ in these experiments. Our method is competitive in all the datasets and even able to outperform the surrogate loss method with access to the true noise rates in most of them. C-SVM is also robust when noise rates are symmetric, and is competitive in 8 datasets.

From Figure \ref{fig:test}, we can see our peer loss can prevent over-fitting, which is also part of the reason of its achieved high robustness across different datasets and noise rates.

\begin{figure}
    \centering
    \subfigure[Twonorm ($e_{-1}=0.2$, $e_{+1}=0.4$)]{\includegraphics[width=0.4\textwidth]{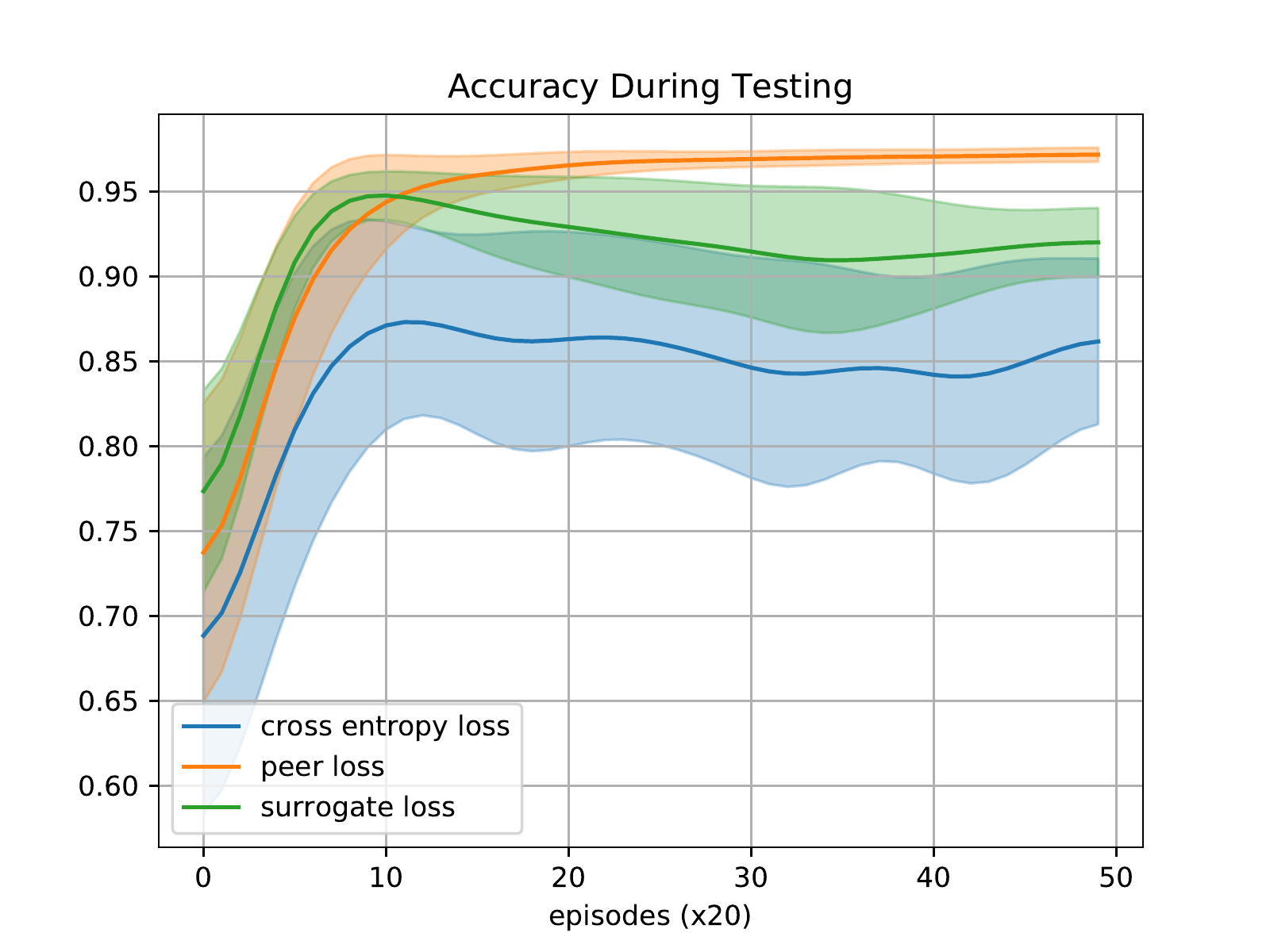}}
    \subfigure[Splice ($e_{-1}=0.1$, $e_{+1}=0.3$)]{\includegraphics[width=0.4\textwidth]{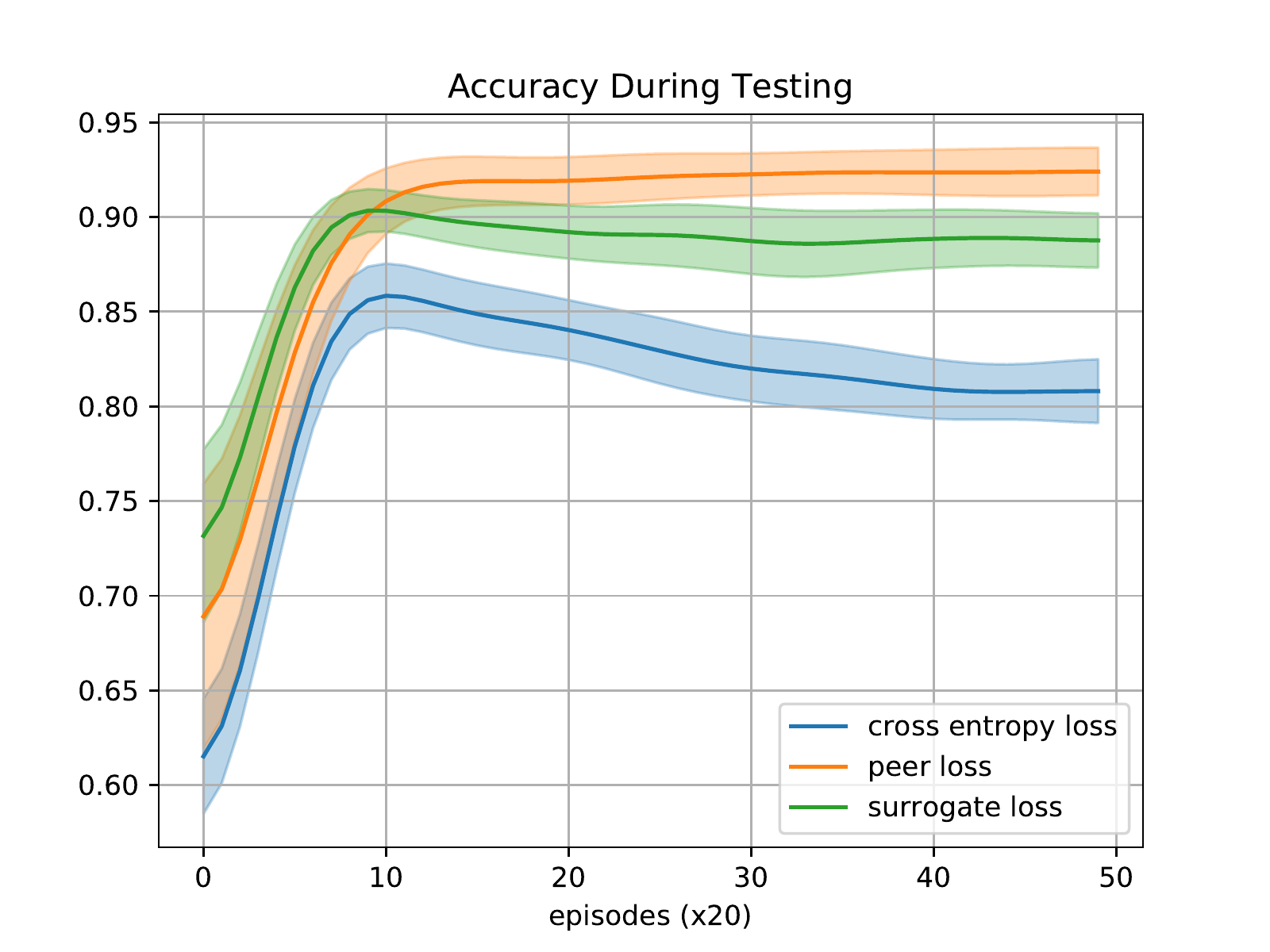}}
    \subfigure[Heart ($e_{-1}=0.2$, $e_{+1}=0.4$)]{\includegraphics[width=0.4\textwidth]{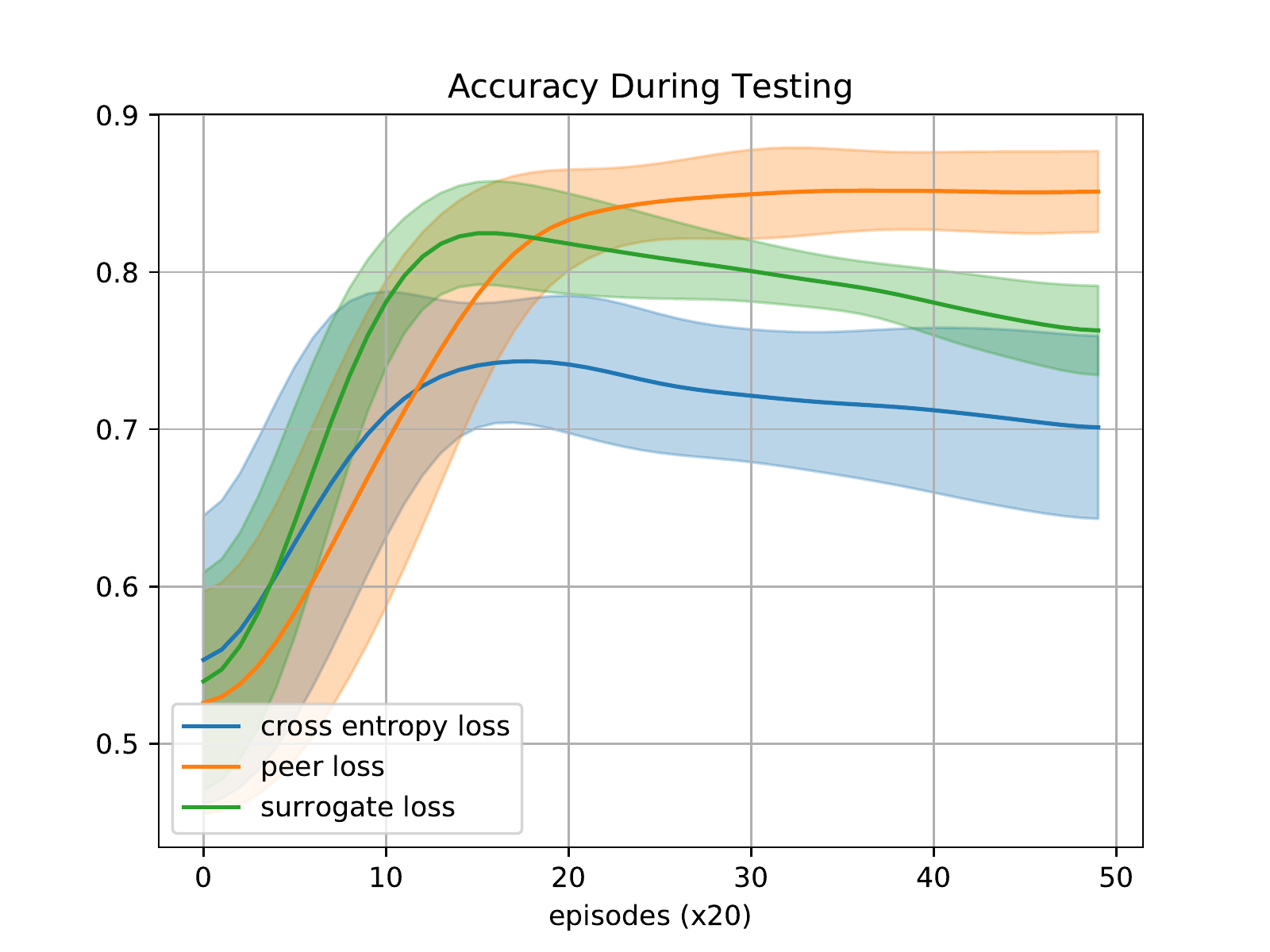}}
    \subfigure[Breast ($e_{-1}=0.4$, $e_{+1}=0.4$)]{\includegraphics[width=0.4\textwidth]{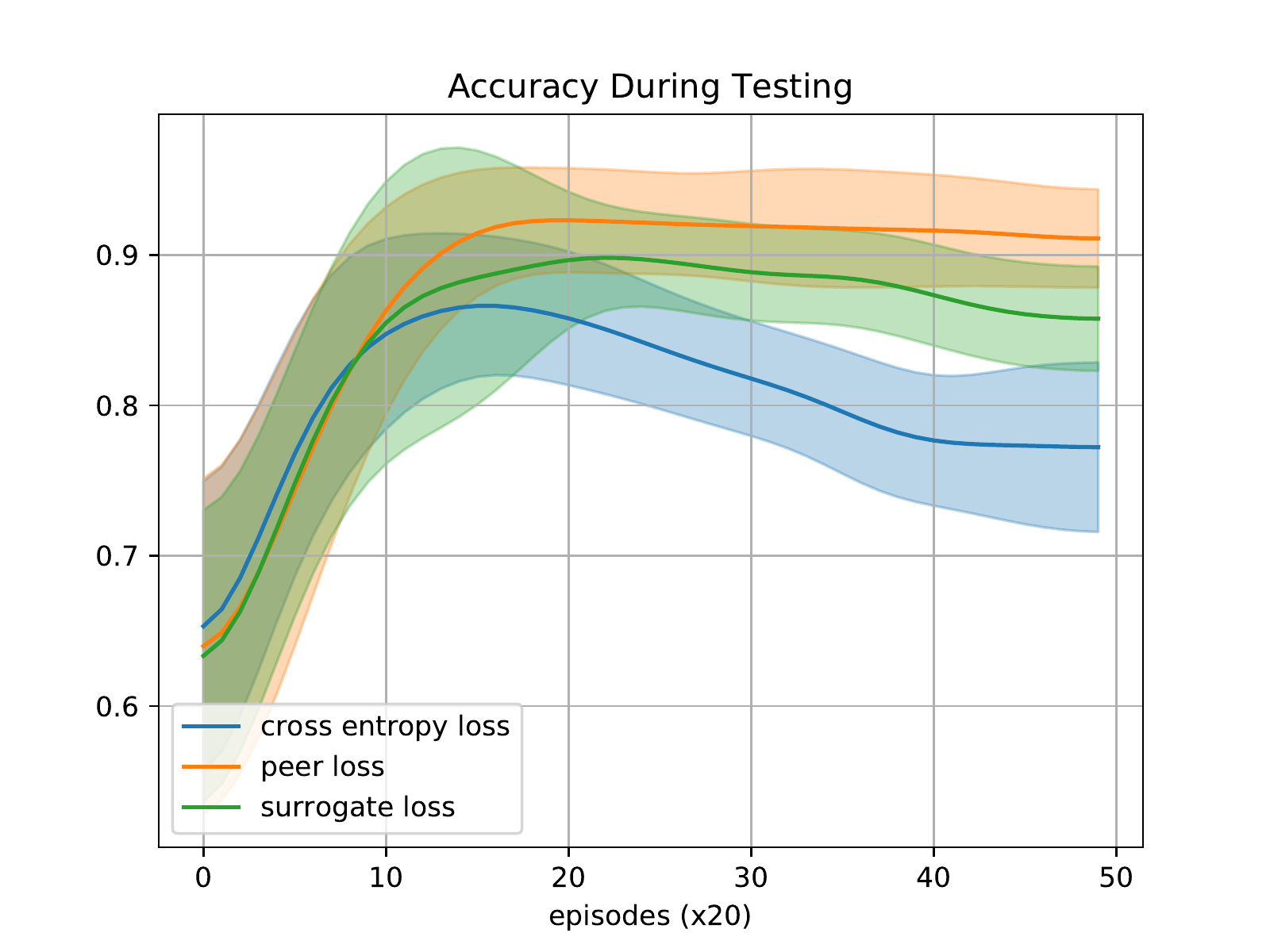}}
    \caption{Accuracy on test set during training}
    \label{fig:test}
\end{figure}

\subsection*{2D visualization of decision boundary}

\begin{figure}[!ht]
\centering
\includegraphics[width=0.25\textwidth]{figures/bcewithlogits_clean.png}\hspace{-0.1in}
\includegraphics[width=0.25\textwidth]{figures/bcewithlogits_random-0_2.png}\hspace{-0.1in}
\includegraphics[width=0.25\textwidth]{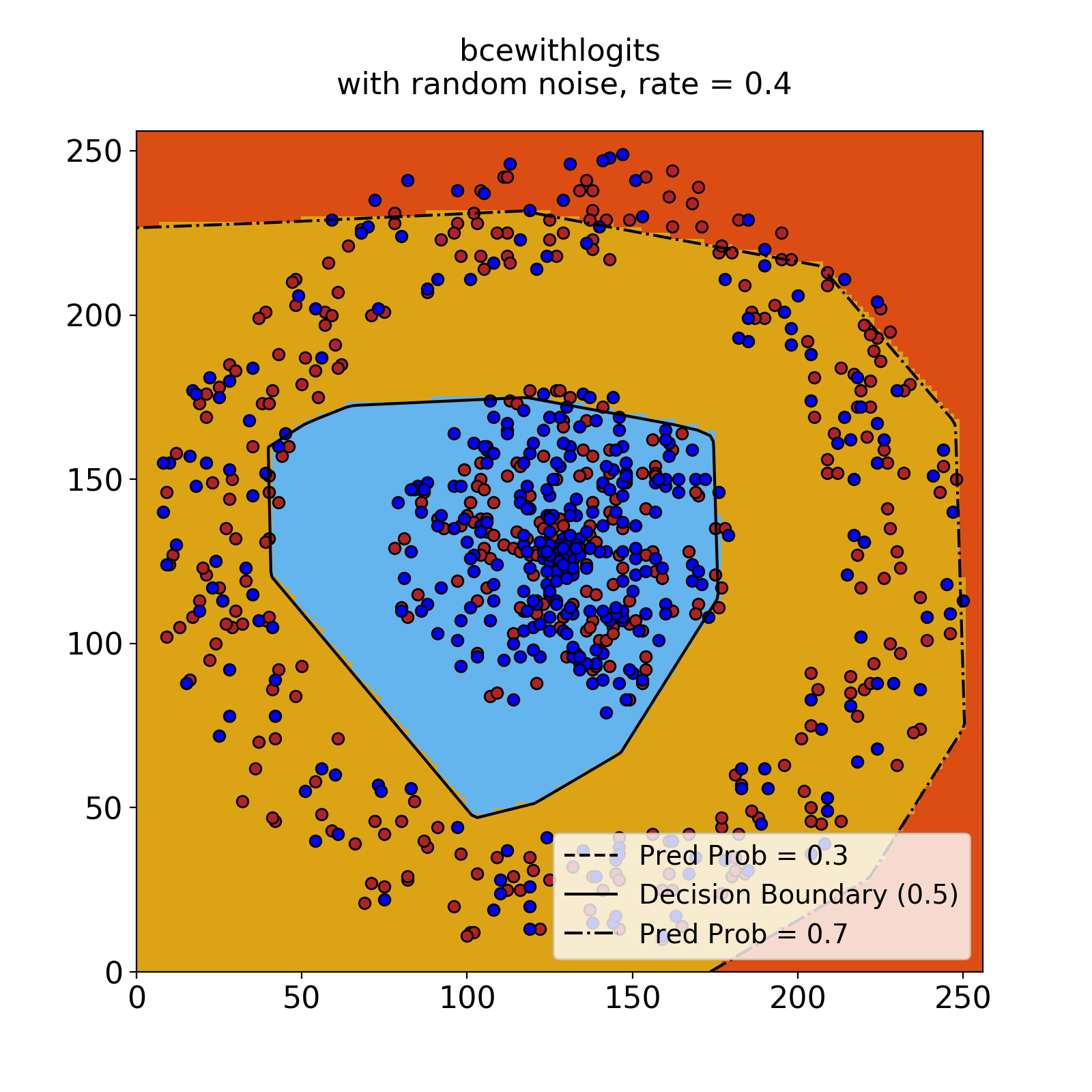}\hspace{-0.1in}
\includegraphics[width=0.25\linewidth]{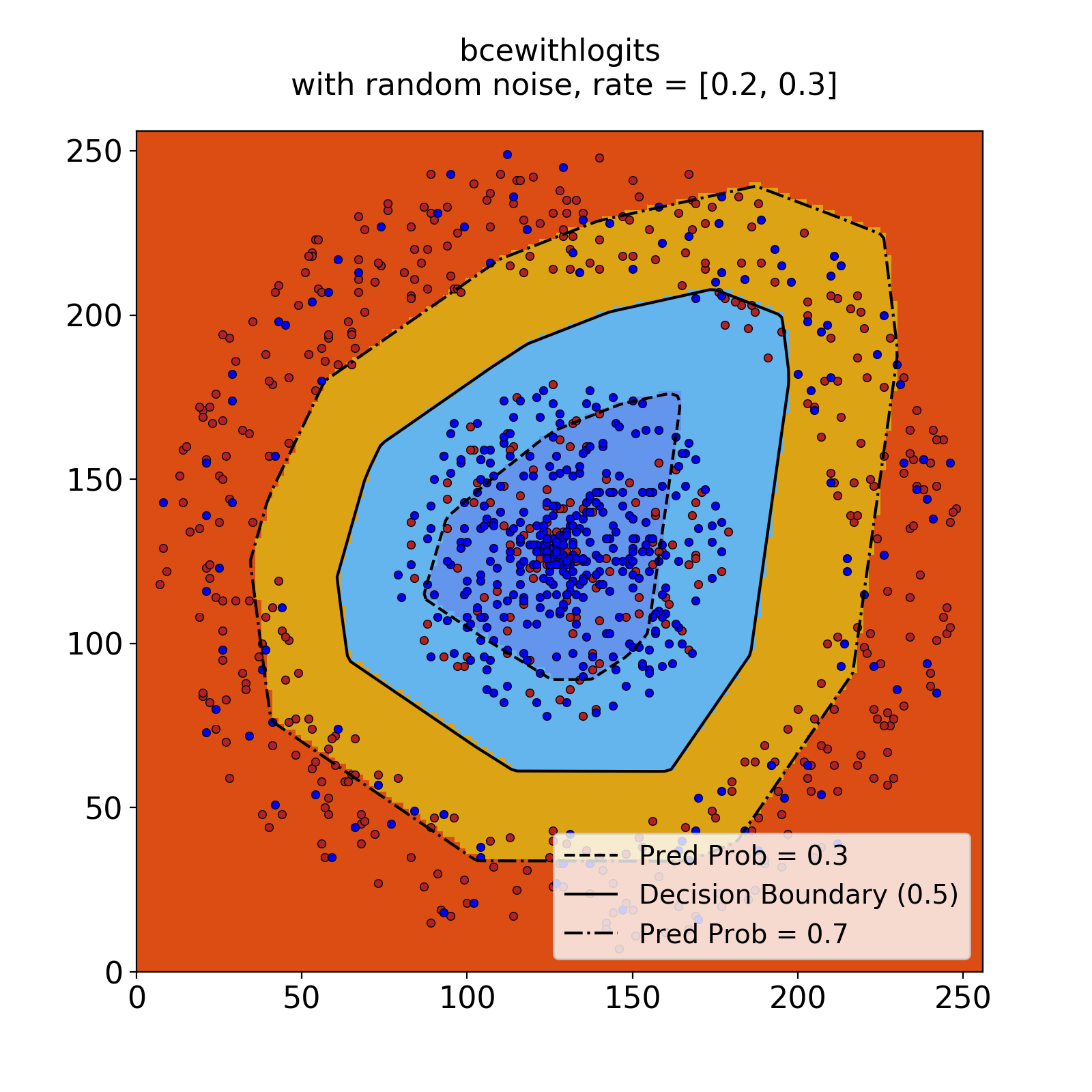}
\caption{Decision boundary for cross entropy trained on clean and noisy data (Left:  trained on noisy labels, $e_{+1}=e_{-1} = 0.2$. Middle: trained on noisy labels, $e_{+1}=e_{-1} = 0.4$. Right: asymmetric noise $e_{+1}=0.2, e_{-1} = 0.3$ ).}\label{CE:db:clean}
\end{figure}

\begin{figure}[!ht]
\centering
\includegraphics[width=0.3\textwidth]{figures/bcewithlogits_peer_random-0_2.png}\hspace{-0.15in}
\includegraphics[width=0.3\textwidth]{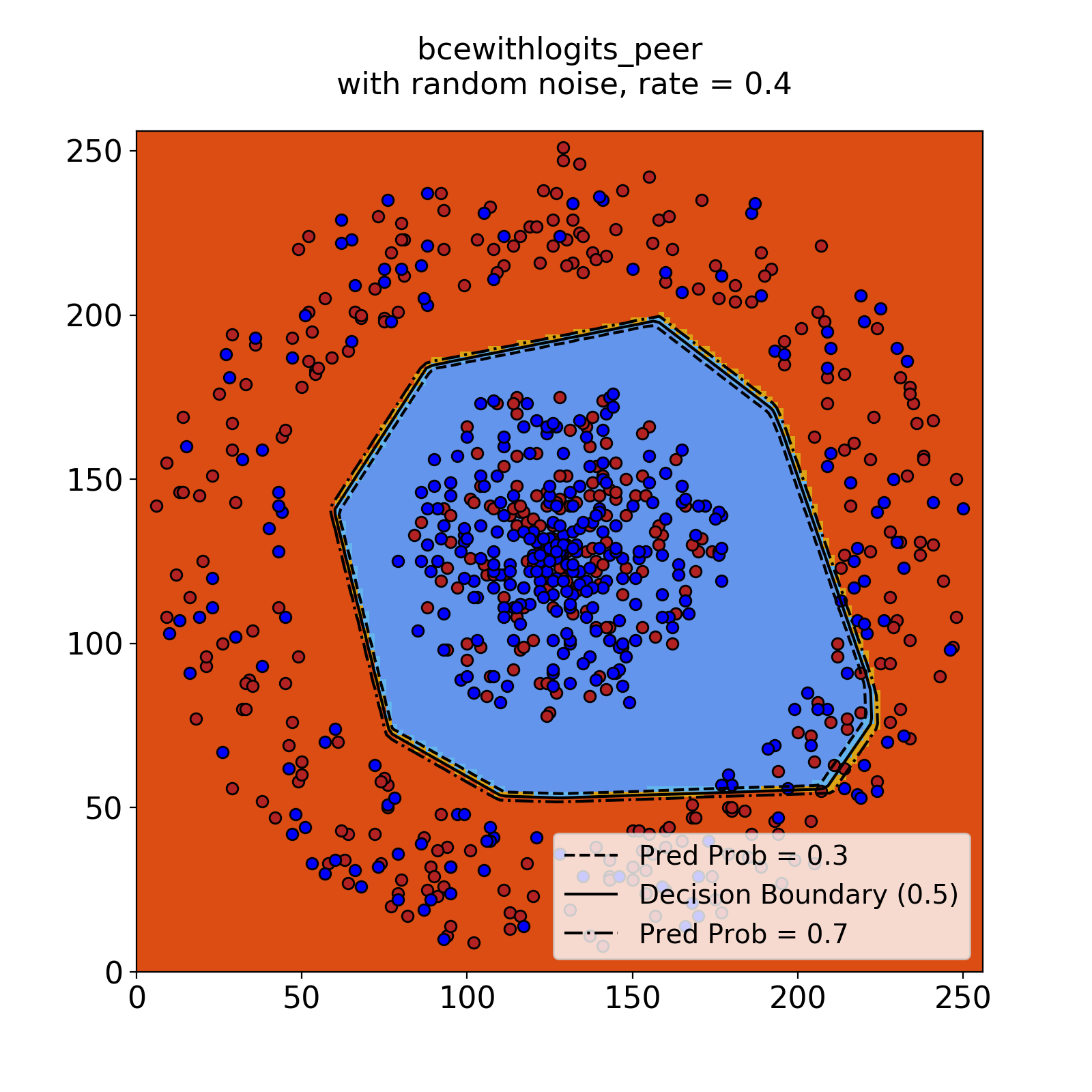}\hspace{-0.15in}
\includegraphics[width=0.3\linewidth]{figures/bcewithlogits_peer_random-2-3.png}
\caption{Decision boundary for peer loss. Left: $e_{+1}=e_{-1} = 0.2$. Middle: $e_{+1}=e_{-1} = 0.4$. Right: asymmetric noise $e_{+1}=0.2, e_{-1} = 0.3$} 
\end{figure}


\end{document}